\documentclass{article}
\usepackage[utf8]{inputenc}
\usepackage{mathtools}
\title{When Do Extended Physics-Informed Neural Networks (XPINNs) Improve Generalization?\footnote{Published in SIAM Journal on Scientific Computing (SISC)}}
\author{Zheyuan Hu\thanks{Department of Computer Science, National University of Singapore, Singapore, 119077 (\href{mailto:e0792494@u.nus.edu}{e0792494@u.nus.edu},\href{mailto:kenji@nus.edu.sg}{kenji@nus.edu.sg})}
\and Ameya D. Jagtap\thanks{Division of Applied Mathematics, Brown University, Providence, RI 02912, USA (\href{mailto:ameya\_jagtap@brown.edu}{ameya\_jagtap@brown.edu}, \href{mailto:george\_karniadakis@brown.edu}{george\_karniadakis@brown.edu})}
\and George Em Karniadakis\footnotemark[3] \and  \linebreak Kenji Kawaguchi\footnotemark[2]}

\date{}

\usepackage[a4paper,top=3cm,bottom=2cm,left=2.3cm,right=2.3cm,marginparwidth=1.75cm]{geometry}

\usepackage[numbers,sort&compress]{natbib}
\usepackage{graphicx}
\usepackage{verbatim}
\usepackage{hyperref}       
\usepackage{url}            
\usepackage{booktabs}       
\usepackage{amsfonts}       
\usepackage{nicefrac}       
\usepackage{microtype}      
\usepackage{amsmath}
\usepackage{multirow}
\usepackage{amsthm}
\newtheorem{theorem}{Theorem}[section]

\newtheorem{lemma}{Lemma}[section]
\newtheorem{assumption}{Assumption}[section]
\newtheorem{definition}{Definition}[section]
\usepackage{subfigure}

\usepackage{bibentry}
\usepackage{times}
\usepackage{color}
\usepackage{amssymb}
\newcommand{\bx}{\boldsymbol{x}}
\newcommand{\bt}{\boldsymbol{\theta}}
\newcommand{\npde}{five } 

\begin{document}

\maketitle
\begin{abstract}
Physics-informed neural networks (PINNs) have become a popular choice for solving high-dimensional partial differential equations (PDEs) due to their excellent approximation power and generalization ability. Recently, Extended PINNs (XPINNs) based on domain decomposition methods have attracted considerable attention due to their effectiveness in modeling multiscale and multiphysics problems and their parallelization. However, theoretical understanding on their convergence and generalization properties remains unexplored. 
In this study, we take an initial step towards understanding how and when XPINNs outperform PINNs. Specifically, for general multi-layer PINNs and XPINNs, we first provide a prior generalization bound via the complexity of the target functions in the PDE problem, and a posterior generalization bound via the posterior matrix norms of the networks after optimization. 
Moreover, based on our bounds, we analyze the conditions under which XPINNs improve generalization. Concretely, our theory shows that the key building block of XPINN, namely the domain decomposition, introduces a tradeoff for generalization. On the one hand, XPINNs decompose the complex PDE solution into several simple parts, which decreases the complexity needed to learn each part and boosts generalization. On the other hand, decomposition leads to less training data being available in each subdomain, and hence such model is typically prone to overfitting and may become less generalizable. Empirically, we choose \npde PDEs to show when XPINNs perform better than, similar to, or worse than PINNs, hence demonstrating and justifying our new theory.
\end{abstract}


\section{Introduction}
Deep learning has revolutionized numerous fields in computer science, such as computer vision and natural language process. Recently, deep neural networks have also been employed to solve partial differential equations (PDEs) and integrated into the field of scientific computing, thanks to their unique optimization \citep{jagtap2020adaptive,kawaguchi2016deep,Xu2021gnnopt,kawaguchi2021theory} and generalization \citep{kawaguchi2018generalization} abilities. Physics-informed neural networks (PINNs) \cite{raissi2019physics} are among the most popular approaches with a wide variety of successful applications, including heat transfer problems \citep{cai2021physics}, thrombus material properties \citep{yin2021non}, nano-optics \citep{chen2020physics}, and fluid mechanics \citep{cai2021flow,jin2021nsfnets}. PINNs are used as surrogates of a target solution for solving PDEs, and a solution is found by searching for the best parameters of PINNs that satisfy the physical laws governed by the PDEs. A more recent work \cite{jagtap2020extended} proposed Extended PINNs (XPINNs), which improves on PINNs by employing a domain decomposition method for partitioning the PDE problem into several sub-problems on subdomains, where each sub-problem can be solved by individual networks called as sub-PINNs. In XPINNs, the continuity of the PINN functions between each subdomain is maintained. XPINNs facilitate parallel computing, accelerate convergence, and improve generalization empirically. Despite great progress in applications, currently no theoretical understanding exists on when and how XPINNs are better than PINNs.

Recently, some works on theoretically understanding of PINNs have emerged \cite{Luo2020TwoLayerNN, lu2021priori, mishra2020estimates, shin2020convergence}. For two-layer networks, Luo and Yang \cite{Luo2020TwoLayerNN} derived  prior and posterior generalization bounds for PINNs based on Barron space and Rademacher complexity, whereas Lu et al. \cite{lu2021priori} provided prior error estimates based on Barron spaces with the softplus activation. For multi-layer networks, Mishra and Molinaro \cite{mishra2020estimates} introduced abstract formalism and stability properties of the underlying PDEs to derive generalization bounds, and Shin et al. \cite{shin2020convergence} used the Holder continuity constant to bound the generalization of PINNs. While these previous bounds significantly advanced our theoretical understanding of PINNs, we cannot rely on them to study the advantages and disadvantages of multi-layer XPINNs over PINNs. This is because the previous studies focus on PINNs, and the previous bounds either only apply to two-layer networks or depend on variables that are hard to be computed analytically or numerically. For example, the Holder continuity is often difficult to compute efficiently and the assumption of Holder continuity regularization is not widely adopted in practice. Accordingly, it is necessary to employ different approaches to derive new generalization bounds for the multi-layer XPINN in order to understand its advantages and limitations.

In this study, we provide an initial step towards understanding how and when XPINNs improve generalization capabilities of PINNs by proving new generalization bounds for multi-layer XPINNs and PINNs. Specifically, we first discuss the Barron space theory for multi-layer networks to define the function space of neural networks. We then derive a prior generalization bound for PINNs with the complexity of a target function measured via its Barron norm without any further assumption. Furthermore, we derive Rademacher complexity bounds of PINNs via capacity controls based on the spectral norm and the (2,1) norm, which are then used for our posterior generalization bounds for PINNs. We then extend these bounds of PINNs to those of XPINNs by applying the bounds to each of the subdomains in XPINNs and combine them to form the final results. Overall, our theoretical results predict that in terms of generalization, the advantages and disadvantages of XPINNs come from the tradeoff between the reduction in the complexity of decomposed target functions (within each subdomain) and the increase in the over-fitting due to less available training data (in each subdomain). That is, the domain decomposition of XPINNs can make a target function in a subdomain to be less complex than the whole target function, resulting in a reduction in a complexity measure, whereas each sub-network tend to utilize less than the entire available training data. To illustrate when and how XPINNs improve generalization based on our theory, we first provide analytical examples, where we mathematically compute and compare the prior bounds of XPINNs and PINNs. Furthermore, we adopt \npde PDEs to numerically demonstrate our posterior bounds via experiments. Both analytical examples and experimental observations confirm our theoretical prediction and deepen our understanding, demonstrating that the two factors in our generalization bounds lead to a tradeoff, leading to different performances of XPINNs over PINNs on various tasks. 

The remainder of this paper is arranged as follows. In Section 2, we provide properties, background and assumptions on PDEs, PINNs, and the function space for multi-layer neural networks. In Section 3, our main generalization results (both prior and posterior bounds) are presented. In Section 4, discussion on theoretical analysis as well as analytical examples are introduced. In Section 5, extensive experiments are conducted to numerically demonstrate our theory.

\section{Preliminaries}
In this section, we present introductory facts for PDEs, neural networks, as well as PINNs and XPINNs. We use bold-faced lowercase letters to denote vectors, and capital letters to denote matrices and network parameters. Given a vector $\boldsymbol{v}$, we denote its Euclidean norm by $\Vert \boldsymbol{v} \Vert$, while $\Vert \cdot\Vert_p$ refers to the $p$-norms. For matrix norms, we denote the spectral norm by $\Vert \cdot\Vert_2$ and $l_{p,q}$ norms by $\Vert \boldsymbol{W} \Vert_{p,q} = (\sum_j(\sum_k |W_{j,k}|^p)^{q/p})^{1/q}$.
Following convention, we define $\inf$ of a set $S$ to be the infimum of the subset $S$ of $\overline R$ (the set of affinely extended real numbers); e.g., the infimum of the empty set is infinity.

\subsection{PDE Problem}

In this paper, we consider PDEs defined on the bounded domain $\Omega = (-1, 1)^d$. More specifically, the PDEs under consideration are in the form of
\begin{equation}\label{eq:PDE}
\begin{aligned}
    \mathcal{L}u^*&=f \ \text{in}\ \Omega, \qquad
    u^*=g \ \text{on}\ \partial\Omega,
\end{aligned}
\end{equation}
where $\mathcal{L}$ is the differential operator characterizing the PDE, $\partial\Omega$ is the boundary of the set $\Omega$, $f:\bx=(x_{1},\dots,x_d)\in{\Omega}\longmapsto f(\bx) \in \mathbb{R}$ and $g:\bx=(x_{1},\dots,x_d)\in\partial{\Omega}\longmapsto g(\bx) \in \mathbb{R}$ are given functions, and the function  $u^*:\bx=(x_{1},\dots,x_d)\in\overline{\Omega}\longmapsto u^*(\bx) \in \mathbb{R}$ is the unknown solution of PDEs with its domain $\overline{\Omega}=\Omega \cup\partial\Omega$.

\subsection{PINN and XPINN}
In this subsection, we introduce neural network-based PDE solvers PINNs and XPINNs. Specifically, in PINNs we optimize neural networks via gradient-based algorithms to enable the network functions to satisfy the data and the physical laws governed by the PDEs. Given $n_b$ boundary training points $\left\{\boldsymbol{x}_{b,i}\right\}_{i=1}^{n_b}\subset\partial\Omega$ and $n_r$ residual training points $\left\{\boldsymbol{x}_{r,i}\right\}_{i=1}^{n_r}\subset\Omega$, we approximate the true PDE solution $u^*:\overline{\Omega}\rightarrow\mathbb{R}$ by the PINN function $u_{\boldsymbol{\theta}}$ parameterized by $\boldsymbol{\theta}$ via minimizing the empirical loss composed of the boundary loss and the residual loss, as given below.
\begin{equation}
R_S(\boldsymbol{\theta}) = \frac{1}{n_b}\sum_{i=1}^{n_b} {|u_{\boldsymbol{\theta}}(\boldsymbol{x}_{b,i})-g(\boldsymbol{x}_{b,i})|}^2 + \frac{1}{n_r}\sum_{i=1}^{n_r} {|\mathcal{L}u_{\boldsymbol{\theta}}(\boldsymbol{x}_{r,i})-f(\boldsymbol{x}_{r,i})|}^2,
\end{equation}
where the first term is included to force the network to satisfy boundary conditions, while the second term forces the network to satisfy the physical laws described by the PDEs. 

XPINN is an extension of PINN, obtained by decomposing the whole domain $\overline{\Omega}$ into several subdomains, mapped to several sub-PINNs. The continuity between each sub-nets is maintained via the interface loss function and the final solution of  XPINN is the combination and ensemble of all sub-nets, where each of them is responsible for prediction on one subdomain. 
More specifically, the original domain $\Omega$ is decomposed into $N_D$ subdomains as $\Omega = \cup_{i=1}^{N_D} \Omega_i$. The loss of XPINN contains the sum of losses for the sub-nets, which consist of boundary loss and residual loss, plus the interface loss using points on $
\partial\Omega_i\cap\partial\Omega_j$, where $i,j\in\left\{1,2,...,N_D\right\}$ such that  $
\partial\Omega_i\cap\partial\Omega_j\neq \emptyset$ to maintain the continuity between the two sub-nets $i$ and $j$. Mathematically, the XPINN loss for the $i$-th subdomain is
\begin{equation}
    R_S^{i}(\boldsymbol{\theta}^i) + \lambda_I \sum_{i,j:\partial\Omega_i\cap\partial\Omega_j\neq\emptyset} R_I(\boldsymbol{\theta}^i, \boldsymbol{\theta}^j),
\end{equation}
where $\lambda_I \geq 1$ is the weight controlling the strength of interface loss, $\boldsymbol{\theta}^i$ is the parameters for subdomain $i$, and each $R_S^i(\boldsymbol{\theta})$ is the PINN loss for subdomain $i$ containing boundary and residual losses, i.e.
\begin{equation}
    R_S^i(\boldsymbol{\theta}^i) = \frac{1}{n_{b,i}}\sum_{j=1}^{n_{b,i}} {|u_{\boldsymbol{\theta}^i}(\boldsymbol{x}^i_{b,j})-g(\boldsymbol{x}^i_{b,j})|}^2 + \frac{1}{n_{r,i}}\sum_{j=1}^{n_{r,i}} {|\mathcal{L}u_{\boldsymbol{\theta}^i}(\boldsymbol{x}^i_{r,j})-f(\boldsymbol{x}^i_{r,j})|}^2,
\end{equation}
where $n_{b,i}$ and $n_{r,i}$ are the number of boundary points and residual points in subdomain $i$ respectively, $\boldsymbol{x}^i_{b,j}$ and $\boldsymbol{x}^i_{r,j}$ are the $j$-th boundary and residual training points in subdomain $i$, respectively. Moreover, $R_I(\boldsymbol{\theta}^i, \boldsymbol{\theta}^j)$ is the interface loss between the $i$-th and $j$-th subdomains based on several interface training points $\{\boldsymbol{x}^{ij}_{I,k}\}_{k=1}^{n_{I,ij}}\subset\partial\Omega_i\cap\partial\Omega_j$
\begin{equation}
\begin{aligned}
R_I(\boldsymbol{\theta}^i, \boldsymbol{\theta}^j) &= \frac{1}{n_{I,ij}} \sum_{k=1}^{n_{I,ij}}[ |u_{\boldsymbol{\theta}^i}(\boldsymbol{x}^{ij}_{I,k})- \{\{ 
u_{\boldsymbol{\theta}^{avg}} \}\} |^2 +\\&\ |(\mathcal{L}u_{\boldsymbol{\theta}^i}(\boldsymbol{x}^{ij}_{I,k}) - f_i(\boldsymbol{x}^{ij}_{I,k}))-(\mathcal{L}u_{\boldsymbol{\theta}^j}(\boldsymbol{x}^{ij}_{I,k}) - f_j(\boldsymbol{x}^{ij}_{I,k}))|^2 ],
\end{aligned}
\end{equation}
where $\{\{ 
u_{\boldsymbol{\theta}^{avg}} \}\}  = u_{avg} \coloneqq ({u_{\boldsymbol{\theta}^i}(\boldsymbol{x}^{ij}_{I,k})+u_{\boldsymbol{\theta}^j}(\boldsymbol{x}^{ij}_{I,k})})/{2}$,  $n_{I,ij}$ is the number of interface points between the $i$-th and $j$-th subdomains, while $\boldsymbol{x}^{ij}_{I,k}$ is the $k$-th interface points between them. The first term is the average solution continuity between the $i$-th and the $j$-th sub-nets, while the second term is the residual continuity condition on the interface given by the $i$-th and the $j$-th sub-nets.

We also notice a recent paper on improving the training and generalization of XPINN \cite{de2022error}, which includes the following additional interface regularization term for XPINN:
\begin{equation}
R_A(\bt^i, \bt^j) = \frac{1}{n_{I,ij}} \sum_{k=1}^{n_{I,ij}} \sum_{m=1}^d \left| \frac{\partial u_{\bt^i}(\bx^{ij}_{I,k})}{\partial \bx_m}-\frac{\partial u_{\bt^j}(\bx^{ij}_{I,k})}{\partial \bx_m}\right|^2,
\end{equation}
where $d$ is the problem dimension, i.e., $\bx \in \mathbb{R}^d$. The additional interface condition forces the continuity of the first order derivatives between two sub-nets. In our experiment on the Poisson equation in section 5.4, which includes residual discontinuity, we shall show how this additional term improves XPINN via decreasing errors near the interface.

For our discussion on generalization, besides the training losses above, the testing loss evaluating generalization ability is defined as
\begin{equation}
    R_{{D}}(\boldsymbol{\theta})=\mathbb{E}_{\text{Unif}(\partial\Omega)} {|u_{\boldsymbol{\theta}}(\boldsymbol{x})-g(\boldsymbol{x})|}^2 + \mathbb{E}_{\text{Unif}(\Omega)} {|\mathcal{L}u_{\boldsymbol{\theta}}(\boldsymbol{x})-f(\boldsymbol{x})|}^2,
\end{equation}
where $\text{Unif}(A)$ is the uniform distribution on a set $A$. Note that in the definition of the population loss, the interface losses of XPINNs are excluded to  compare PINNs and XPINNs with the same quantity -- the generalization bound for the boundary and residual terms. The beneficial effect of the interface loss is in improving the generalization of the boundary and residual terms instead of helping the generalization of the interface term itself. This is because the interface allows the sub-net in the subdomain $\Omega_i$ to implicitly use samples from other subdomain $\Omega_j$ ($i \neq j$) for regularization through the continuity.

Lastly, we denote the boundary empirical loss and the residual empirical loss by$R_{S\cap\partial\Omega}$ and  $R_{S\cap\Omega}$, respectively, and their population versions by$R_{D\cap\partial\Omega}$ and $R_{D\cap\Omega}$, inspired by the fact that boundary points are on $\partial \Omega$ and that residual points are in $\Omega$. Specifically, their mathematical definitions are:
\begin{equation}\label{eq:loss_definition}
\begin{aligned}
R_{S\cap\partial\Omega} &= \frac{1}{n_b}\sum_{i=1}^{n_b} {|u_{\boldsymbol{\theta}}(\boldsymbol{x}_{b,i})-g(\boldsymbol{x}_{b,i})|}^2. \quad R_{S\cap\Omega} = \frac{1}{n_r}\sum_{i=1}^{n_r} {|\mathcal{L}u_{\boldsymbol{\theta}}(\boldsymbol{x}_{r,i})-f(\boldsymbol{x}_{r,i})|}^2.\\
R_{D\cap\partial\Omega} &= \mathbb{E}_{\text{Unif}(\partial\Omega)} {|u_{\boldsymbol{\theta}}(\boldsymbol{x})-g(\boldsymbol{x})|}^2.\quad R_{D\cap\Omega} = \mathbb{E}_{\text{Unif}(\Omega)} {|\mathcal{L}u_{\boldsymbol{\theta}}(\boldsymbol{x})-f(\boldsymbol{x})|}^2.
\end{aligned}
\end{equation}

\subsection{Neural Networks}
In this subsection, we define neural networks and their related properties.
\begin{definition}\label{def:DNN}
(Neural Network). A deep neural network (DNN) $u_{\boldsymbol{\theta}}:\bx=(x_{1},\dots,x_d$ $)\in\overline{\Omega}\longmapsto u_{\boldsymbol{\theta}}(\bx) \in \mathbb{R}$,
parameterized by $\boldsymbol{\theta}$ of depth $L$ is the composition of $L$ linear functions with element-wise non-linearity $\sigma$, is expressed as below.
\begin{equation}\label{eq:DNN}
    u_{\boldsymbol{\theta}}(\boldsymbol{x})=\boldsymbol{W}^L \sigma (\boldsymbol{W}^{L-1} \sigma(\cdots \sigma(\boldsymbol{W}^1\boldsymbol{x})\cdots ),
\end{equation}
where $\boldsymbol{x}\in\mathbb{R}^d$ is the input, and $\boldsymbol{W}^l\in\mathbb{R}^{m_l \times m_{l-1}}$ is the weight matrix at $l$-th layer with $d=m_0$ and $m_L=1$. The parameter vector $\boldsymbol{\theta}$ is the vectorization of the collection of all parameters. We denote $h$ as the maximal width of the neural network, i.e., $ h = \max(m_L, \cdots, m_0)$.
\end{definition}
We consider DNNs without bias because one can always set $\boldsymbol{x}\leftarrow[\boldsymbol{x},1]$ to involve the bias term.
Note that the non-linearity $\sigma$ is Lipschitz continuous with Lipschitz constant 1 in $\Omega$. The widely adopted ReLU activation function ReLU$(\boldsymbol{x}) = \max(0,\boldsymbol{x})$ cannot be used in our setting due to its non-differentiability.

Because the neural network $u_{\boldsymbol{\theta}}$ is always differentiated with respect to its input in the residual losses in PINNs, we introduce their expressions as follows:
\begin{equation}
    \frac{\partial u_{\boldsymbol{\theta}}(\boldsymbol{x})}{\partial \boldsymbol{x}} = \boldsymbol{W}^L \cdot \boldsymbol{\Phi}^{L-1} \boldsymbol{W}^{L-1} \cdot \dots \cdot \boldsymbol{\Phi}^1 \boldsymbol{W}^1 \in \mathbb{R}^{d},
\end{equation}
\begin{equation}
\begin{aligned}
    \frac{\partial^2 u_{\boldsymbol{\theta}}(\boldsymbol{x})}{\partial \boldsymbol{x}^2} &=\{\sum_{l=1}^{L-1} 
    (\boldsymbol{W}^L\boldsymbol{\Phi}^{L-1}\cdots\boldsymbol{W}^{l+1})
    \text{diag}(\boldsymbol{\Psi}^l\cdots\boldsymbol{\Psi}^1\boldsymbol{W}^1_{:,j})
    (\boldsymbol{W}^l\cdots \boldsymbol{\Phi}^1\boldsymbol{W}^1)\}_{1 \leq j \leq d},
\end{aligned}
\end{equation}
where
$\boldsymbol{\Phi}^l = \text{diag}[\sigma'(\boldsymbol{W}^l\sigma(\boldsymbol{W}^{l-1}\sigma(\cdots\sigma(\boldsymbol{W}^1\boldsymbol{x}))))] \in \mathbb{R}^{m_l\times m_l}$, and
$\boldsymbol{\Psi}^l = \text{diag}[\sigma''(\boldsymbol{W}^l\sigma($ $\boldsymbol{W}^{l-1}\sigma(\cdots\sigma(\boldsymbol{W}^1\boldsymbol{x}))))] \in \mathbb{R}^{m_l\times m_l}$.
In the Appendix, we provide detailed computation of the derivatives.

\subsection{Generalized Barron Space}
In this subsection, we introduce the generalized Barron space \cite{Weinan2020OnTB}, which is a natural building block to construct a function space of multi-layer deep networks. This will facilitate our study of their approximation and generalization properties. We begin by presenting some mathematical background.

Let $X$ be a Banach space such that $X$ embeds continuously into the space $C^{2,1}(\overline{\Omega})$ of functions on $\overline{\Omega}$. We further assume that the closed unit ball $B^X$ in $X$ is closed in the topology of $C^2(\overline{\Omega})$. 

Because $X$ embeds continuously into the space $C^{2,1}(\overline{\Omega})$, the Lipschitz constants of functions in $B^X$ and their derivatives up to second order are bounded by the same constant and thus the subset $B^X$ is uniformly $2$-equicontinuous (see the supplementary for definition).

The subset $B^X$ is pre-compact, i.e., its closure is compact, in the separable Banach space $C^2(\overline{\Omega})$, because $\overline{\Omega}$ is compact, and that $C^0(\overline{\Omega})$ is separable. Since $B^X$ is $C^2$-closed and pre-compact, it is compact and is a Polish space in particular. 

Let $\mu$ be a finite signed measure on the Borel $\sigma$-algebra of $B^X$, with respect to the $C^0$-norm. Then $\mu$ is a signed Radon measure. 
We therefore consider the infinite-dimensional vector function version of the activation $\sigma: B^X \rightarrow C^2(\overline{\Omega})$ given by
$\sigma: g \longmapsto (\sigma \circ g) \text{ for each } g \in  B^X$,
where $(a \circ b)$ represents the composition of functions $a$ and $b$. Then, 
the infinite-dimensional vector function $\sigma$ is continuous due to the Lipschtiz continuity of the dimensional wise version of $\sigma$. Thus, the infinite-dimensional vector function $\sigma$ is strongly measurable (the preimage of Borel sets in $B^X$ are Borel sets in $C^0(\overline{\Omega})$), and $\mu$-integrable in the sense of the Bochner integral. The above construction of new function class containing all $\sigma \circ g$ from the class of $g$ in $B^X$ can be formalized as follows.
\begin{definition} (Generalized Barron Space) The generalized Barron space modeled on $X$ associated with the non-linearity $\sigma$ is a normed space $(\mathcal{B}_{X,\Omega}, \Vert \cdot \Vert_{X,\Omega})$ with
\begin{equation}
\begin{aligned}
\mathcal{B}_{X,\Omega} &= \left\{f \in C^2(\overline{\Omega}): \Vert f \Vert_{X,\Omega} < \infty \right\}, \text{ and }\\
\Vert f \Vert_{X,\Omega} &= \inf \left\{\Vert \mu \Vert_{\mathcal{M}(B^X)}:\mu \in \mathcal{M}(B^X) \ \text{s.t.}\ f=f_\mu \ \text{on}\ \overline{\Omega}  \right\},
\end{aligned}
\end{equation}
where $\mathcal{M}(B^X)$ denotes the space of Radon measures on $B^X$ and
$f_\mu = \int_{B^X} \sigma(g) d \mu(g)$.
Here, the integral represents the Bochner integral with $g \in B^X$.
\end{definition}
For example, if $X$ is the space of linear functions from $\mathbb{R}^d$ to $\mathbb{R}$ (which is isomorphic to $\mathbb{R}^{d+1}$), then the generalized Barron space modeled on $X$ is the (usual) Barron space for two-layer neural networks \cite{Weinan2018APE}. If $X$ is the (usual) Barron space of two-layer neural networks, then the generalized Barron space modeled on $X$ is the space for three-layer neural networks. That is, we can construct the generalized Barron space of $L$-layer networks from that of $(L-1)$-layer networks by recursively applying its definition. This recursive construction leads to the \textit{tree-like function space} (Definition \ref{def:1}), which is a function space of multi-layer neural networks (Theorem \ref{thm:embedding}), as given below.
\begin{definition} \label{def:1}
(Tree-Like Function Space for Deep Networks) The tree-like function space $\mathcal{W}^L(\Omega)$ of depth $L$ is recursively defined by $\mathcal{W}^l(\Omega)=\mathcal{B}_{\mathcal{W}^{l-1}(\Omega), \Omega}$ for all $l \in \{2,3,\dots,L\}$ where $\mathcal{W}^1(\Omega)$ is the space of linear functions from $\mathbb{R}^d$ to $\mathbb{R}$. 
\end{definition}
The term ``tree like function space'' is taken from one of the most related works \cite{Weinan2020OnTB}, and a famous paper on neural network approximation \cite{poggio2017and}. Intuitively, the neural networks resemble the tree structure, where each neuron in the network is a node in the graph, and the edges of the graph connect the neurons in the previous and the present layers. Our ``tree like function space'' has nothing to do with the tree functions in graph theory. Our definition contains a vast function space covering the majority of PDE solutions.
\begin{theorem}\label{thm:embedding}
    (Embedding of Finite Networks). The tree-like function space $\mathcal{W}^L(\Omega)$ contains all finite multi-layer networks $u_{\boldsymbol{\theta}}(\boldsymbol{x})$ of depth $L$ satisfying $\Vert \boldsymbol{W}^l \Vert_{1,\infty} \leq 1, 1 \leq l \leq L-1 $, and the Barron norm of networks satisfies
    $\Vert u_{\boldsymbol{\theta}} \Vert_{\mathcal{W}^L(\Omega)} \leq  \Vert \boldsymbol{W}^L \Vert_{1,\infty}$.
\end{theorem}
Theorem \ref{thm:embedding} shows that the tree-like function space constructed via the generalized Barron space indeed contains the class of multi-layer neural networks, the norm of which is controlled by the $1,\infty$ matrix norm of their parameters. The following is a list of basic properties of the generalized Barron space, which also holds for the tree-like function space and justifies our recursive construction of the tree-like function space:
\begin{theorem}\label{thm:barron_property}
(Property of Generalized Barron Spaces). The following two statements are true. (1) The generalized Barron space is complete in the metric defined by the generalized Barron norm $\Vert \cdot \Vert_{X,\Omega}$: i.e., the generalized Barron space is a Banach space. (2) $\mathcal{B}_{X,\Omega}$ embeds continuously into $C^{2,1}(\overline{\Omega})$ and the closed unit ball of $\mathcal{B}_{X,\Omega}$ is a closed subset of $C^2(\overline{\Omega})$.
\end{theorem}

The last property indicates that $\mathcal{B}_{X,\Omega}$ satisfies the same properties, which we imposed on $X$ during the construction, i.e., we can repeat the construction and consider $B_{\mathcal{B}_{X,\Omega}}$, hence, ensuring the validity of the recursive construction in the tree-like function space. As universal approximators, neural networks can also approximate arbitrary Barron functions accurately.
\begin{theorem}\label{thm:approximation}
(Approximation Properties of Tree-Like Functions). Let $\mathbb{P}$ be a probability measure with compact support in $\Omega$, and $\mathbb{Q}$ be a probability measure with compact support in $\partial\Omega$. Then for any $L \geq 1, f \in \mathcal{W}^L(\Omega)$ and $m \in \mathbb{N}$, there exists a neural network $u_{\bt}(\bx)$ of depth $L$, with width $m_l=m^{L-l+1},\ \forall l>1$ such that
\begin{equation}
\begin{aligned}
\Vert u_{\bt} - f\Vert_{H^2(\mathbb{P})} &\leq \frac{3L\Vert f \Vert_{\mathcal{W}^L(\Omega)}}{\sqrt{m}},\\
\Vert u_{\bt} - f\Vert_{L^2(\mathbb{Q})} &\leq \frac{3C_{\Omega}L\Vert f \Vert_{\mathcal{W}^L(\Omega)}}{\sqrt{m}},
\end{aligned}
\end{equation}
where $C_{\Omega}$ is a universal constant only depends on the domain $\Omega$, and $H^2 = W^{2,2}$ is the Sobolev space, and $\Vert \boldsymbol{\theta} \Vert_{\mathcal{P}} \leq \Vert \boldsymbol{W}^L \Vert_{1,\infty} \leq  \Vert f \Vert_{\mathcal{W}^L(\Omega)}$,
where $\Vert \cdot \Vert_{\mathcal{P}}$ is the path norm defined as 
\begin{equation}
\Vert \boldsymbol{\theta} \Vert_\mathcal{P} = \sum_{i_L}\cdots\sum_{i_0}|\boldsymbol{W}^L_{i_L}\cdots\boldsymbol{W}^1_{i_1 i_0}|.
\end{equation}
\end{theorem}
The path norm is one type of complexity measure of neural networks correlated to generalization \cite{neyshabur2017exploring}. The above theorem shows that the neural networks can approximate any target function in the generalized Barron space and its derivatives well with complexities controlled by the Barron norm of the target functions, which shows the efficiency of network approximation. Note that since $\Omega$ is fixed in this paper, the constant $C_\Omega$ is actually universal. This is utilized in the proof of our prior generalization bound in Theorem \ref{thm:generalization}.

Actually, the theoretical result that neural networks can approximate a function and its derivative is not new \cite{cardaliaguet1992approximation,attali1997approximations}.
Their proof idea is two-step. First, they show that polynomials are dense in $C^k(\overline{\Omega})$. Second, they can approximate polynomials in $C^k$-norm, using Taylor's expansion. 
In our paper, we adopt a functional analysis approach to adapt to the Barron space setting, i.e., to show additionally that such networks have low complexity measured by Barron norm, which is indispensable to our prior bound.

\section{Theory}

In this section, we introduce our main generalization results, including a prior bound based on the Barron space and a posterior bound based on the Rademacher complexity. For both of them, we use the following assumption adopted from a closely related previous study \cite{Luo2020TwoLayerNN}.

\begin{assumption}
\label{assumption:bounded}
(Symmetry and boundedness of $\mathcal{L}$). Throughout the analysis in this paper, we assume the differential operator $\mathcal{L}$ in the PDE satisfies the following conditions.
The operator $\mathcal{L}$ is a linear second-order differential operator in a non-divergence form, i.e.,
$(\mathcal{L}u^*)(\bx) = \sum_{\alpha=1,\beta=1}^d \boldsymbol{A}_{\alpha \beta} (\bx)u^*_{x_\alpha x_\beta} (\bx)+ \sum_{\alpha=1}^d \boldsymbol{b}_\alpha (\bx)u^*_{x_\alpha} (\bx)+ c(\bx) u^*(\bx)$,
where all $\boldsymbol{A}_{\alpha \beta},\boldsymbol{b}_\alpha, c:\Omega \rightarrow \mathbb{R}$ are given coefficient functions and $u^*_{x_\alpha}$ are the first-order partial derivatives of the function $u^*$ with respect to its $\alpha$-th argument (the variable $x_\alpha$) and $u^*_{x_\alpha x_\beta}$ are the second-order partial derivatives of the function $u^*$ with respect to its $\alpha$-th and $\beta$-th arguments (the variables $x_\alpha$ and $x_\beta$).
Furthermore, there exists constant $K>0$ such that for all $\bx\in\Omega=[-1,1]^d$, and $\alpha,\beta\in[d]$, we have $A_{\alpha\beta}=A_{\beta\alpha}$ and
$A_{\alpha\beta}(\boldsymbol{x}),  b_\alpha(\boldsymbol{x}),  c(\boldsymbol{x}) $ are all $K$-Lipschitz, and their absolute values are not larger than $K$.
\end{assumption}
Because multiplying the network functions by the coefficients $\boldsymbol{A}, \boldsymbol{b}, c$ and differentiation on them influence their complexities, the universal bound on the coefficients $\boldsymbol{A}, \boldsymbol{b}, c$ and the restriction to second order PDEs are required for our estimation on the Rademacher complexity of the hypothesis class of PINNs.  

\subsection{A Prior Generalization Bound (Theorem \ref{thm:generalization})}
In this subsection, we introduce our prior bound based on  the Barron space.
\begin{theorem}\label{thm:generalization}
(A prior generalization bound on PINN). Let the Assumption given in \ref{assumption:bounded} holds, then for any $\delta\in(0,1)$ and the depth $L$, suppose that the true solution $u^*(\boldsymbol{x})$ lies in the tree-like function space $\mathcal{W}^L(\Omega)$, and set $\lambda = {3(2KC_{\Omega}+1)L^2}/{m}$. Let
$\bt^* = \arg \min_{\boldsymbol{\theta}}  R_S(\boldsymbol{\theta}) + \lambda \Vert \boldsymbol{\theta} \Vert^2_{\mathcal{P}}$.
Then, with probability at least $1-\delta$ over the choice of random samples $S=\left\{\boldsymbol{x}_i\right\}_{i=1}^{n_b+n_r} \subset \overline{\Omega}$ with $n_b$ boundary points and $n_r$ residual points, we obtain the following generalization bound
\begin{equation} \label{eq:1}
\begin{aligned}
R_{{D} \cap \partial  \Omega}(\bt^*) &\leq R_{S \cap \partial \Omega}(\bt^*) + 8\Vert u^* \Vert_{\mathcal{W}^L(\Omega)}\frac{C(h)\log n_b}{{\sqrt{n_b}}} + 2\sqrt{\frac{\log(2/\delta)}{n_b}},\\
R_{{D} \cap \Omega}(\bt^*) \leq & R_{S \cap \Omega}(\bt^*)+8\left(\Vert u^* \Vert_{\mathcal{W}^L(\Omega)}\right)^3\frac{C(h,K)\log n_r}{{\sqrt{n_r}}} + 2\sqrt{\frac{\log(2/\delta)}{n_r}},
\end{aligned}
\end{equation}
where $C(h)$ and $C(h,K)$ are universal constants depending only on $h$ and $h,K$, respectively.
\end{theorem}

In Theorem \ref{thm:generalization}, the generalization  bounds on the right-hand side of equation \eqref{eq:1} (for both the boundary and the residual points) contain three terms, where the first term is the empirical training loss, the second term is the complexity of the model (original network for boundary loss and differentiated network for residual loss), and the third term is the statistical term. 
Moreover, this theorem shows that under certain regularization of the path norm, the generalization errors of PINNs are controlled by the Barron norm of the target function $u^*$. If the target function $u^*$ is more complex (simpler), i.e., it has larger (smaller) Barron norm, the generalization error will be larger (smaller). This reflects a data-dependent bound in which neural networks control their complexity based on those of target functions. The above advantages are also summarized in the appendix to justify our choice of Barron space.

\subsection{A Posterior Generalization Bound (Theorem \ref{thm:post_generalization})}

We now provide a \textit{posterior} generalization bound based on the \textit{optimized} network parameters (which are obtained after optimization). We begin by defining the Rademacher complexity \cite{bartlett2002rademacher}, which is one of key notions in statistical learning theory.
\begin{definition}
(Rademacher Complexity). Let $S=\left\{x_i\right\}_{i=1}^n \subset \overline{\Omega}$ be a dataset containing $n$ samples. The Rademacher complexity of a function class $\mathcal{F}$ on $S$ is defined as
$\text{Rad}(\mathcal{F};S)=\mathbb{E}_{\epsilon}\left[\sup_{f\in\mathcal{F}} \frac{1}{n}\sum_{i=1}^n \epsilon_if(x_i)\right]$,
where $\epsilon_1,\dots,\epsilon_n$ are independent and identically distributed (i.i.d.)  random variables taking values uniformly in $\{-1,1\}$. 
\end{definition}
Intuitively, Rademacher complexity measures the richness of function class $\mathcal{F}$ by studying its ability to fit random labels of $x_i$ generated by $\epsilon_i$. Since simpler function classes tend to generalize better on unseen testing data, we will investigate the Rademacher complexities of PINNs and XPINNs, which begins with
a key lemma on that of neural networks.
\begin{lemma}\label{lemma:rad_nn}
\cite{bartlett2017spectrally} For every $L$, and every set of $n$ points $S \subset \overline{\Omega}$, the hypothesis class $\mathcal{NN}^L_{M,N}$ given by the neural networks
\begin{equation}
\mathcal{NN}^L_{M,N} := \left\{\boldsymbol{x} \mapsto W_L \sigma (W_{L-1} \sigma(\cdots \sigma(W_1\boldsymbol{x})\cdots )\ |\ \Vert W_l \Vert_{2} \leq M(l), \frac{\Vert W_l \Vert_{2,1}}{\Vert W_l \Vert_{2}} \leq N(l)\right\},
\end{equation}
satisfies the Rademacher complexity bound
\begin{equation}
\text{Rad}(\mathcal{NN}^L_{M,N};S) \leq   \frac{4}{n\sqrt{n}} + \frac{18\sqrt{d\log(2h^2)}\log n}{\sqrt{n}} \prod_{l=1}^L M(l)\Big(\sum_{l=1}^L N(l)^{2/3}\Big)^{3/2},
\end{equation} 
where $h$ is the maximal width of the neural network, i.e., $h = \max(m_L, \cdots, m_0)$.
\end{lemma}
This lemma controls Rademacher complexities of neural networks by the product of the spectral  norms of the network parameter matrices at each layer, i.e., $M(l)$. The complexity depends on the network depth via the $N(l)$ term as we always have $N(l) \geq 1$.
We extend this result to the differentiated PINNs in the following lemma.
\begin{lemma}\label{lemma:rad_pinn}
(Rademacher Complexity of Differentiated Networks). For every $L$, and every set of $n$ points $S \subset \overline{\Omega}$, the hypothesis class
\begin{equation}
\mathcal{PINN}^L_{M,N}= \left\{\boldsymbol{x} \mapsto\mathcal{L}u(\bx)\ |\ u(\bx)\in\mathcal{NN}^L_{M,N}\right\},
\end{equation}
satisfies the Rademacher complexity bound
\begin{equation}
\begin{aligned}
& \text{Rad}(\mathcal{PINN}^L_{M,N};S) \leq\frac{8K + 4d(L-1)K}{n_r\sqrt{n_r}} +\frac{18K\sqrt{d\log(2h^2)}\log n_r}{\sqrt{n_r}} \cdot \\
&\quad \prod_{l=1}^L M(l) \left(\sum_{l=1}^L N(l)^{2/3}\right)^{3/2}\left[ 1 + \sqrt{2} L\prod_{l=1}^L M(l)+ \sqrt{2} d(L^2-1)\left(\prod_{l=1}^L M(l)\right)^2 \right].
\end{aligned}
\end{equation}
\end{lemma}
This lemma shows a similar Rademacher complexity bound for PINNs, where the main differences are due to the first order and second order differentiation in PINNs, respectively. 
Using these lemmas, we derive the following posterior generalization bound.
\begin{theorem}\label{thm:post_generalization}
(A posterior generalization bound on PINN). Let Assumptions \ref{assumption:bounded} hold, for any $\delta\in(0,1)$ and the depth $L$, let the (not regularized) empirical loss function be $
    \boldsymbol{\theta}_{S} = \arg \min_{\boldsymbol{\theta}}  R_S(\boldsymbol{\theta}).$
Then, with probability at least $1-\delta$ over the choice of random samples $S=\left\{\boldsymbol{x}_i\right\}_{i=1}^{n_b+n_r} \subset \overline{\Omega}$ with $n_b$ boundary points and $n_r$ residual points, we have the following generalization bound
\begin{equation}
\begin{aligned}
R_{D\cap\partial\Omega}(\bt_S)&\leq R_{S\cap\partial\Omega}(\bt_S)+ \frac{32}{n_b\sqrt{n_b}} + \frac{144\sqrt{d\log(2h^2)}\log n_b}{\sqrt{n_b}}\cdot  \\
&\quad \prod_{l=1}^L M(l)\Big(\sum_{l=1}^L N(l)^{2/3}\Big)^{3/2}+2 \sqrt{\frac{\log(2/\delta(M,N))}{2n_b}} ,
\end{aligned}
\end{equation}
\begin{equation}
\begin{aligned}
&\quad R_{D\cap\Omega}(\bt_S)-R_{S\cap\Omega}(\bt_S)\\ 
&\leq \frac{64K + 32d(L-1)K}{n_r\sqrt{n_r}} + 2 \sqrt{\frac{\log(2/\delta(M,N))}{2n_r}} +  \frac{144K\sqrt{d\log(2h^2)}\log n_r}{\sqrt{n_r}}\cdot\\
&\quad \prod_{l=1}^L M(l) \left(\sum_{l=1}^L N(l)^{2/3}\right)^{3/2} \left[ 1 + \sqrt{2} L\prod_{l=1}^L M(l)+ \sqrt{2} d(L^2-1)\left(\prod_{l=1}^L M(l)\right)^2 \right].
\end{aligned}
\end{equation}
where $M(l) = \lceil \Vert \boldsymbol{W}^l \Vert_{2} \rceil$, and $N(l) = \lceil \Vert \boldsymbol{W}^l \Vert_{2,1} / \Vert \boldsymbol{W}^l \Vert_{2} \rceil$, in which $\lceil a \rceil$ of $a \in \mathbb{R}$ is the smallest integer that is greater than or equal to $a$, and
\begin{equation}
\delta(M,N) = {\delta}/\left[{\prod_{l=1}^L M(l)(M(l)+1)N(l)(N(l)+1)}\right].
\end{equation}
\end{theorem}
The generalization bounds are mainly controlled by the complexity of networks measured by the spectral norm $\left\{M(l)\right\}_{l=1}^L$ and the (2,1) norm $\left\{N(l)\right\}_{l=1}^L$, as well as the last statistical term. 
When compared to the prior bound in Theorem \ref{thm:generalization}, the posterior bound in Theorem  \ref{thm:post_generalization} is easier to compute numerically, because it only involves terms related to neural network parameters. 
Despite the difference, Theorems \ref{thm:generalization} and \ref{thm:post_generalization} are related. Concretely, Theorem \ref{thm:generalization} shows that if the target function $u^*$ is more complex (simpler), i.e., has larger (smaller) Barron norm, the generalization error is larger (smaller), which implies that a complex (simple) neural network has been learnt to fit the target since complex (simple) network generalizes worse (better). On the other hand, the complexity of neural networks can also be reflected by the quantities $\left\{ M(l), N(l)\right\}_{l=1}^L$ in Theorem \ref{thm:post_generalization} since they are directly linked to Rademacher complexity in Lemmas \ref{lemma:rad_nn} and \ref{lemma:rad_pinn}. Hence, we can expect that PINNs learning more complex target functions have larger $\left\{ M(l),N(l)\right\}_{l=1}^L$ quantities. And if a trained PINN has large $\left\{ M(l),N(l)\right\}_{l=1}^L$ quantities, which signifies higher complexity, the target function fitted should also be complicated (i.e., it has a larger Barron norm) due to the implicit regularization of network training, i.e., stochastic gradient training finds out complex (simple) solution given a complex (simple) target.

\subsection{Posterior $L_2$ Error Generalization Bound}
In this subsection, we will bridge the gap between the boundary + residual generalization error with the $L_2$ generalization error. In particular, we adopt the following assumption widely used in numerical PDE methods \cite{bochev2016least,bramble1970rayleigh} and PINN theory \cite{shin2020convergence, mishra2020estimates, de2022error, Ryck2021ErrorAF}. Intuitively, the assumption states that minimization of the boundary and residual errors contributes to the minimization of $L_2$ error.

\begin{assumption}\label{assumption:L2}
Assume that the PDE satisfies the following norm constraint:
\begin{equation}
\begin{aligned}
C_1 \Vert u \Vert_{L_2(\Omega)} \leq \Vert \mathcal{L}u \Vert_{L_2(\Omega)} + \Vert u \Vert_{L_2(\partial\Omega)}, \qquad \forall u \in \mathcal{NN}_{L}^M, \forall L, M,
\end{aligned}
\end{equation}
where the positive constant $C_1$ does not depend on $u$ but on the domain and the coefficients of the operators $\mathcal{L}$.
\end{assumption}
This is Assumption 2.5 in \cite{shin2020error} and Assumption 2.1 in \cite{mishra2020estimates}, which are two theory papers on PINN. In particular, the function space $(X, \Vert \cdot \Vert_X)$ in \cite{shin2020error} becomes the function space of all neural network functions with the $L_2$ norm, $(Y, \Vert \cdot \Vert_Y)$ becomes the $L_2$ function space $(\Omega, \Vert \cdot \Vert_{L_2(\Omega)})$, and $(Z, \Vert \cdot \Vert_Z)$ becomes the $L_2$ function space $(\partial\Omega, \Vert \cdot \Vert_{L_2(\partial\Omega)})$. 

Intuitively, Assumption \ref{assumption:L2} specifies the well-posedness of the PDE problem, and justifies the motivation of minimizing the boundary and residual losses. 

To show that Assumption \ref{assumption:L2} is realistic, various papers have proved that Assumption \ref{assumption:L2} holds for various PDEs. \cite{mishra2020estimates} proves that our Assumption \ref{assumption:L2} holds for Poisson equation, heat equation, wave equation and Stokes equation. \cite{Ryck2021ErrorAF} proves that Assumption \ref{assumption:L2} holds for Kolmogorov equations that include the heat equation and Black-Scholes equation as special cases. \cite{de2022error} proves Assumption \ref{assumption:L2} holds for incompressible Navier-Stokes equations. Lastly, \cite{lu2021priori} proves that optimization of the Deep Ritz Method objective contributes to $L_2$ error, on Poisson equation and static Schrödinger equation. Therefore, we refer the readers to \cite{lu2021priori, mishra2020estimates, Ryck2021ErrorAF, de2022error} for concrete PDE examples satisfying the assumption.

From the practice side \cite{raissi2019physics}, the motivation of the boundary and residual losses is to embed the physical law governed by the PDE and the data into the neural networks. Empirically, as long as the PDE problem is well-posed (i.e., it satisfies Assumption \ref{assumption:L2}) and that data are sufficient (which means the corresponding bound will be small), then PINNs can solve the PDE problem with small error. Assumption \ref{assumption:L2} is basically the foundation of PINNs.

The following theorem bridges the boundary and residual losses with the L2 error.
\begin{theorem}\label{thm:L2_generalization}
Let Assumption \ref{assumption:L2} holds, then for all neural networks parameterized by $\theta$, we can connect their $L_2$ generalization errors with the boundary + residual losses as follows:
\begin{equation}
\Vert u_{\theta} - u \Vert_{L_2(\Omega)} \leq \sqrt{2}C_1^{-1} \left(R_{D\cap\Omega}(\theta) + R_{D\cap\partial\Omega}(\theta)\right)^{1/2}.
\end{equation}
\end{theorem}
Here, we can use both the prior bound (Theorem \ref{thm:generalization}) and the posterior bound (Theorem \ref{thm:post_generalization}) to evaluate the $L_2$ distance between solution of PINN or XPINN and the true solution of PDEs.

\subsection{Comparing XPINN and PINN by Theorem \ref{thm:generalization}} 
In this subsection, we compare PINNs with XPINNs by using the generalization bound in Theorem \ref{thm:generalization}. We focus on comparing PINN with XPINN on the residual loss, i.e. $R_{D\cap\Omega}$ and $R_{S\cap\Omega}$, because it is more representative of the differentiated nets in PINNs. The case for boundary loss is similar but simpler, which is included in the Appendix.
Specifically, the comparison is performed by computing their respective theoretical bounds.
In particular, the generalization performance of PINN depends on the upper bound in Theorem \ref{thm:generalization}, which is: $R_{S \cap \Omega}(\bt^*)+8\left(\Vert u^* \Vert_{\mathcal{W}^L(\Omega)}\right)^3{C(h,K)\log n_r}/{{\sqrt{n_r}}} + 2\sqrt{{2(\log(2/\delta))}/{n_r}}$,
where $n_r$ is the number of residual training points.

For XPINN's generalization, we can apply Theorem \ref{thm:generalization} to each of the subdomains in the XPINN. Specifically, for the $i$-th sub-net in the $i$-th subdomain of XPINN, i.e., the $\Omega_i, i\in\left\{1,2,...,N_D\right\}$, its generalization performance is upper bounded by
$R_{S \cap \Omega_i}(\bt^*)+8\left(\Vert u^* \Vert_{\mathcal{W}^L(\Omega_i)}\right)^3{C(h,K)\log n_{r,i}}/{{\sqrt{n_{r,i}}}} + 2\sqrt{{2(\log(2/\delta))}/{n_{r,i}}}$,
where $n_{r,i}$ is the number of training boundary points in the $i$-th subdomain. 

Hence, since the $i$-th subdomain has $n_{r,i}$ training boundary points and is in charge of the prediction of $\frac{n_{r,i}}{n_r}$ proportion of testing data, we weight-averaged their generalization errors to get the generalization error of XPINN:
$\sum_{i=1}^{N_D} \frac{n_{r,i}}{n_r}(R_{S \cap \Omega_i}(\bt^*)+8\left(\Vert u^* \Vert_{\mathcal{W}^L(\Omega_i)}\right)^3{C(h,K)\log n_{r,i}}/{{\sqrt{n_{r,i}}}} + 2\sqrt{{2(\log(2/\delta))}/{n_{r,i}}})$.
If we omit the last term and assume that their empirical losses are similar, i.e.,
$R_{S \cap \Omega} \approx \sum_{i=1}^{N_D}({n_{r,i}}/{n_r})R_{S \cap \Omega_i}$, and $\sqrt{{2\log(2/\delta)}/{n_{r,i}}} \ll \Vert u^* \Vert_{\mathcal{W}^L(\Omega)}, \Vert u^* \Vert_{\mathcal{W}^L(\Omega_{i})}$,
then comparing the generalization ability of PINN and XPINN reduces to the following:
\begin{equation} \label{eq:bar_compare}
{\Vert u^* \Vert^3_{\mathcal{W}^L(\Omega)}} {\text{(PINN)}} \ \qquad \text{versus} \qquad {\sum_{i=1}^{N_D}\frac{\log n_{r,i}\sqrt{n_{r,i}}}{\log n_r\sqrt{n_r}}\Vert u^* \Vert^3_{\mathcal{W}^L(\Omega_i)}} {\text{(XPINN)}},
\end{equation}
where the model having smaller corresponding quantity is more generalizable.

In the next section, we will present three analytic examples and adopt the above comparison method to illustrate the circumstances under which XPINN is better or worse than PINN.

\subsection{Comparing XPINN and PINN by Theorem \ref{thm:post_generalization}}
The comparison using Theorem \ref{thm:post_generalization} is also done via computing their respective theoretical bounds. The residual loss is considered in the main text while the case for boundary loss is included in the supplementary material. Concretely, we denote the upper bound of PINN testing loss as $B_{\text{PINN}}$ and those of the sub-net $i$ in XPINN as $B_{i,\text{XPINN}}$, $i \in \left\{1,2,...,N_D\right\}$, which are provided by the right sides of Theorem \ref{thm:post_generalization}, i.e., the bounds are
\begin{equation}
\begin{aligned}
&\quad B_{\text{PINN}} = \frac{64K + 32d(L-1)K}{n_r\sqrt{n_r}} + 2 \sqrt{\frac{\log(2/\delta(M,N))}{2n_r}} +  \frac{144K\sqrt{d\log(2h^2)}\log n_r}{\sqrt{n_r}}\cdot\\
&\quad \prod_{l=1}^L M(l) \left(\sum_{l=1}^L N(l)^{2/3}\right)^{3/2} \left[ 1 + \sqrt{2} L\prod_{l=1}^L M(l)+ \sqrt{2} d(L^2-1)\left(\prod_{l=1}^L M(l)\right)^2 \right],
\end{aligned}
\end{equation}
and
\begin{equation}
\begin{aligned}
&\quad B_{i,\text{XPINN}} = \frac{64K + 32d(L-1)K}{n_{r,i}\sqrt{n_{r,i}}} + 2 \sqrt{\frac{\log(\frac{2}{\delta(M_i,N_i)})}{2n_{r,i}}} +  \frac{144K\sqrt{d\log(2h^2)}\log n_{r,i}}{\sqrt{n_{r,i}}}\cdot\\
&\quad \prod_{l=1}^L M_i(l) \left(\sum_{l=1}^L N_i(l)^{2/3}\right)^{3/2} \left[ 1 + \sqrt{2} L\prod_{l=1}^L M_i(l)+ \sqrt{2} d(L^2-1)\left(\prod_{l=1}^L M_i(l)\right)^2 \right].
\end{aligned}
\end{equation} 
Specifically, we assume that all sub-PINNs as well as the PINN model use neural networks with depth $L$ and width $h$. In the bound of PINN, $n_r$ is the total number of residual training samples. $M(l) = \lceil \Vert \boldsymbol{W}^l \Vert_{2} \rceil$, and $N(l) = \lceil \Vert \boldsymbol{W}^l \Vert_{2,1} / \Vert \boldsymbol{W}^l \Vert_{2} \rceil$, where $\boldsymbol{W}^l$ is the $l$-th layer parameter matrix in the PINN model. Moreover, in the bound of XPINN, $n_{r,i}$ is the number of residual training samples in subdomain $i$. $M_i(l) = \lceil \Vert \boldsymbol{W}^l_i \Vert_{2} \rceil$, and $N_i(l) = \lceil \Vert \boldsymbol{W}^l_i \Vert_{2,1} / \Vert \boldsymbol{W}^l_i \Vert_{2} \rceil$, where $\boldsymbol{W}^l_i$ is the $l$-th layer parameter matrix of the $i$-th subnet in the XPINN model.
Because the $i$-th sub-net in XPINN is in charge of the prediction of $\frac{n_{r,i}}{n_r}$ proportion of testing data, we weight-averaged their bounds to get that of XPINN, i.e., $B_{\text{XPINN}} = \sum_{i=1}^{N_D} ({n_{r,i}}/{n_r})B_{i,\text{XPINN}}$ where $B_{\text{XPINN}}$ is the bound for XPINN. Thus, we only need to compare $B_{\text{PINN}}$ with $B_{\text{XPINN}}$, where the model having smaller corresponding quantity is more generalizable.
These quantities can be directly measured and calculated from the trained deep nets, which allows for easy numerical validation. Thus, we validate this comparison method in computational experiments.

\subsection{Comparing XPINN and PINN by Theorem \ref{thm:L2_generalization}}
In this subsection, we compare PINNs with XPINNs on the $L_2$ error by using the generalization bound in Theorem \ref{thm:L2_generalization}. 
Specifically, the comparison is performed by computing their respective theoretical bounds.
In particular, we have already shown how to compare PINN and XPINN based on the boundary and residual losses. Denote the boundary and residual bounds of PINN (XPINN) as $B_{\text{PINN}}^{\text{boundary}}$ ($B_{\text{XPINN}}^{\text{boundary}}$) and $B_{\text{PINN}}^{\text{residual}}$ ($B_{\text{XPINN}}^{\text{residual}}$). Then, we only need to compare the following quantites:
\begin{equation}
\left(B_{\text{PINN}}^{\text{boundary}} + B_{\text{PINN}}^{\text{residual}}\right)^{1/2} (\text{PINN}) \quad \text{versus} \quad \left(B_{\text{XPINN}}^{\text{boundary}} + B_{\text{XPINN}}^{\text{residual}}\right)^{1/2} (\text{XPINN}).
\end{equation}

\section{Analytical Examples Based on Theorem \ref{thm:generalization}}
In this section, we provide analytical examples to further analyze the prior generalization bound in Theorem \ref{thm:generalization}. The examples ensure analytical expressions of the Barron norm, which results in precise calculation of prior bounds. Specifically, we show in what cases XPINNs are better than, similar to, and worse than PINNs in order to demonstrate the tradeoff in XPINN generalization.

\begin{figure}[htbp]
\centering
\includegraphics[scale=0.39]{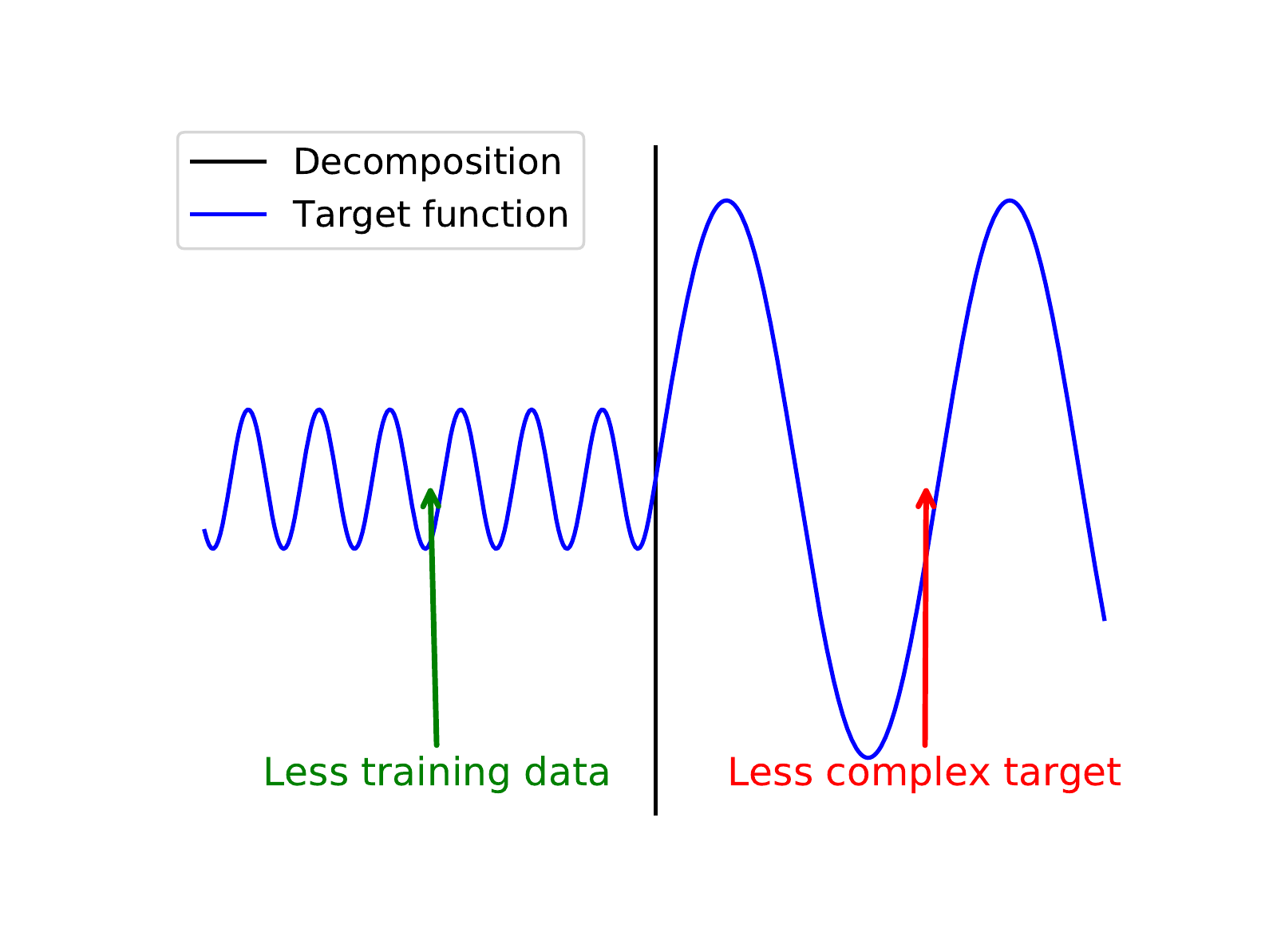}
\includegraphics[scale=0.39]{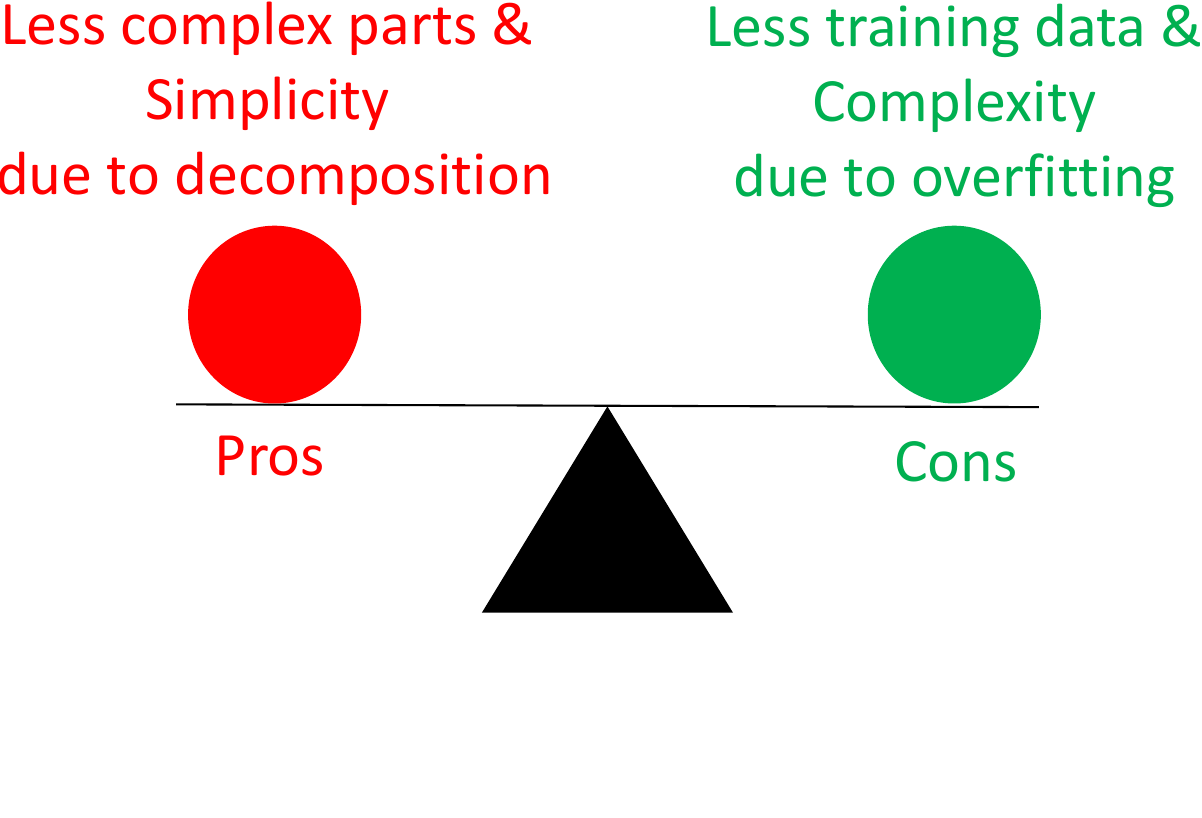}
\caption{Left: Decomposition causes simpler target in each part (red), but less
training data leads to overfitting. Right: Tradeoff between the two factors.}
\label{fig:XPINN}
\end{figure}

\subsection{Case where XPINN Outperforms PINN} \label{sec:1}
Let us consider the target function $u^*(x, y) = 2 \sin x + \sin y$, on the broken line $\Omega = \Omega_1 \cup \Omega_2 = [0,1]\times \left\{0\right\} \cup \left\{0\right\} \times [0,1]$. Obviously, we have $\Vert u^* \Vert_{\mathcal{W}^2(\Omega)}=3$, recall that $\mathcal{W}^2(\Omega)$ is the natural function space of two-layer sine networks on $\Omega$. Interestingly, if we restrict $u^*(x,y)$ to $\Omega_1$, we have $u^*(x,y)=2\sin x$ on $\Omega_1$, with a Barron norm $\Vert u^* \Vert_{\mathcal{W}^2(\Omega_1)}=2$. Similarly, if we restrict our observation to $\Omega_2$, we have $\Vert u^* \Vert_{\mathcal{W}^2(\Omega_1)}=1$. Since the lines $\Omega_1$ and $\Omega_2$ have the same length, it is natural to assume that the numbers of train residual data points on $\Omega_1$ and $\Omega_2$ are the same, i.e. $n_{r,1}=n_{r,2}=n_r/2$ in equation (\ref{eq:bar_compare}). In addition, we also assume that there is sufficient training data, with $\log n_r \gg \log 2$. We compare a PINN on $\Omega$ and an XPINN with two sub-nets on $\Omega_1$ and $\Omega_2$, respectively. Applying Theorem \ref{thm:generalization} and our discussion in section 3.1,
$27=\Vert u^* \Vert_{\mathcal{W}^2(\Omega)}^3  \gtrsim (\Vert u^* \Vert_{\mathcal{W}^2(\Omega_1)}^3+\Vert u^* \Vert_{\mathcal{W}^2(\Omega_2)}^3)/\sqrt{2} = {(8+1)}/{\sqrt{2}} = {9}/{\sqrt{2}}$.
Thus, XPINN generalizes better than PINN.

The underlying reason is as follows: $u^*$ remains simple on $\Omega_2$, while being complex on $\Omega_1$. Since XPINN optimizes several sub-nets at different subdomains, XPINN may learn a simple network on $\Omega_2$ where the solution is simpler, (i.e., very small $\Vert u^* \Vert_{\mathcal{W}^2(\Omega_2)}=1$). Also, it learns a complex network in other subdomains, where the solution is relatively complicated, (i.e., relatively large $\Vert u^* \Vert_{\mathcal{W}^2(\Omega_1)}=2$). Therefore, compared with PINN, which learns a very complex network on the whole domain $\Omega$ (extremely large $\Vert u^* \Vert_{\mathcal{W}^2(\Omega)}=3$), XPINN tends to have lower overall complexity, because it is complex on only part of the domain and remains simple on the rest of the domain, which leads to better generalization.

\subsection{Case where XPINN is Worse Than PINN} \label{sec:2}
Let us consider the same target function $u^*(x, y) = 2 \sin x + \frac{1}{2}\sin y$, on a different broken line $\Omega = \Omega_1 \cup \Omega_2$, where $\Omega_1 = [0,1]\times \left\{0\right\}$ and $ \Omega_2 = \left\{(x,y)|y=x,x\in[0,{\sqrt{2}}/{2}]\right\}$. Obviously, we have $\Vert u^* \Vert_{\mathcal{W}^2(\Omega)}=2.5$. Moreover, if we restrict $u^*(x,y)$ to $\Omega_1$, we have $u^*(x,y)=2\sin x$ on $\Omega_1$, with a Barron norm $\Vert u^* \Vert_{\mathcal{W}^2(\Omega_1)}=2$. However, even if we restrict our observation to $\Omega_2$, we still have $\Vert u^* \Vert_{\mathcal{W}^2(\Omega_2)}=2.5$. Since the lines $\Omega_1$ and $\Omega_2$ have the same length, it is natural to assume that the numbers of train residual data points on $\Omega_1$ and $\Omega_2$ are the same, i.e. $n_{r,1}=n_{r,2}=n_r/2$ in equation (\ref{eq:bar_compare}). We compare a PINN on $\Omega$ and an XPINN with two sub-nets on $\Omega_1$ and $\Omega_3$ respectively. Applying Theorem \ref{thm:generalization} and following our discussion in section 3.1, we have
$15.625=\Vert u^* \Vert_{\mathcal{W}^2(\Omega)}^3 \lesssim (\Vert u^* \Vert_{\mathcal{W}^2(\Omega_1)}^3+\Vert u^* \Vert_{\mathcal{W}^2(\Omega_2)}^3)/\sqrt{2} = {(2^3+2.5^3)}/{\sqrt{2}} = {23.625}/{\sqrt{2}} \approx 16.70$.
Thus, in this example XPINN is worse than PINN.

Although XPINN decreases the target function complexity via decomposition, at least on $\Omega_1$ it decreases to 2 from 2.5, it cannot complement the overfitting of less available training data on generalization, which is reflected in the $\sqrt{1/n_r}$ term in the bound, where $n_r$ is the number of residual training samples, i.e. the complexity grows with less data. Unfortunately, in this example, the more complexity brought by overfitting due to less data exceeds the benefit of simpler target function parts after decomposition. Hence, XPINN performs worse than PINN.

\subsection{Illustration of a Tradeoff in XPINN generalization} \label{sec:3}
In this section, we summarize the above two examples and derive a tradeoff in XPINN generalization, which is illustrated in Figure \ref{fig:XPINN}. There are two factors that counter-balance each other to affect XPINN generalization, namely the simplicity of decomposed target function within each subdomain thanks to domain decomposition, and the complexity and inclination to overfit due to less available training data, which counter-balance each other as follows. When the former is more dominant, XPINN outperforms PINN, as in our  example in section \ref{sec:1}. Otherwise, PINN outperforms XPINN,  as in our example in section \ref{sec:2}. When the two factors reach a balance, XPINN and PINN perform similarly.

To make the idea clearer, we consider another analytical example. Let us consider the target function $u^*(x, y) = 2 \sin x + q\sin y$, on the broken line $\Omega = \Omega_1 \cup \Omega_2$, where $\Omega_1 = [0,1]\times \left\{0\right\}$ and $ \Omega_2 = \left\{(x,y)|y=x,x\in[0,{\sqrt{2}}/{2}]\right\}$, where $q\in\mathbb{R}^+$ is a fixed constant to be decided. Obviously, we have $\Vert u^* \Vert_{\mathcal{W}^2(\Omega)}=2+q$. Further, if we restrict $u^*(x,y)$ to $\Omega_1$, we have $u^*(x,y)=2\sin x$ on $\Omega_1$, with a Barron norm $\Vert u^* \Vert_{\mathcal{W}^2(\Omega_1)}=2$. However, even if we restrict our observation to $\Omega_2$, we still have $\Vert u^* \Vert_{\mathcal{W}^2(\Omega_2)}=2+q$. Because the lines $\Omega_1$ and $\Omega_2$ have the same length, it is natural to assume that the numbers of train residual data points on $\Omega_1$ and $\Omega_2$ are the same, i.e. $n_{r,1}=n_{r,2}=n_r/2$ in equation (\ref{eq:bar_compare}). We compare a PINN on $\Omega$ and an XPINN with two sub-nets on $\Omega_1$ and $\Omega_2$, respectively. Applying Theorem \ref{thm:generalization} and following our discussion in Section 3.1, we need to compare the following quantities to determine when XPINN outperforms PINN:
$(2+q)^3=\Vert u^* \Vert_{\mathcal{W}^2(\Omega)}^3 \ \text{versus}\ (\Vert u^* \Vert_{\mathcal{W}^2(\Omega_1)}^3+\Vert u^* \Vert_{\mathcal{W}^2(\Omega_2)}^3)/\sqrt{2} = {(8+(2+q)^3)}/{\sqrt{2}}$.
When $(2+q)^3<8/(\sqrt{2}-1)\approx 19.31$, i.e., $q < 0.683$, PINN is better due to its more obvious effect with less data. When $q>0.683$, XPINN performs better due to the more obvious effect by decomposing complexity into simplicity.

In summary, inspired by the two analytical examples, we have shown in this section that there exists a tradeoff in XPINNs, which results in their different performance when compared to PINNs.
We further demonstrate these phenomena in computational experiments for various PDEs in the next section.

\section{Computational Experiments}

\begin{figure}[htbp]
\centering
\includegraphics[scale=0.6]{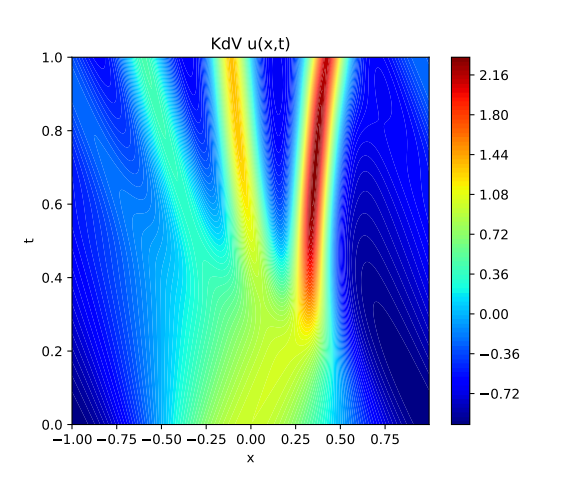}
\includegraphics[scale=0.4]{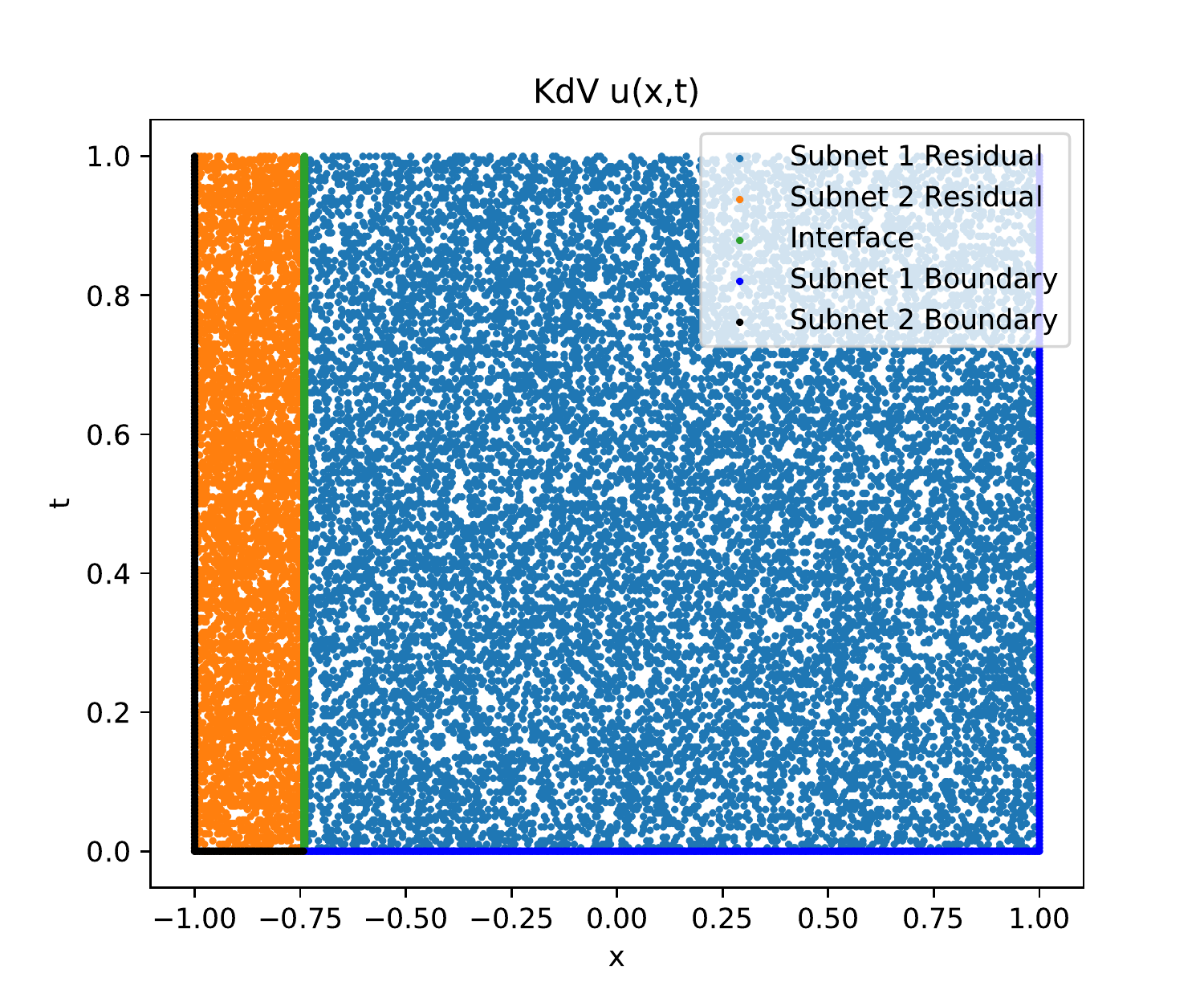}
\caption{Data visualization for the KdV experiment. Left: exact solution of the KdV equation. Right: Training points.}
\label{fig:kdv}
\end{figure}

\begin{figure}[htbp]
\centering
\includegraphics[scale=0.65]{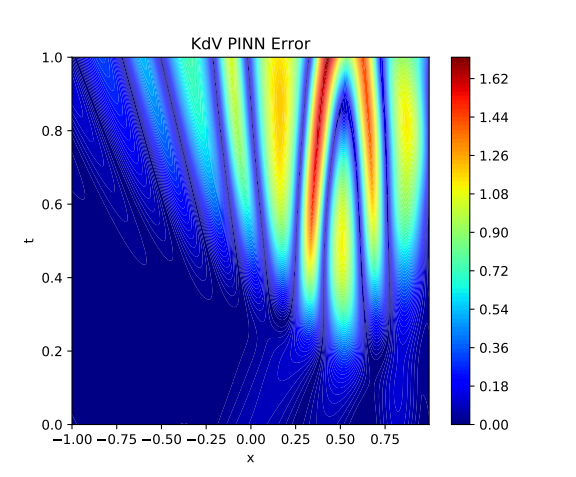}
\includegraphics[scale=0.65]{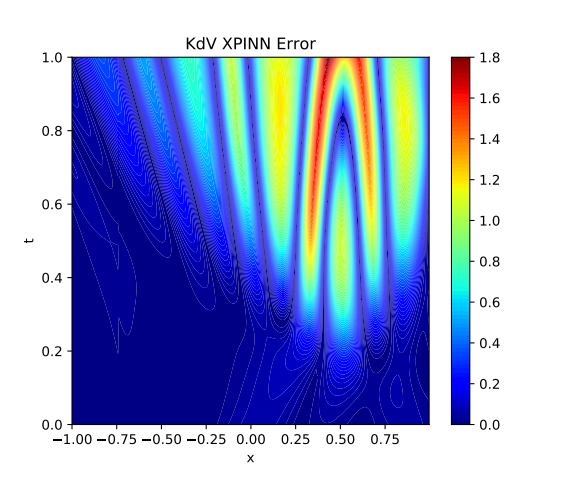}
\caption{Error visualization for the KdV experiment.}
\label{fig:kdv_error}
\end{figure}

\subsection{KdV Equation}
\subsubsection{Setup}
In this experiment, we consider a one-dimensional KdV equation given by
$u_t + u u_x = 0.0025 u_{xxx}, x \in [-1, 1], t \in [0, 1]$,
with the boundary condition of
$u(x, 0) = \cos (\pi x), x\in [-1,1]$ and the periodic initial condition.
The true solution is visualized in Figure \ref{fig:kdv} left. The entire dataset for this PDE is provided by the paper of PINN \cite{raissi2019physics} and CPINN \cite{jagtap2020conservative}.

Following the KdV equation experiment in CPINN \cite{jagtap2020conservative}, the training dataset for PINN contains 18000 residual points and 914 boundary points. The testing dataset for PINN contains 102400 points uniformly distributed within the domain. The backbone model for PINN is a 10-layer neural network with 20 hidden units activated by sine as in our theory. Adam \cite{kingma2014adam} optimizer with 1e-3 learning rate is used for optimization. No regularization is used. 

Moreover, from the solution of the KdV equation in Figure \ref{fig:kdv} left, we observe that it is complex, and that it fluctuates when $x \geq 0$ corresponding to the right part. In contrast, the left part is smoother and simpler.
Hence, to reflect our discussion on the prior bound and analytical examples, for XPINN we partition the whole domain into two subdomains, including (1) the right domain, corresponding to sub-net named XPINN-R as $x > -0.74$ and (2) the left domain, corresponding to sub-net named XPINN-L as $x \leq -0.74$. Then, the target function in subdomain 1 is complicated and it fluctuates significantly, while that in subdomain 2 is simpler and smoother. 

For XPINN, it is given 646 and 268 boundary points, 14000 and 4000 residual points, in subdomain 1 and 2, respectively. The number of interface points is 10000. The backbone models for XPINN are two 10-layer neural networks with 20 hidden units activated by sine as in our theory. Two Adam \cite{kingma2014adam} optimizers with 1e-3 learning rates are used for optimizations. No regularization is used.

For fair comparison, we keep the same training procedure, i.e., training epochs, learning rate, model structure, and weight decay, etc. We train each model for 5000 epochs, and the results reported in the table are those at the 5000-th epochs. 
For both models, we use unity weight for the residual loss and zero weight for the residual continuity loss, and also use unity weight for the boundary loss and the boundary interface loss.
For reproducibility, we run each model for 5 times using fixed random seed 0, 1, 2, 3, 4.

Lastly, we need to pay attention to the choice of $\delta$ in the posterior bound. 
Because the bound in Theorem \ref{thm:post_generalization} holds with probability $1-\delta$, in the computations we take $\delta=0.1$ so that the bound holds with probability 0.9. At the same time, for XPINN, we need to take $\delta=0.05$ because there are two sub-nets for a union bound.

\subsubsection{Results}
{\color{red}
\begin{table}[]
\centering
\caption{Computational results for KdV equation.}
\label{tab:kdv}
\begin{tabular}{|c|c|c|c|c|}
\hline
& Train Loss & Relative $L_2$ error & Complexity & Bound\\ \hline
PINN & 3.597e-3$\pm$7.194e-4 & {6.899e-1$\pm$8.015e-3} & 100.00\% & 100.00\% \\ \hline
XPINN-R & \multirow{2}{*}{5.619e-3$\pm$5.056e-3} & \multirow{2}{*}{{6.955e-1$\pm$9.905e-3}} & 101.31\% & \multirow{2}{*}{121.08\%} \\ \cline{1-1} \cline{4-4}
XPINN-L & & & 28.50\% & \\ \hline
\end{tabular}
\end{table}
}

We present the experimental results including train loss, and test relative L2 error, as well as the calculated theoretical generalization bound in XPINN and PINN models in Table \ref{tab:kdv}. Moreover, we also provide the product of norms of the neural network parameter matrices, which is directly linked to the Rademacher complexity of neural networks and thus their complexities (see Lemma \ref{lemma:rad_nn}). We compute this quantity to provide an intuitive observation on the complexities of the optimized neural networks.

In addition, the products of norms of the weight matrices are presented in the columns ``Complexity" where that of PINN is denoted 100\% for clear comparison. 
The complexity of PINN is similar to that of XPINN-R, while XPINN-R is more complex than XPINN-L. XPINN-R corresponds to the domain $x>-0.74$, where the solution is more complicated and XPINN-L corresponds to the simpler domain, validating the implicit regularization of gradient descent, which learns simple (complex) function with simple (complex) neural networks.
Although XPINN-R is fitting a less complex function, the available data in its subdomain 1 is less than those of PINN, so overfitting causes XPINN-R to be as complex as the PINN fitting the entire function.

Furthermore, the ``Bound" columns in Table \ref{tab:kdv} are the theoretical generalization bounds for PINN and XPINN, where that of PINN is denoted as 100\% for clarity. The theoretical generalization bound of XPINN is 121.08\%, which is slightly larger than PINN and consistent with their testing performances, i.e., the error of PINN (6.899e-1) is slightly better than that of XPINN (6.955e-1), justifying the effectiveness of our generalization bound. 
In Figure \ref{fig:kdv_error}, we visualize the errors of PINN and XPINN. PINN and XPINN have similar error distributions, which justifies their similar performances and bounds.
In conclusion, for the KdV equation, the positive effect of simplicity of target functions in every subdomain brought by the domain decomposition is similar to the negative overfitting effect caused by less available data in each subdomain, which leads to similar performance of XPINN than PINN overall.

\subsection{Heat Equation}
\subsubsection{Setup}

In this subsection, we consider the heat equation, which is a second order linear PDE. The one-dimensional heat equation under consideration is
$u_t = u_{xx}, x\in[-1,1], t\in[0,1]$,
where its boundary conditions on $(x, t) \in \{-1, 1\} \times [0,1]$, and its initial conditions on $(x, t) \in [0, 1] \times \{0\}$, are given by the ground truth solution
\begin{equation}
    u(x, t) = e^{-\pi^2 t}\cos(\pi x) + 0.6 e^{-4\pi^2 t}\cos(2\pi x) + 0.3e^{4t-4}\cosh(2x) + 0.1e^{t-1}\sinh x.
\end{equation}

\begin{figure}[htbp]
\centering
\includegraphics[scale=0.4]{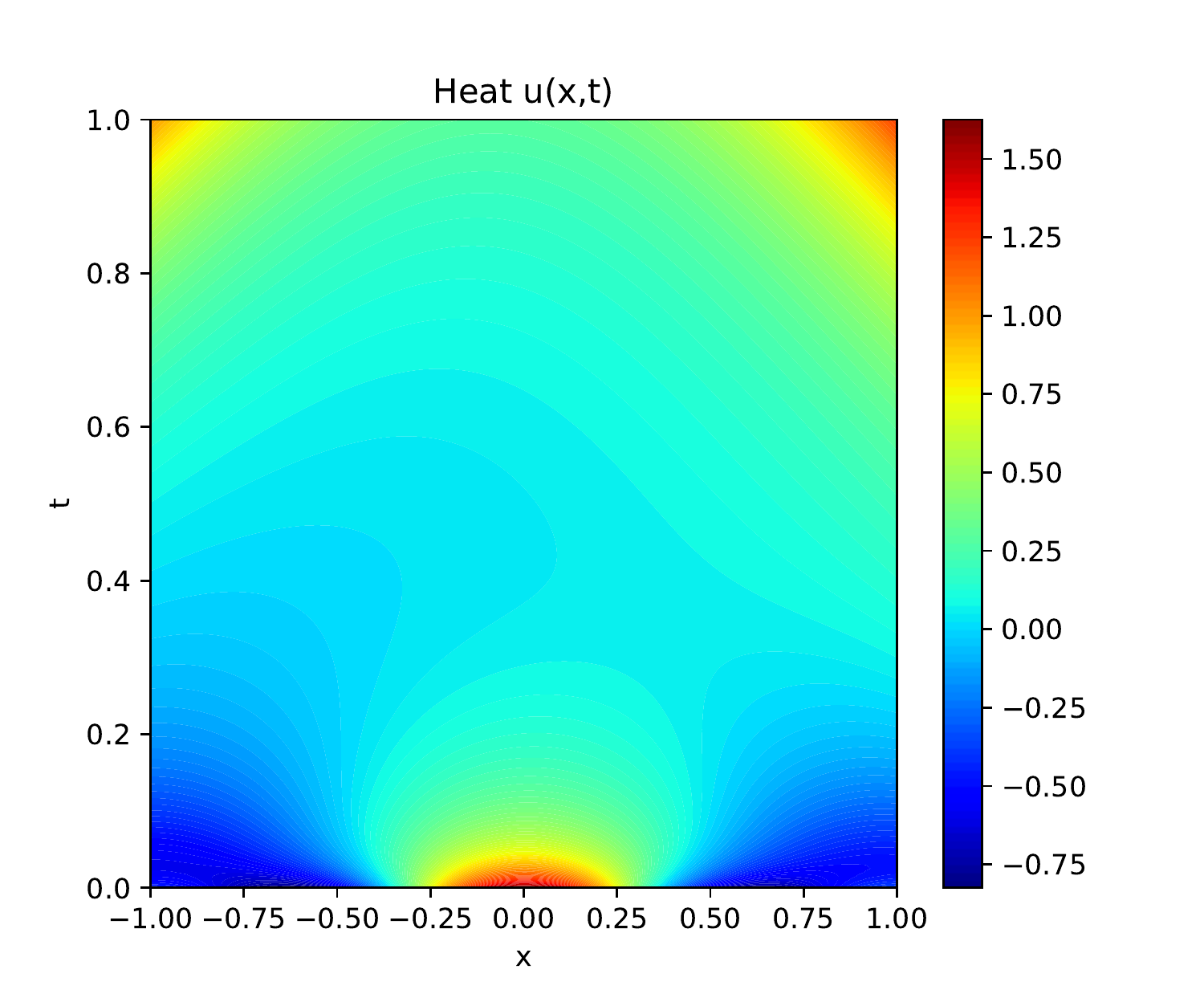}
\includegraphics[scale=0.4]{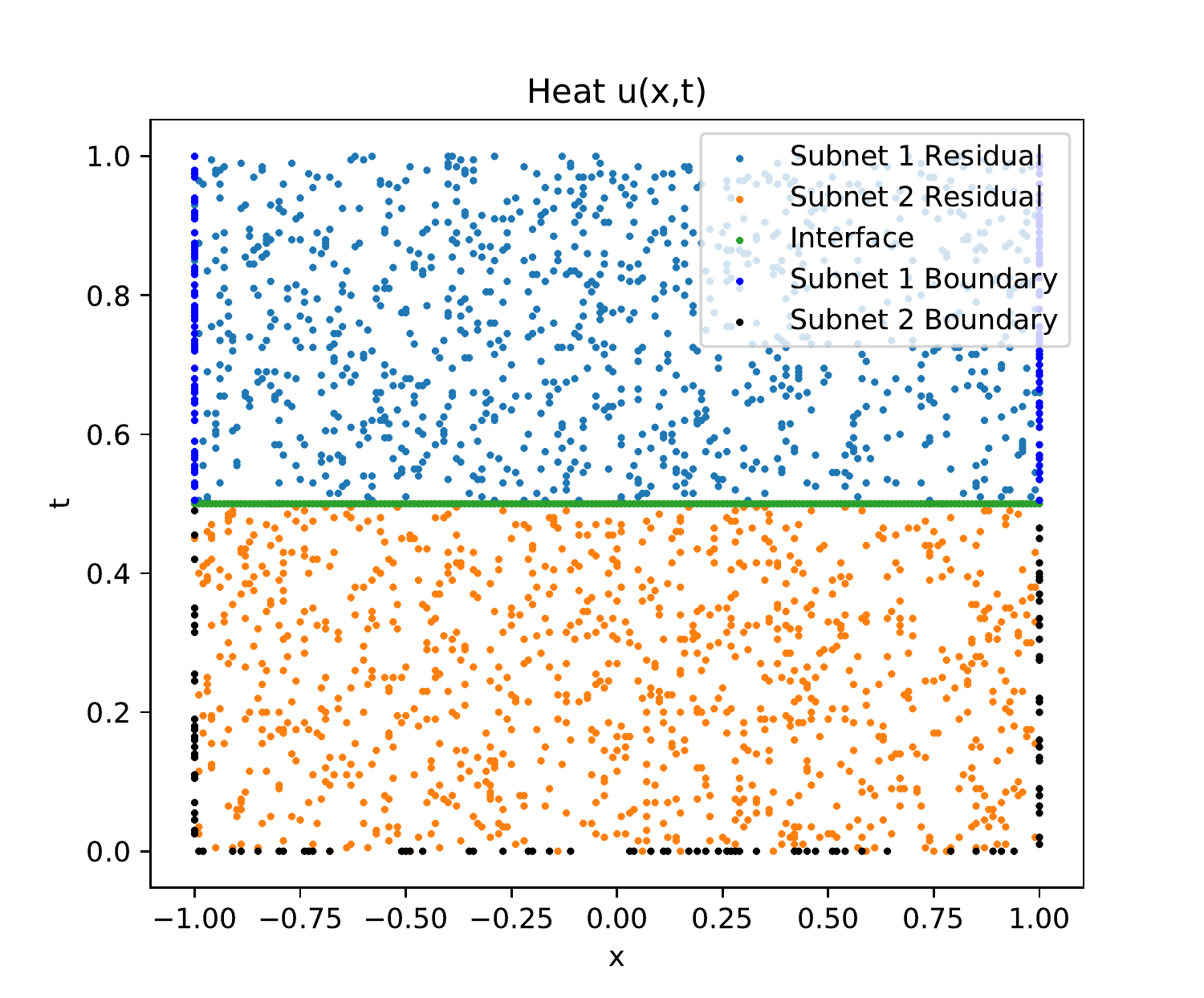}
\caption{Data visualization for the heat experiment. Left: exact solution of the heat equation. Right: Training points.}
\label{fig:heat1}
\end{figure}

\begin{figure}[htbp]
\centering
\includegraphics[scale=0.4]{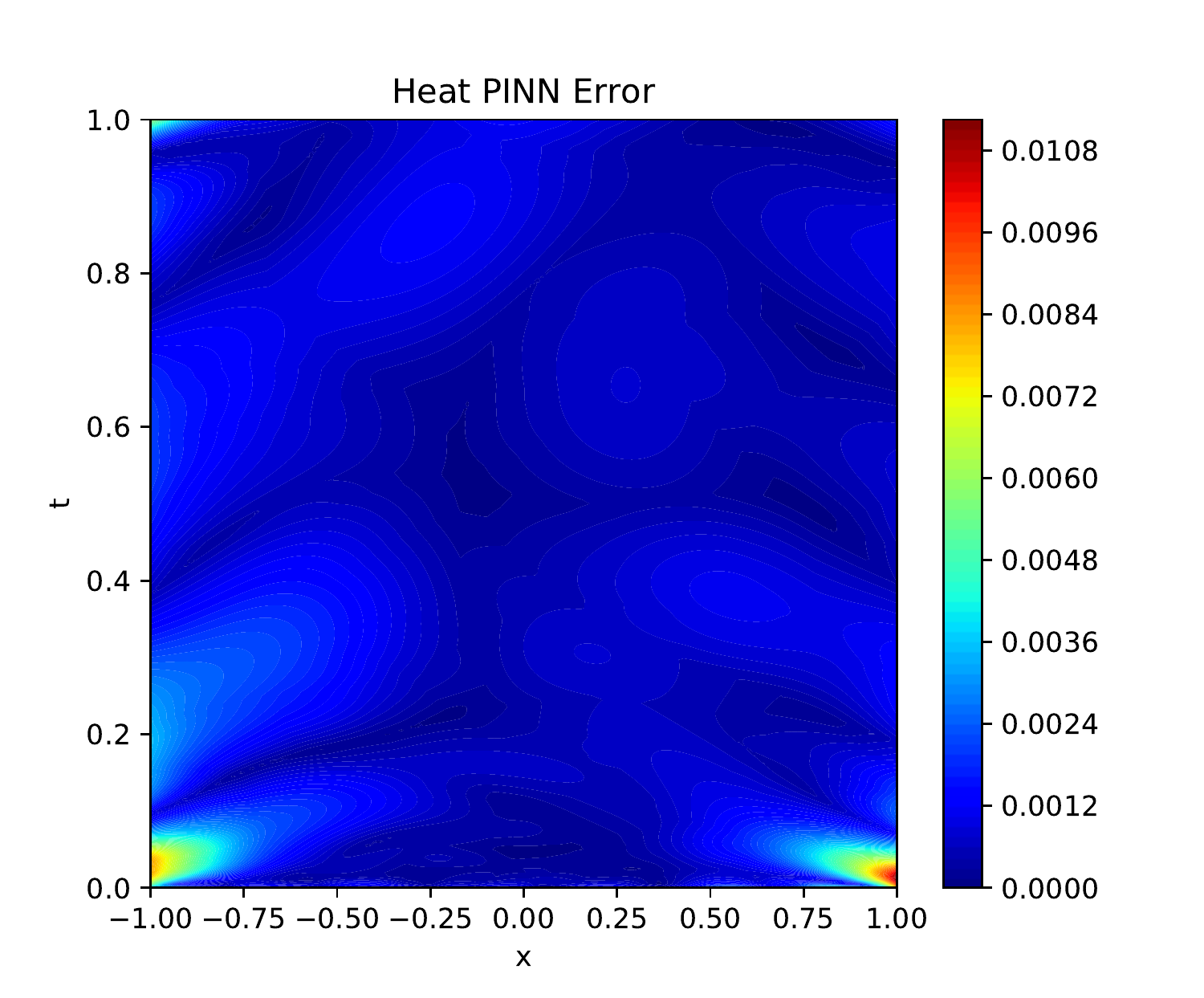}
\includegraphics[scale=0.4]{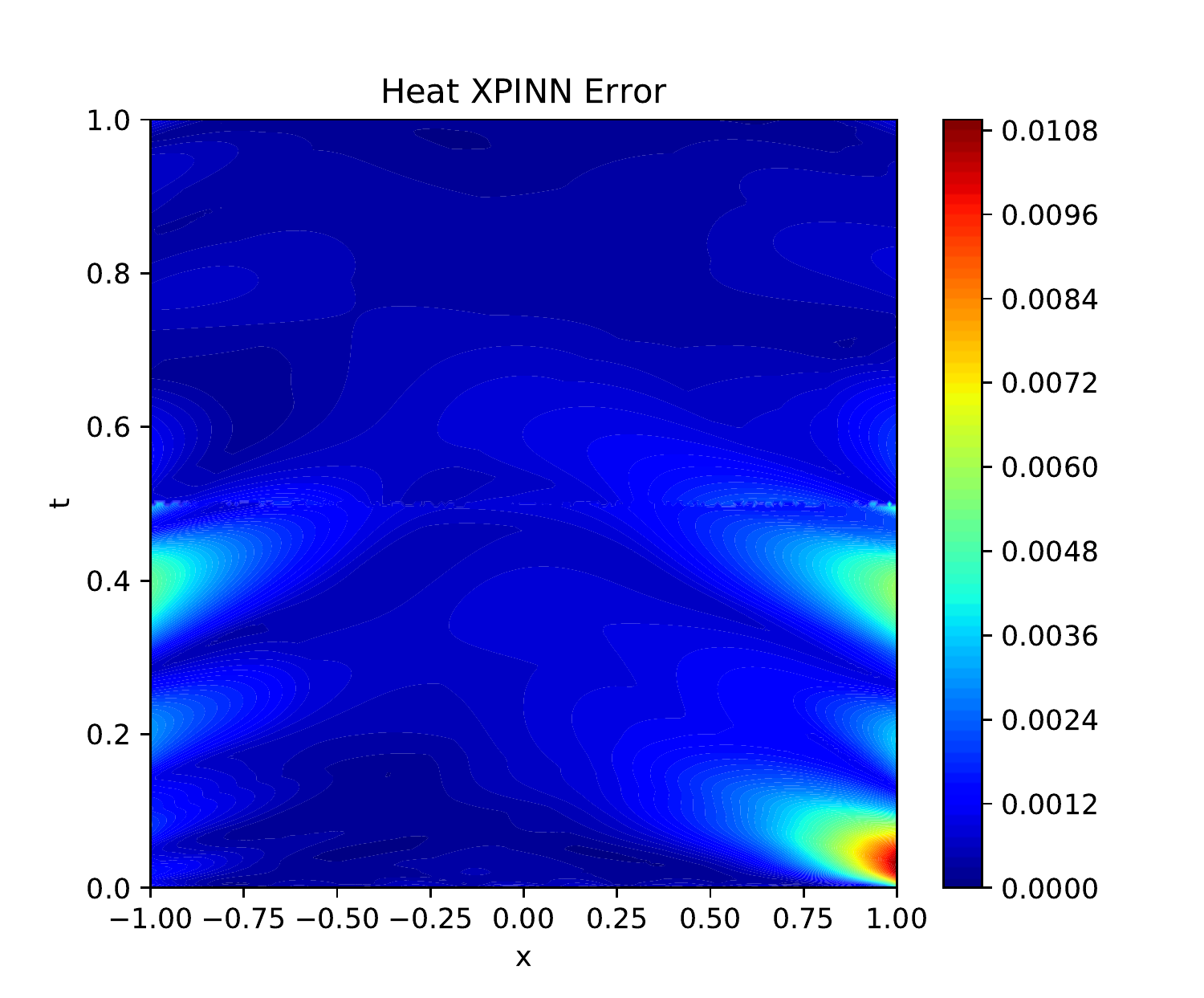}
\caption{Error visualization for the heat experiment.}
\label{fig:heat_error}
\end{figure}

\begin{table}[]
\centering
\caption{Computational results for the heat equation.}
\label{tab:heat}
\begin{tabular}{|c|c|c|c|c|}
\hline
 & Train Loss & {Relative $L_2$ error} & Complexity & Bound \\ \hline
PINN & 8.589e-5$\pm$2.218e-5 & {1.778e-3$\pm$2.195e-4} & 100.00\% & 100.00\% \\ \hline
XPINN-T & \multirow{2}{*}{2.585e-4$\pm$1.726e-4} & \multirow{2}{*}{{4.490e-3$\pm$1.517e-3}} & 156.24\% & \multirow{2}{*}{243.22\%} \\ \cline{1-1} \cline{4-4}
XPINN-B &  &  & 75.75\% &  \\ \hline
\end{tabular}
\end{table}

The solution is visualized in Figure \ref{fig:heat1}. The training dataset for PINN contains 2000 residual points and 200 boundary points, whereas the testing dataset for PINN contains 160801 points within the domain. The {backbone model for PINN} is a 9-layer neural network with 20 hidden units activated by tanh. LBFGS with 1e-1 learning rate is used. No regularization is used. 

For domain decomposition of XPINN, from Figure \ref{fig:heat1} left, we observe that the solution of heat equation is complex near $t=0$ and $t=1$, due to the two nearby heat sources. The two heat sources are also dissimilar: at $t=0$ the source is generated by trigonometric functions, while at $t=1$ it is generated by hyperbolic functions. To design a good XPINN, we should partition the complexities of the two heterogeneous heat sources into different subdomains. Thus, we partition the whole domain into a bottom domain $t\leq0.5$ containing trigonometric heat source, corresponding to XPINN-B, and a top one $t>0.5$ containing the hyperbolic source, corresponding to XPINN-T.

For fair comparison, we keep the same training procedure, i.e., training epochs, learning rate, model structure, and weight decay, etc. We train each model for 20000 epochs, which can train orders of magnitude to $1e-4$ to $1e-5$, and the results reported in the table are those at the 20000-th epochs. 
For both models, we use unity weight for the residual loss and the residual continuity loss, and use 20 weight for the boundary loss and the boundary interface loss, which has been adopted in the original code of XPINN \cite{jagtap2020extended}. 
For reproducibility, we run each model for 5 times using fixed random seed 0, 1, 2, 3, 4.

\subsubsection{Results}
Table \ref{tab:heat} shows the experimental results for the heat equation.
In the table, Train Loss denotes the training boundary plus residual losses, Relative L2 denote the relative L2 error of the model, Complexity denotes the products of norms of the weight matrices to quantify the network complexity, and Bound denotes the theoretical bound for the (relative) L2 error.

PINN (1.778e-3) outperforms XPINN (4.490e-3) in relative L2 error, and its bound (100.00\%) is also smaller than that of XPINN (243.22\%), i.e., our bound is consistent with numerical results. 
For complexity, the order is XPINN-T (156.24\%) $>$ PINN (100\%) $>$ XPINN-B (75.75\%). XPINN-T is the most complex because the top heat source is more complicated than the bottom one, and XPINN-T only has half of the training data. In this case, the negative effect due to overfitting caused by less data available is more obvious than the positive influence of less complex target function, so PINN performs better than XPINN.
In Figure \ref{fig:heat_error}, we visualize the errors of PINN and XPINN. The errors of both models are distributed mainly in the bottom part of the domain, where XPINN-B performs much worse than PINN in the corners of the bottom sub-domain, which may be due to limit data in this sub-domain, finally deteriorating the performance in this part.

\subsection{Advection Equation}
\subsubsection{Setup}
\begin{figure}[htbp]
\centering
\includegraphics[scale=0.4]{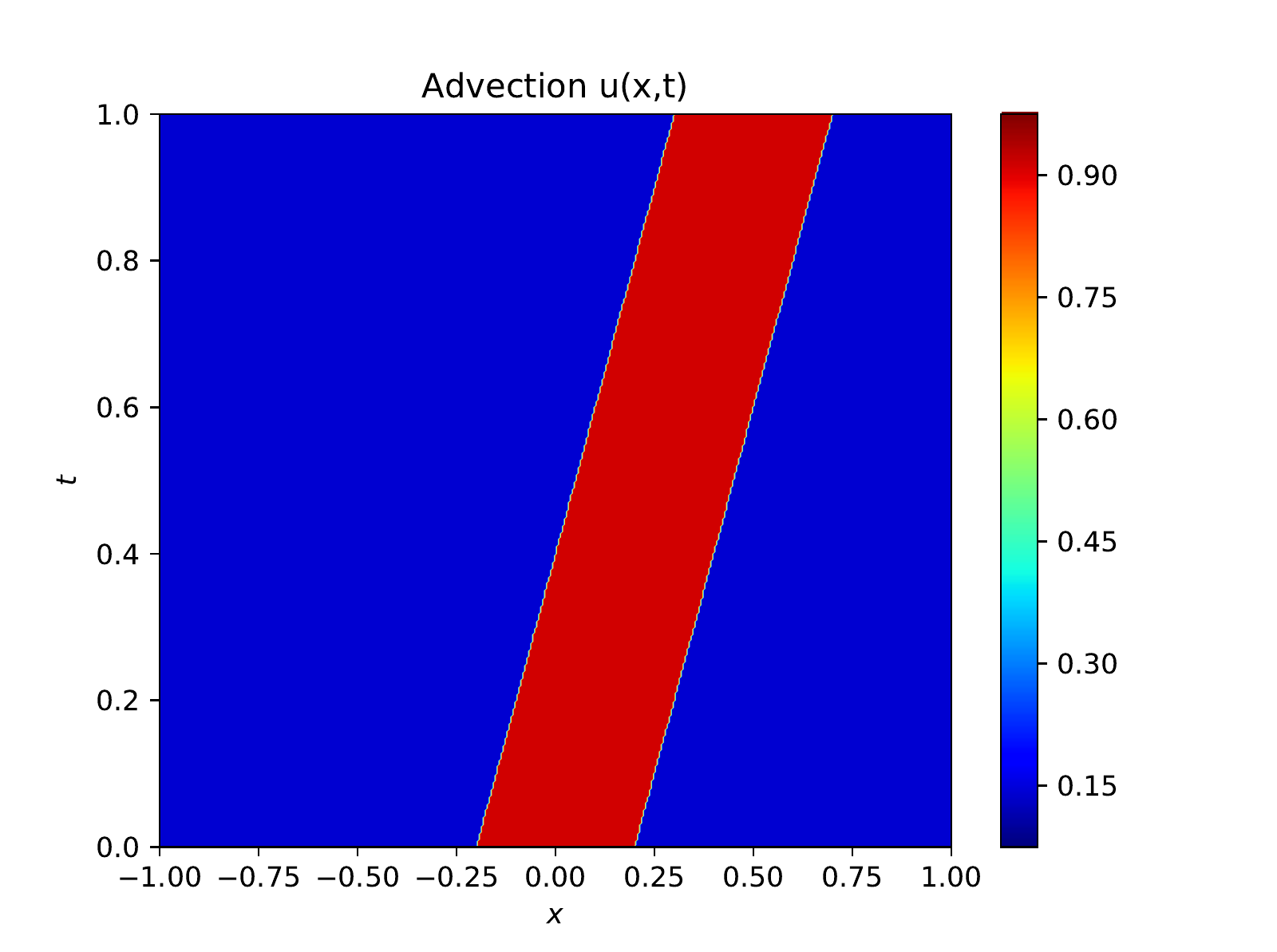}
\includegraphics[scale=0.4]{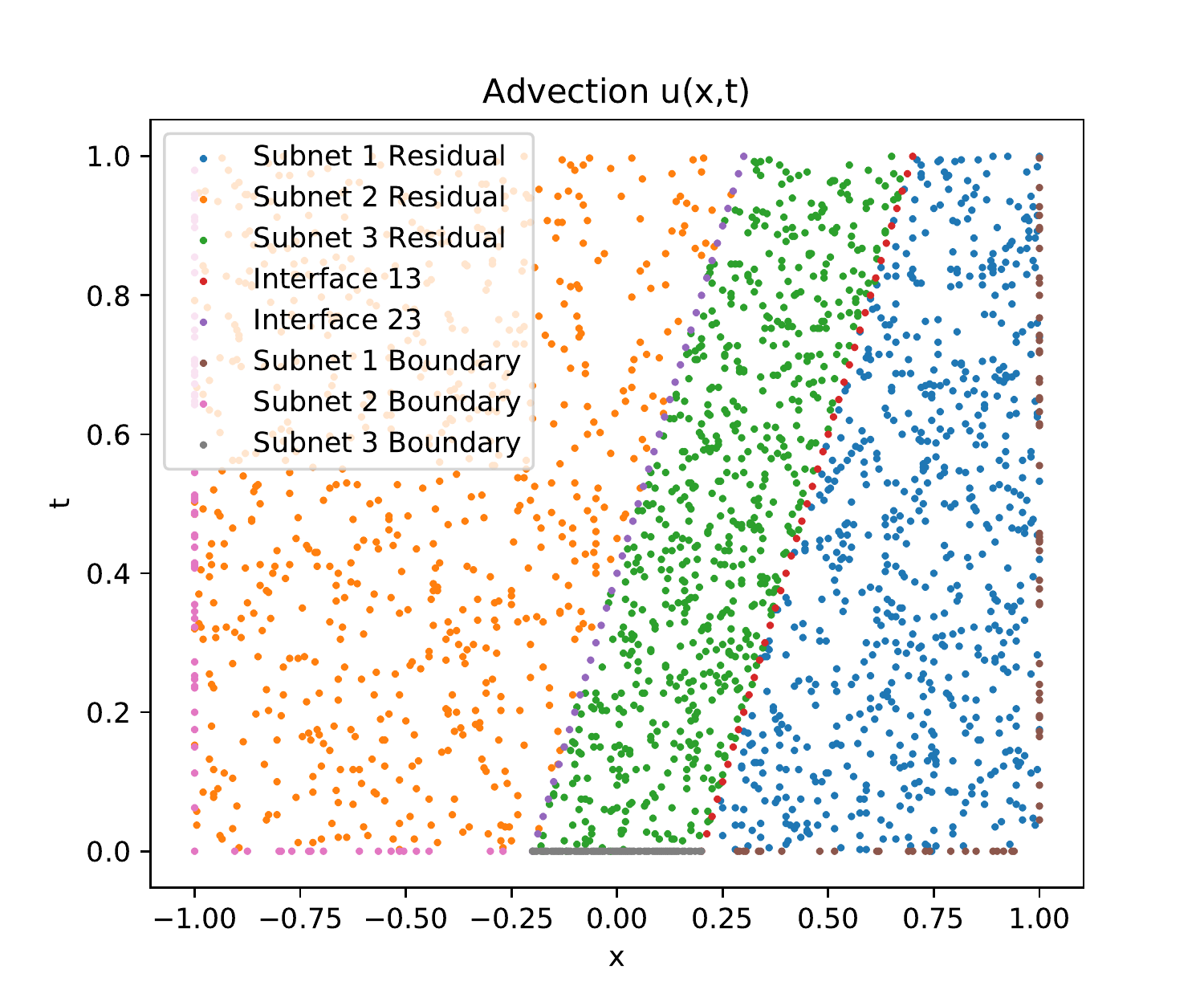}
\caption{Data visualization for the advection experiment. Left: exact solution of the advection equation. Right: Training points.}
\label{fig:advection}
\end{figure}

\begin{figure}[htbp]
\centering
\includegraphics[scale=0.65]{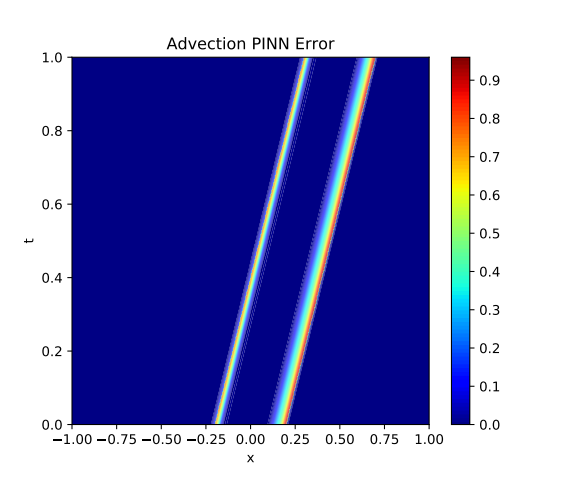}
\includegraphics[scale=0.65]{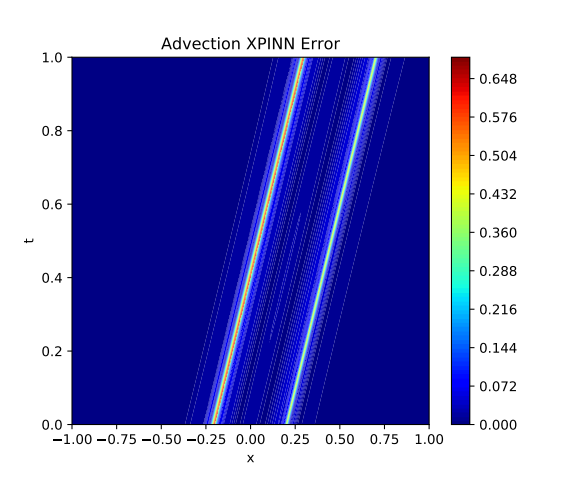}
\caption{Error visualization for the advection equation.}
\label{fig:advection_error}
\end{figure}

In this subsection, we consider the advection equation to show the difference between XPINN and PINN, which is given by $u_t + 0.5 u_x = 0, x\in[-1,1], t\in[0,1]$,
with the initial condition
$u(x,0)=1_{-0.2\leq x \leq 0.2}$.
The solution is presented in Figure \ref{fig:advection} left. The training dataset for PINN contains 2000 residual points and 200 boundary points. The {backbone model for PINN} is a 6-layer neural network with 20 hidden units activated by tanh. {Adam \cite{kingma2014adam} optimizer with 1e-3 learning rate} is used for optimization. No regularization is used. 

For XPINN, from Figure \ref{fig:advection} left, we observe that the solution of the advection equation can be divided into the following two parts, $\left\{-0.2 < x - 0.5t\right\}\cap\left\{x - 0.5t<0.2\right\}$ where $u=1$ and $\left\{x - 0.5t \geq 0.2\right\}$ and $\left\{x - 0.5 t \leq -0.2\right\}$ where $u=0$. Hence, in XPINN-LMR (left, middle, and right), we partition the domain into these three continuous parts mentioned above. XPINN-LMR seems to be a good partition since in each subdomain the target function is extremely simple constant function, while the whole function is complex and discontinuous. The visualization of domain decomposition is provided in Figure \ref{fig:advection} right.

For fair comparison, we keep the same training procedure. We train each model for 5000 epochs, and the results reported in the table are those at the 5000-th epochs. 
For both models, we use unity weight for the residual loss and zero weight for the residual continuity loss, and also use unity weight for the boundary loss and the boundary interface loss.
For reproducibility, we run each model for 5 times using fixed random seed 0, 1, 2, 3, 4.
\subsubsection{Results}

\begin{table}[]
\centering
\caption{Computational results for advection equation.}
\label{tab:advection}
\begin{tabular}{|c|c|c|c|c|}
\hline
Method & Train Loss & {Relative $L_2$ error} & Complexity & Bound \\ \hline
PINN & 1.387e-5$\pm$1.298e-5 & {2.052e-1$\pm$1.001e-1} & 100\% & 100\% \\ \hline
XPINN-L & \multirow{3}{*}{4.239e-3$\pm$2.385e-5} & \multirow{3}{*}{{1.617e-1$\pm$3.582e-2}} & 40.53\% & \multirow{3}{*}{66.59\%} \\ \cline{1-1} \cline{4-4}
XPINN-M &  &  & 53.16\% &  \\ \cline{1-1} \cline{4-4}
XPINN-R &  &  & 79.95\% &  \\ \hline
\end{tabular}
\end{table}

Table \ref{tab:advection} presents all computational results for the advection equation. XPINN (2.052e-1) performs better than PINN (1.617e-1) in relative L2 error, which is consistent with their theoretical bounds, i.e., the bound of XPINN (66.59\%) is also smaller than that of PINN (100\%).
The reason is revealed by the norms. All sub-nets in XPINN (40.53\%, 53.16\%,79.95\%) are less complicated than PINN (100\%), because in each subdomain of XPINN, the target function is constant, whose positive influence is more obvious than the overfitting due to less data. In sum, in the experiment of advection equation, since the target functions in every sub-domains are extremely simple constant functions, the positive influence of less complex target functions is much stronger than the negative effect of overfitting due to less available data in each sub-domains, thus, XPINN performs better than PINN.
Lastly, in Figure \ref{fig:advection_error} we visualize the errors of PINN and XPINN. The error mainly concentrates in the areas near the discontinuity part, where XPINN has smaller error than PINN.

\subsection{Poisson Equation}
\begin{figure}[htbp]
\centering
\includegraphics[scale=0.6]{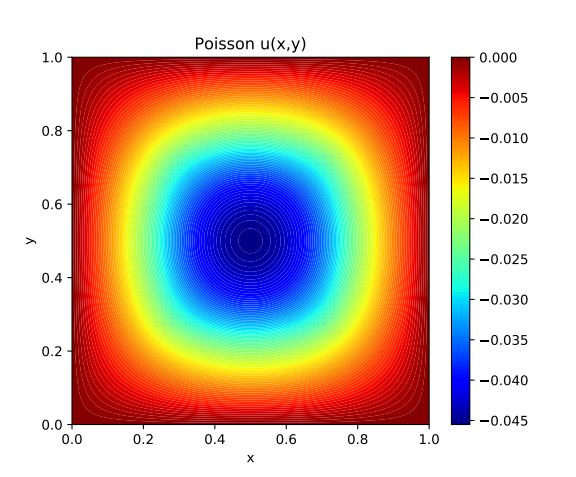}
\includegraphics[scale=0.4]{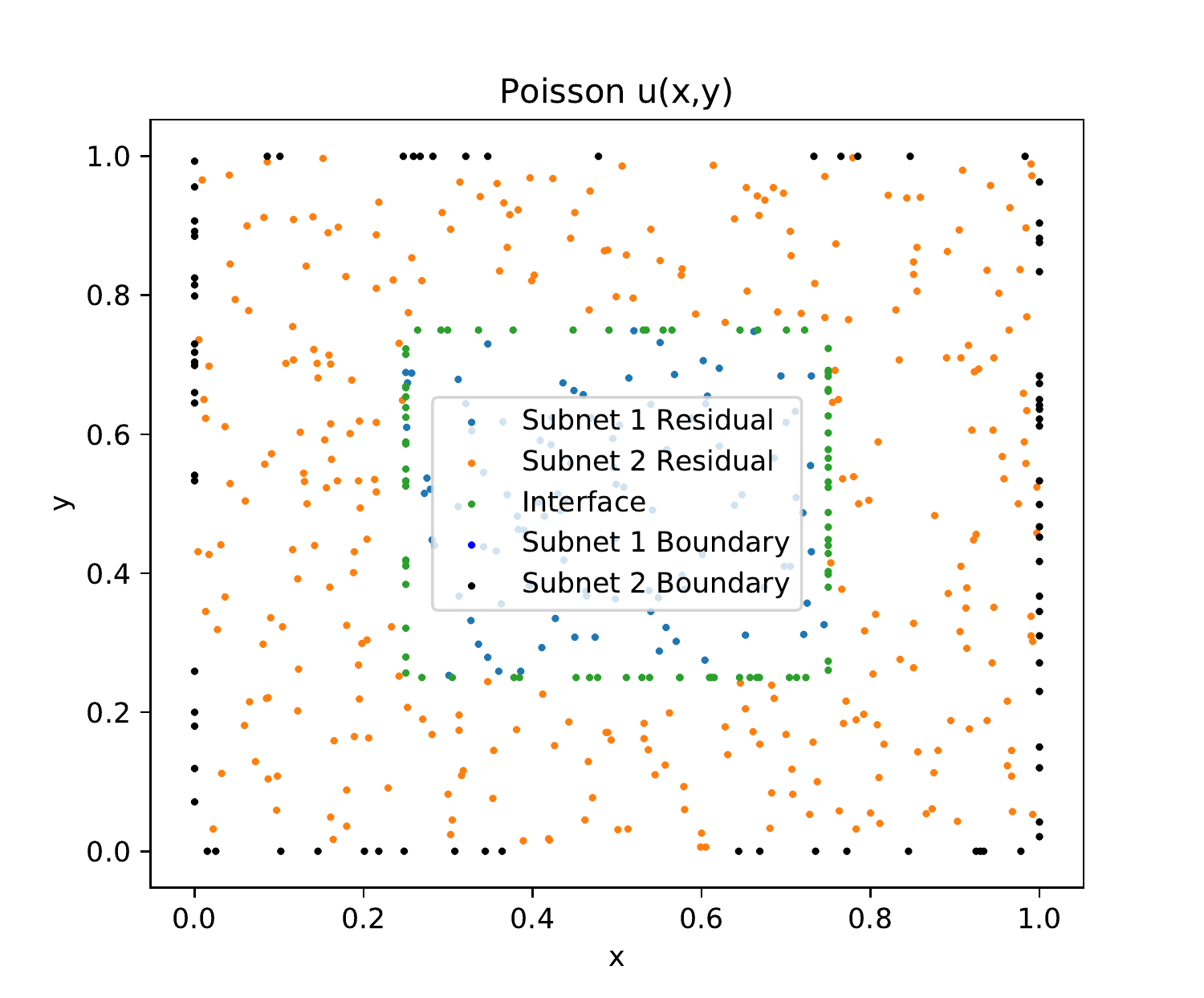}
\caption{Data visualization for the Poisson experiment. Left: exact solution of the Poisson equation. Right: Training points.}
\label{fig:poisson}
\end{figure}
\subsubsection{Setup}
In this subsection, we consider a Poisson equation with residual discontinuity, which is also a second order linear PDE. The equation under consideration is $u_{xx} + u_{yy} = f$, where $(x, y) \in [0, 1] \times [0, 1]$, and $f$ is given by $f(x, y) = 1$ for $(x,y)\in[0.25, 0.75] \times [0.25, 0.75]$, and $f(x, y) = 0$ for the rest of the domain. The boundary condition is zero.

The solution is visualized in Figure \ref{fig:poisson}. The training dataset for PINN contains 400 residual points and 80 boundary points, whereas the testing dataset for PINN contains 1,002,001 points within the domain. The backbone model for PINN is a 9-layer neural network with 20 hidden units activated by tanh. LBFGS with 1e-1 learning rate is used. No regularization is used. 

The weighting strategies for PINN and XPINNs are summarized in Table \ref{tab:poisson_w}. 
For PINN, we use unity weight for residual and 20 weight for boundary. 
For XPINN1, i.e., the simplest XPINN model, we use 20 weight for boundary, boundary interface, and residual interface, while using unity weight for residual. Due to the existence of residual discontinuity, we need to give stronger constraint to maintain the residual continuity. So, we use 20 weight for residual interface rather than unity weight used in the official XPINN code \cite{jagtap2020extended}.

In addition, to test the effective of the newly proposed regularization in \cite{de2022error}, we design the second XPINN2, which builds upon XPINN1, i.e., it adopts the same weights as XPINN1 and use 30 weight for the additional regularization on the first order derivatives near the interface.

For the last XPINN3, we shall see in the results that XPINN2 does not perform well on the boundary. So, we try to remedy XPINN2 by XPINN3 through increasing the weight for the boundary loss, i.e., from 20 to 80 weight on the boundary.

For domain decomposition of XPINN, sub-domain 1 contains the area $(x,y)\in[0.25, 0.75]\times [0.25, 0.75]$, while sub-domain 2 contains the rest of the domain. Specifically, we partition the domain according to the discontinuity, to force each sub-net in the XPINN to focus on one continuous part, rather than fitting the entire function containing residual discontinuity by one network. The sub-net for the sub-domain 1 is called XPINN-M since it is in the middle of the domain, while the sub-net for sub-domain 2 is called XPINN-A since it is around the entire domain.

For fair comparison, we keep the same training procedure, i.e., training epochs, learning rate, model structure, and weight decay, etc. We train each model for 20000 epochs, and the results reported in the table are those at the 20000-th epochs. For reproducibility, we run each model for 5 times using fixed random seed 0, 1, 2, 3, 4.

\begin{table}[]
\centering
\caption{Weighting strategies for the Poisson equation.}
\label{tab:poisson_w}
\begin{tabular}{|c|c|c|c|c|c|}
\hline
 & Residual & Interface R & Additional I & Boundary & Interface B \\ \hline
PINN & 1 & NA & NA & 20 & NA \\\hline
XPINN1 & 1 & 20 & 0 & 20 & 20 \\ \hline
XPINN2 & 1 & 20 & 30 & 20 & 20 \\ \hline
XPINN3 & 1 & 20 & 30 & 80 & 20 \\ \hline
\end{tabular}
\end{table}

\begin{table}[]
\centering
\caption{Computational results for the Poisson equation.}
\label{tab:poisson}
\begin{tabular}{|c|c|c|c|c|}
\hline
 & Train Loss & {Relative $L_2$ error} & Complexity & Bound \\ \hline
PINN & 2.688e-4$\pm$3.411e-4 & {5.553e-2$\pm$2.936e-2} & 100.00\% & 100.00\% \\ \hline
XPINN1-A & \multirow{2}{*}{1.181e-2$\pm$4.319e-3} & \multirow{2}{*}{{4.022e-1$\pm$1.648e-1}} & 142.71\% & \multirow{2}{*}{122.56\%} \\ \cline{1-1} \cline{4-4}
XPINN1-M &  &  & 297.91\% &  \\ \hline
XPINN2-A & \multirow{2}{*}{1.016e-2$\pm$3.713e-3} & \multirow{2}{*}{{1.387e-1$\pm$7.030e-3}} & 183.44\% & \multirow{2}{*}{108.57\%} \\ \cline{1-1} \cline{4-4}
XPINN2-M &  &  & 292.93\% &  \\ \hline
XPINN3-A & \multirow{2}{*}{1.621e-2$\pm$5.222e-3} & \multirow{2}{*}{{1.108e-1$\pm$1.561e-2}} & 195.57\% & \multirow{2}{*}{106.28\%} \\ \cline{1-1} \cline{4-4}
XPINN3-M &  &  & 300.47\% &  \\ \hline
\end{tabular}
\end{table}

\begin{figure}[htbp]
\centering
\includegraphics[scale=0.65]{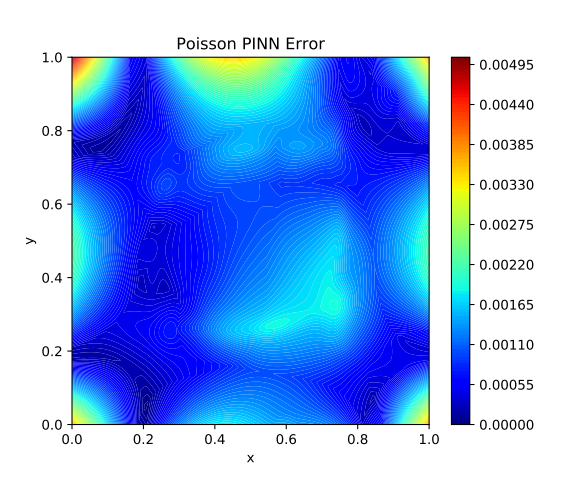}
\includegraphics[scale=0.65]{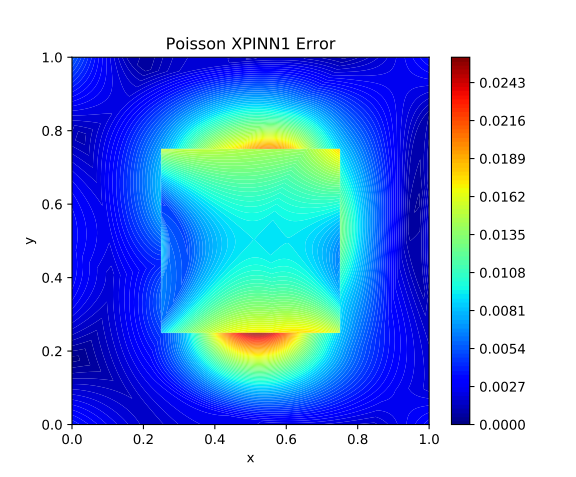}
\includegraphics[scale=0.65]{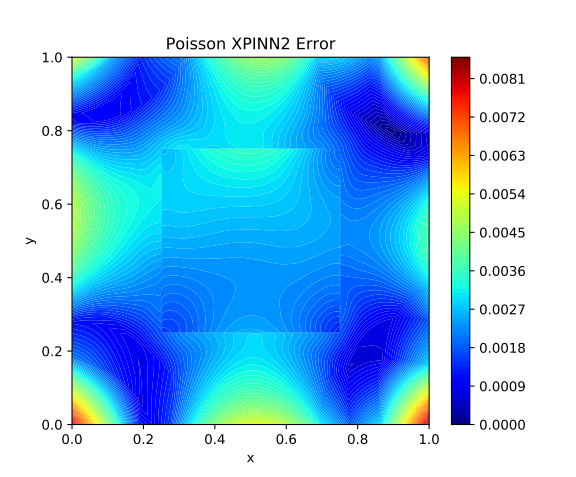}
\includegraphics[scale=0.65]{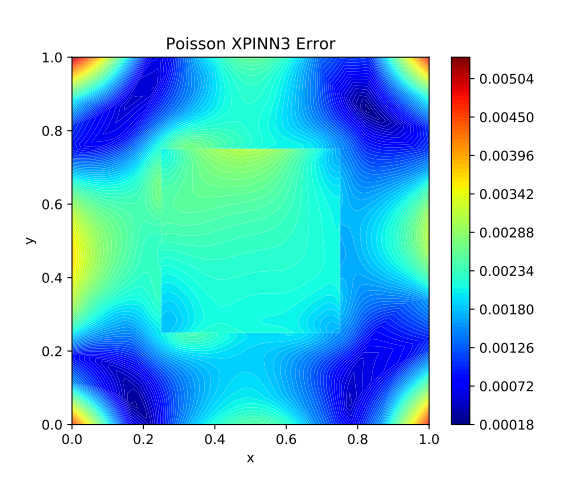}
\caption{Error visualization in the Poisson equation.}
\label{fig:poisson_errors}
\end{figure}
\subsubsection{Results}
Table \ref{tab:poisson} shows the experimental results for the Poisson equation.
PINN (5.553e-2) outperforms XPINN3 (1.108e-1), and XPINN3 outperforms XPINN2 (1.387e-1), and finally XPINN2 outperforms XPINN1 (4.022e-1), and their bounds also point to the same result, i.e., PINN (100\%) $<$ XPINN3 (106.28\%) $<$ XPINN2 (108.57\%) $<$ XPINN1 (122.56\%). All sub-nets in the XPINNs (142.71\%, 297.91\%, 183.44\%, 292.93\%, 195.57\%, 300.47\%) are more complicated then the PINN (100\%), which is the reason accounting for the failure of XPINNs. In this case, the lack of available data in each sub-domain of XPINNs impacts the generalization negatively, whose effect is much more obvious than the reduction of target function complexity in each sub-domain. 

To understand the failure of XPINNs, we visualize the error plots in Figure \ref{fig:poisson_errors}.
For XPINN1, due to the lack of the additional residual interface regularization on the interface \cite{de2022error}, the error is extremely large near the interface, which causes the largest error of XPINN1. 
For XPINN2, although the additional residual regularization does mitigate the error near the interface, the boundary error becomes much larger, which is because the introduction of an additional regularization decreases the importance of boundary loss during the optimization.
To solve the problem, in XPINN3 we add more weight to the boundary loss. While the error on the boundary decreases, the error near the interface increases, which is due to the tradeoff between different loss components, namely the tradeoff between the boundary loss and the interface loss.
To conclude, XPINNs perform worse than PINN since they perform bad either on the boundary or on the interface, which justifies the necessity of XPINNs with adaptive weight adjustment.

\subsection{Compressible Euler Equations}
\subsubsection{Setup}
Next, we consider the nonlinear inviscid compressible Euler equations, which govern the physics of high-speed compressible fluid flows. The inviscid compressible Euler equations admit discontinuous solutions called shock or contact waves, which are difficult to capture with good accuracy. 
The two-dimensional steady-state Euler equations are given as
$F_x(U) + F_y(U) = 0, ~~ (x, y) \in[0,1]^2$,
where fluxes in $x$ and $y$ directions are defined as $F_x(U) = (\rho u, p+ \rho u^2, \rho uv,  p u+\rho u E)$ and $F_y(U) = (\rho v, \rho uv, p + \rho v^2, p v+ \rho v E)$,
\begin{figure}[htbp]
\centering
\includegraphics[width=6.3cm, height=4.5cm]{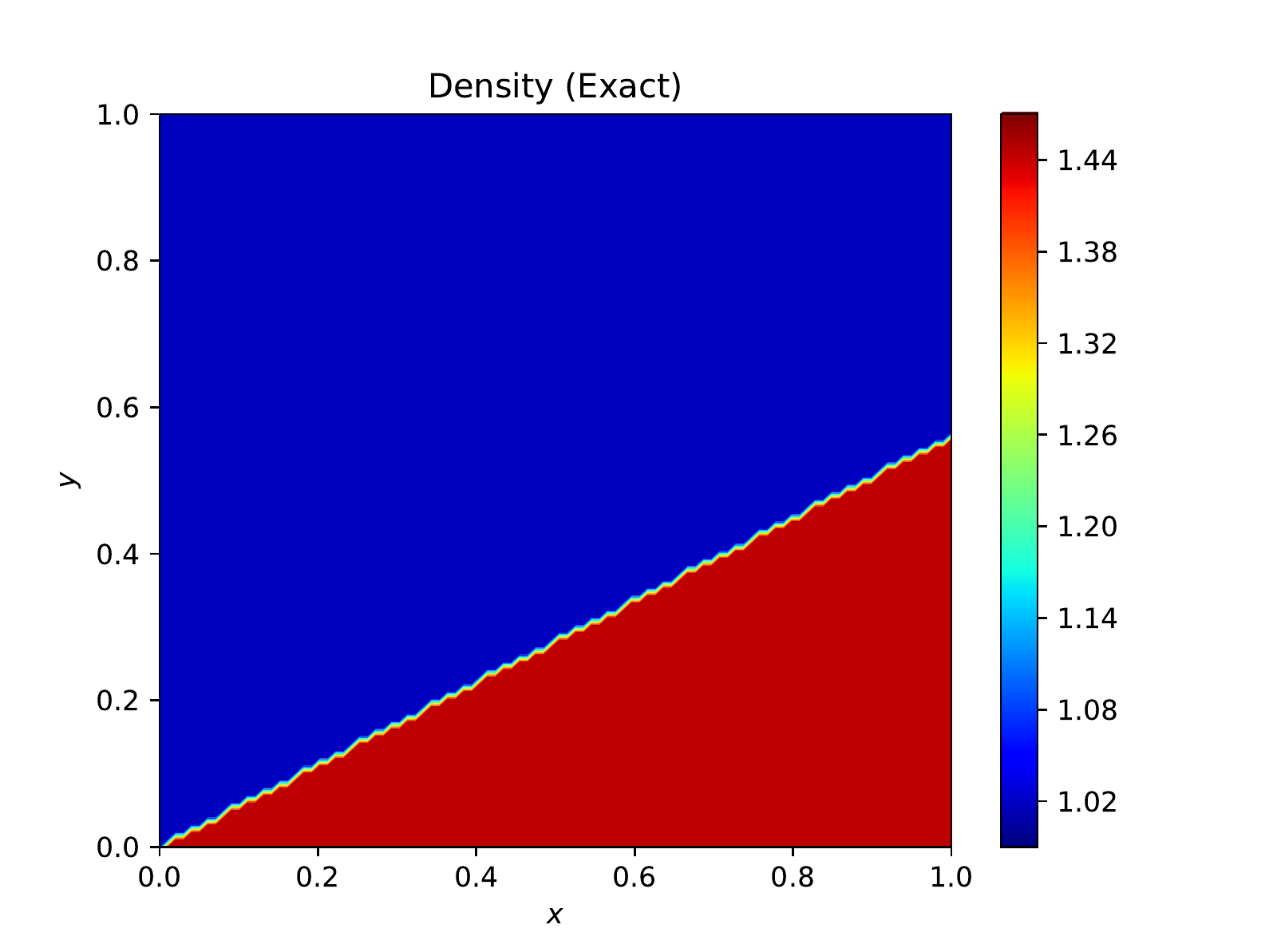}
\includegraphics[width=6.3cm, height=4.5cm]{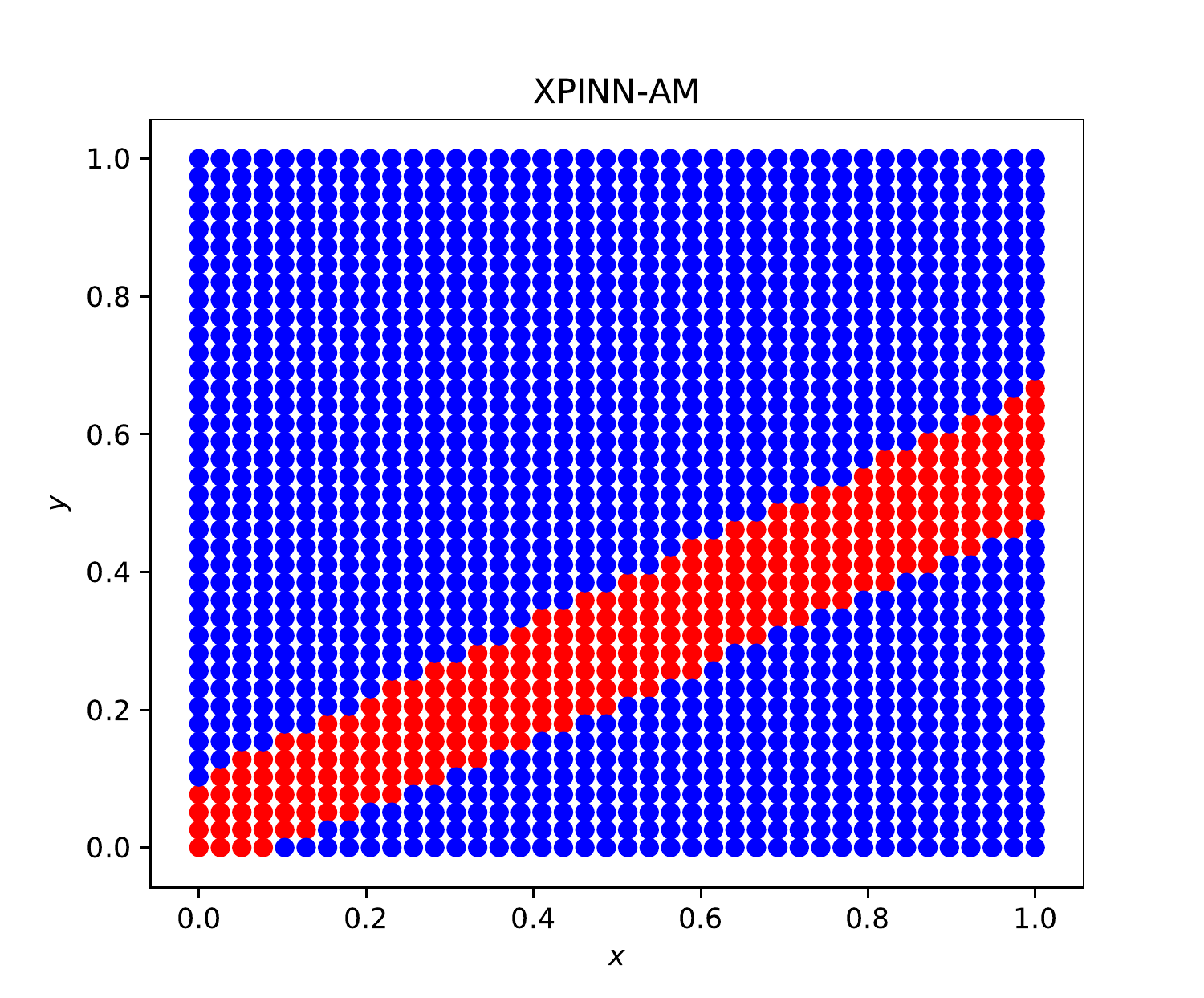}
\caption{Two-dimensional compressible Euler equations: Exact solution (left) and domain decomposition for XPINN-AM (right).}
\label{fig:Euler2D}
\end{figure}
where $\rho, u, v$ and $p$ are density, velocity components in $x$ and $y$ directions, and pressure, respectively. The total energy $E$ is defined as
$E = \frac{p}{\rho(\gamma-1)} + \frac{1}{2} ||\mathbf{u}||_2^2$,
where $\mathbf{u} = (u,v)$. In this case, we are solving the oblique shock wave problem on a square domain $[0,1]^2$.  
The bottom boundary is the wall where slip boundary conditions are applied, whereas left and top boundary are the inflow boundary where Dirichlet boundary conditions are applied. The right boundary has extrapolation boundary conditions. A Mach 2 flow is at an angle of -10 degrees with respect to the bottom wall, which generates an oblique shock at an angle of 29.3 degrees with the bottom horizontal wall.
The exact solution is given as
$$ (\rho, u,v,p) = \begin{cases} (1.0,~ \cos10^o, -\sin10^o, ~0.17857) & \text{before shock}, \\ (1.4584,~ 0.8873,~ 0.0,~ 0.3047) & \text{after shock}. \end{cases}
$$
Among all the primitive variables, we have plotted the fluid density which accurately shows the position of an oblique shock wave. Figure \ref{fig:Euler2D} shows the exact value of density (left) and domain decomposition (right) for XPINN-AM.  
From the domain decomposition figure we observe that the solution of the Euler equations is divided into the following two parts, $\left\{y \geq 0.57x + 0.1 \right\} \cup \left\{y \leq 0.5222 x - 0.0522 \right\}$, where solution is constant (shown by blue points) and the remaining strip (shown by red points), where oblique shock wave is present.

\subsubsection{Results}

\begin{figure}[htbp]
\centering
\includegraphics[width=6.3cm, height=4.5cm]{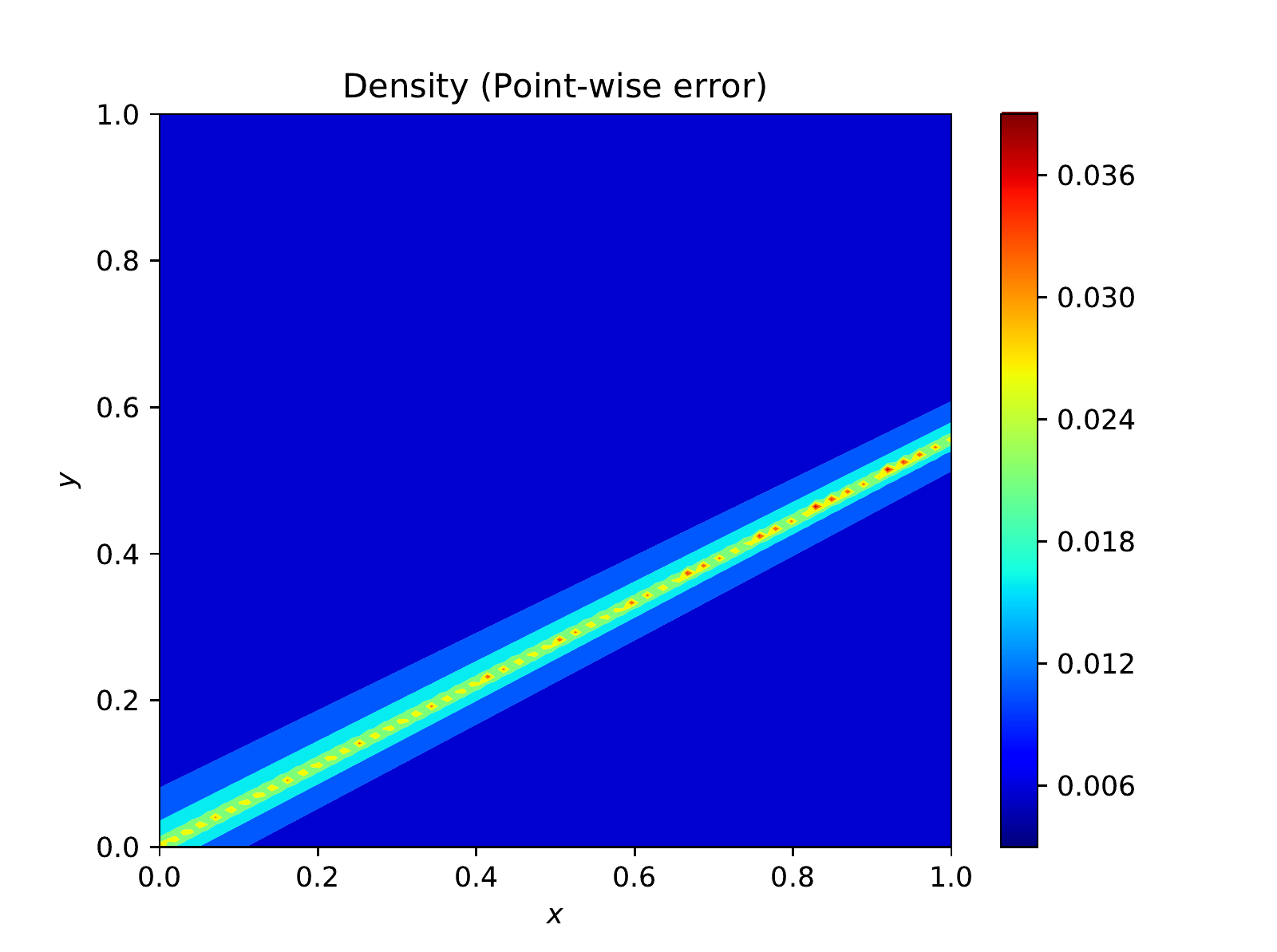}
\includegraphics[width=6.3cm, height=4.5cm]{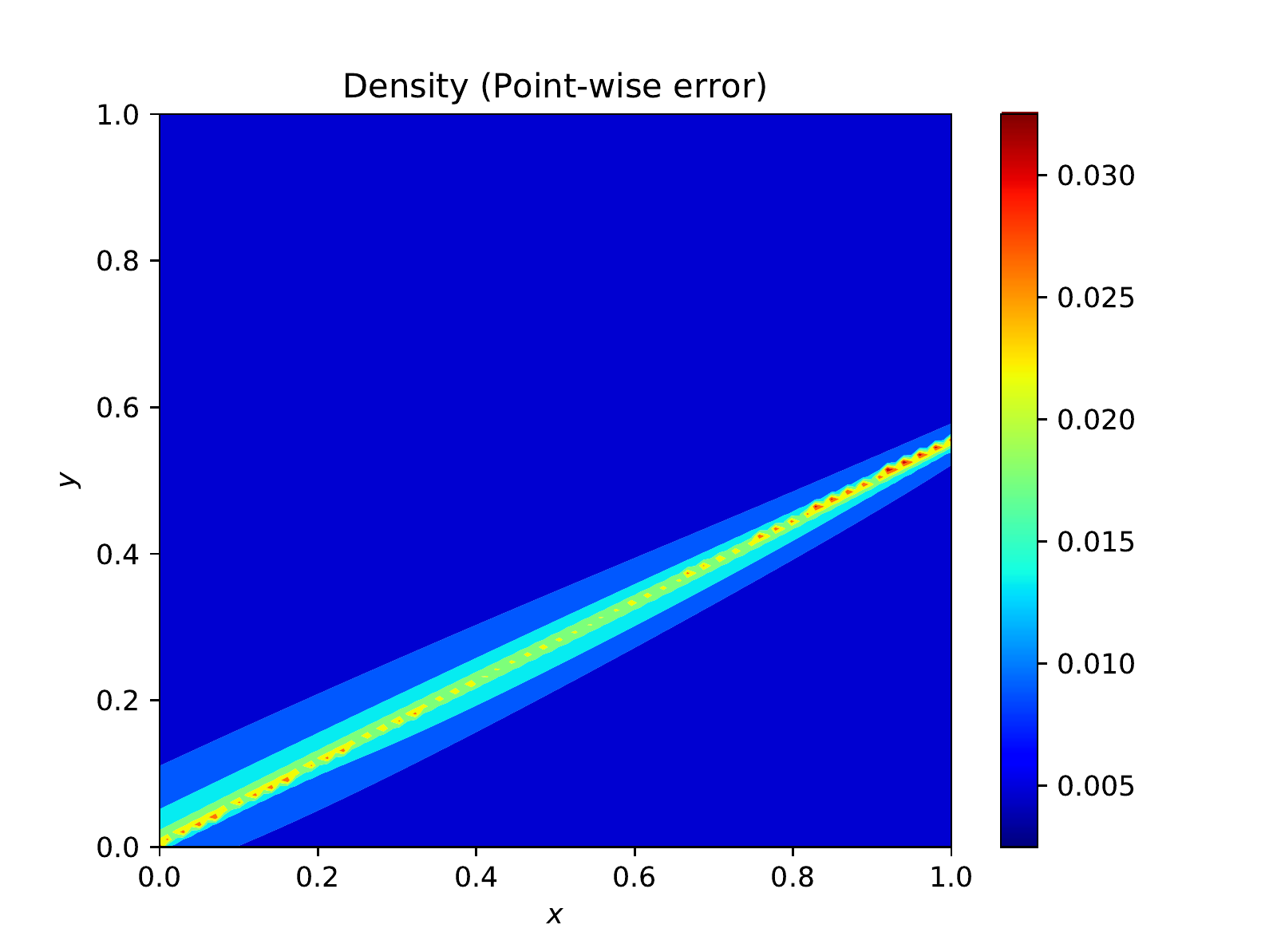}
\caption{Two-dimensional compressible Euler equations: Point-wise errors for PINN (left) and XPINN (right).}
\label{fig:Euler2D2}
\end{figure}

\begin{table}[]
\centering
\caption{Computational results for compressible Euler equations.}
\label{tab:Euler}
\begin{tabular}{|c|c|c|c|c|}
\hline
Method & Train Loss & {Relative $L_2$ error in $\rho$} & Norms & Bound \\ \hline
PINN & 1.819e-3$\pm$6.043e-4 & {3.4604e-2$\pm$7.385e-3} & 100.00\% & 100.00\% \\ \hline
XPINN-A & \multirow{2}{*}{9.210e-4$\pm$1.882e-4} & \multirow{2}{*}{{1.048e-2$\pm$5.3793e-3}} & 37.28\% & \multirow{2}{*}{81.09\%} \\ \cline{1-1} \cline{4-4}
XPINN-M &  &  & 64.37 \% &  \\ \hline
\multicolumn{1}{|c|}{XPINN-T} & \multicolumn{1}{c|}{\multirow{2}{*}{1.067e-3$\pm$4.829e-4}} & \multicolumn{1}{c|}{\multirow{2}{*}{{3.5722e-2$\pm$4.290e-3}}} & \multicolumn{1}{c|}{42.37\%} & \multicolumn{1}{c|}{\multirow{2}{*}{137.63\%}} \\ \cline{1-1} \cline{4-4}
\multicolumn{1}{|c|}{XPINN-B} & \multicolumn{1}{c|}{}  & \multicolumn{1}{c|}{}& \multicolumn{1}{c|}{131.26\%} & \multicolumn{1}{c|}{} \\ \hline
\end{tabular}
\end{table}
We used a deep net with 10000 residual points, 5 hidden-layers with 20 neurons in each layer, and 8e-4 learning rate. The activation function is hyperbolic tangent.
Table \ref{tab:Euler} gives the computational results, and Figure \ref{fig:Euler2D2} gives the point-wise error for the density of the fluid. In this case, the XPINN-AM generalizes better than PINN. Furthermore, the norms of XPINN-AM are much smaller than that of PINN (100\%). We further divide the domain into top ($y \geq 0.5$) and bottom subdomains ($y < 0.5$) for XPINN-TB. The XPINN-TB does not generalize well compared to PINNs, and the complexities of the two sub-nets are 42.37\% and 131.26\% for top and bottom subdomains, respectively. These results prove that the norms and generalization bounds are good indicators for XPINN based domain decomposition, and can be efficiently used to further decompose the subdomains.

\section{Conclusion}
In this study, we have investigated the generalization abilities of PINNs and XPINNs, as well as when and how XPINNs improve generalization. For this purpose, we have provided both prior and posterior generalization bounds to explain this from different viewpoints, where for the former we have developed the Barron space for multi-layer networks, while for the latter we have derived the complexity of norm-based Rademacher for PINNs. 

Through our discussion on theoretical results, analytical examples, and extensive experiments, we conclude that the domain decomposition method in XPINNs introduces a tradeoff on generalization. On the one hand, its advantage is that it decomposes the complex target function into several simple parts, which lead to the phenomenon that the sum of all parts is smaller than the whole. However, on the other hand, domain decomposition causes less available training data in each subdomain, leading to higher empirical Rademacher complexity and makes models prone to overfitting. When the complexity reduction brought by XPINN exceeds the increased complexity caused by less training data, XPINN outperforms PINN, as in our experiment on KdV equation, heat equation, advection equation, compressible Euler equation and the analytical example in Section \ref{sec:1}. When the overfitting caused by insufficient data is more dominant than the simplicity due to domain decomposition, PINN outperforms XPINN, as in our experiment on heat equation, wave equation and our analytical example in Section \ref{sec:2}. When the two factors reach a balance, XPINN and PINN perform similarly, as shown in our experiment on advection equation and our analytical example in Section \ref{sec:3}. 

Our results can also provide a partial explanation for  the following  observation. For long-time integration of several PDEs, it has been empirically observed that only XPINNs are applicable as PINNs  tend to be inaccurate. According to our theory, this is expected as we tend to have  a very high complexity measured by the norms for a whole solution of a long-time integration, which can be decomposed into less complex sub-solutions in XPINNs. 
Our proposed theory can also be useful for adaptive domain decomposition. Specifically, after initialization of a decomposition, we can compute the bound during optimization to know how the corresponding XPINN generalizes. If it does not generalize well, we can diagnose the reason, e.g., the negative overfitting effect caused by less training points is more obvious than the positive effect of less complex target function in each sub-domain, then we can restructure the decomposition by letting the sub-domain(s) contain more training points.
Overall, the present work provides the first theoretical understanding on when and how to employ XPINN for better generalization performances
over the vanilla PINNs.

\section*{Acknowledgment}
A. D. Jagtap and G. E. Karniadakis would like to acknowledge the funding by  OSD/AFOSR  MURI  Grant  FA9550-20-1-0358, and the US Department of Energy (DOE) PhILMs  project  (DE-SC0019453).
The authors would like to thank Michael Penwarden for his insightful comments and reviews on the paper and for his generous help with the code for experiments.

\newpage

\appendix

\section{Preliminary: Functional Analysis}
\begin{definition}
(Multi-Index) If $k$ is a positive integer, and $u \in C^k(\Omega)$, then we define the multi-index $\alpha=(\alpha_1,\cdots,\alpha_d)$ of order $|\alpha| = \alpha_1+\cdots+\alpha_d=k$, and the corresponding derivative is:
\begin{equation}
D^\alpha u = \frac{\partial^{\alpha_1}}{\partial x_1^{\alpha_1}} \cdots \frac{\partial^{\alpha_n}}{\partial x_d^{\alpha_d}}u.
\end{equation}
\end{definition}
\begin{definition} (Holder space)
The Holder space $C^{k,\gamma}(\overline{\Omega})$ consists of all functions $u \in C^k(\overline{\Omega})$ for which the norm
\begin{equation}
\Vert u \Vert_{C^{k,\gamma}(\overline{\Omega})} := \sum_{|\alpha| \leq k}\Vert D^{\alpha} u \Vert_{C(\overline{\Omega})} + \sum_{|\alpha| = k}[ D^{\alpha} u ]_{C^{0,\gamma}(\overline{\Omega})}
\end{equation}
is finite, where the norm and the semi-norm are defined as
\begin{equation}
\begin{aligned}
\Vert u \Vert_{C(\overline{\Omega})} &:= \sup_{x \in \Omega} |u(x)|.\\
[ u ]_{C^{0,\gamma}(\overline{\Omega})} &:= \sup_{x,y \in \Omega} \left\{\frac{|u(x) - u(y)|}{|x - y|^{\gamma}}\right\}.
\end{aligned}
\end{equation}
\end{definition}

\begin{theorem}
The space of functions $C^{k,\gamma}(\overline{\Omega})$ is a Banach space.
\end{theorem}

\begin{definition}
(Completeness) A metric space $X$ is called complete if every Cauchy sequence in $X$ has a limit that is also in $X$.
\end{definition}

\begin{definition}
(Relative compactness) A relatively compact subspace $Y$ of a topological space $X$ is a subset whose closure is compact.
\end{definition}

\begin{lemma}\label{lemma:fa}
Suppose $(f_n)_{n \in \mathbb{N}}$ is an equicontinuous sequence in $C(\overline{\Omega})$, then if $(f_n)_{n \in \mathbb{N}}$ converges to $f$ pointwise, then $f \in C(\overline{\Omega})$ and the convergence is uniform.
\end{lemma}
\begin{proof}
Suppose $x \in \overline{\Omega}$, since $(f_n)_{n \in \mathbb{N}}$ is an equicontinuous sequence in $C(\overline{\Omega})$, there exists $V_x \in \mathcal{N}(x)$, where $\mathcal{N}(x)$ denotes the neighborhood of $x$, such that when $y \in V_x$, for all $n \in \mathbb{N}$ we have
\begin{equation}
|f_n(y) - f_n(x)| \leq \epsilon.
\end{equation}
Therefore, 
\begin{equation}
\begin{aligned}
|f(y) - f(x)| &\leq|f(y) - f_n(y)| + |f_n(x) - f(x)| + |f_n(x) - f_n(y)|\\
&\leq |f(y) - f_n(y)| + |f_n(x) - f(x)| + \epsilon.
\end{aligned}
\end{equation}
Let $n \rightarrow \infty$, $|f(y) - f(x)| \leq \epsilon$. Thus, $f$ is continuous.

Since the set $\overline{\Omega}$ is compact, there exists $x_1,\cdots, x_m$ and their neighborhoods $V_{x_1},\cdots,V_{x_m}$, such that
\begin{equation}
\overline{\Omega} = \cup_{k=1}^m V_{x_k}. \qquad \forall x \in \overline{\Omega}, \exists x_k, x \in V_{x_k}, |f_n(x) - f_n(x_k)| \leq \epsilon, |f(x) - f(x_k)| \leq \epsilon,
\end{equation}
where $m$ is large enough for the given $\epsilon$, and we consider the equicontinuity of $f_n$ and the continuity of $f$.
For all $x \in \overline{\Omega}$, there exists $x_k$ such that $x \in V_{x_k}$. Given large enough $n$, due to the continuity of $f_n$ and $f$, we have
\begin{equation}
|f_n(x) - f(x)| \leq |f_n(x) - f_n(x_k)| + |f_n(x_k) - f(x_k)| + |f(x_k) - f(x)| \leq 3\epsilon.
\end{equation}
Thus the convergence is uniform.
\end{proof}

\begin{theorem}
(Arzela-Ascoli theorem) Suppose the funciton class $\mathcal{H} \subset C(\overline{\Omega})$, then $\mathcal{H}$ is relatively compact in $C(\overline{\Omega})$ if and only if $\mathcal{H}$ is equicontinuous.
\end{theorem}
\begin{proof}
We shall only use the sufficiency part of this theorem. So, only that part will be proved.

We first prove that $\overline{\mathcal{H}}$ is complete in $C(\Omega)$. We select arbitrary Cauchy sequence $(f_n)_{n \geq 1}$ in $\overline{\mathcal{H}}$. Then for all $\bx \in \overline{\Omega}$, $(f_n(\bx))_{n \geq 1}$ is a Cauchy sequence in the Euclidean space $\mathbb{R}$, which is convergent. By Lemma \ref{lemma:fa}, $(f_n)_{n \geq 1}$ uniformly converges to a continuous $f$. Thus, $\overline{\mathcal{H}}$ is complete.

We only need to show that, for all $\epsilon > 0$, $\overline{\mathcal{H}}$ can be covered by finite balls of radius $\epsilon$. Since $\mathcal{H}$ is equicontinuous, then for all $\bx \in \overline{\Omega}$ and $\epsilon > 0$, there exists open set $O_x \in \mathcal{N}(x)$, where $\mathcal{N}(x)$ denotes the neighbourhood of $x$, such that for all $y \in O_x$, we have $|f(y) - f(x)| \leq \epsilon$ for all $f \in \mathcal{H}$. Since $\overline{\Omega}$ is a compact set, there exists finite $(O_{x_i})_{i=1}^n$, such that $\overline{\Omega} = \cup_{i=1}^n O_{x_i}$.

Since all $f \in \mathcal{H}$ are continuous, the set $\cup_{i=1}^n \mathcal{H}(x_i)$ is compact, so there exists finite set $Z \subset \overline{\Omega}$, such that 
\begin{equation}
\cup_{i=1}^n \mathcal{H}(x_i) \subset \cup_{z \in Z} B(z, \epsilon),
\end{equation}
where $\mathcal{H}(x):=\{f(x): f \in \mathcal{H}\}$.
Now we consider the finite set $Z^n$, for all $\overline{z} = (z_1, \cdots, z_n) \in Z^n$. Let
\begin{equation}
B_{\overline{z}} = \left\{f \in C(\overline{\Omega}): \sup_{1\leq i \leq n}\sup_{x \in O_{x_i}}|f(x)-z_i|< 2\epsilon\right\}.
\end{equation}
Then we can a finite number of open subsets $(B_{\overline{z}})_{\overline{z} \in Z^n}$. 

For all $f,g \in B_{\overline{z}}$, we have
\begin{equation}
\Vert f - g\Vert_{C(\overline{\Omega})} \leq \sup_{1\leq i \leq n}\sup_{x \in O_{x_i}}\left(|f(x)-z_i| + |g(x)-z_i|\right) \leq 4\epsilon.
\end{equation}
Thus $\text{diam}B_{\overline{z}} \leq 4\epsilon$. We want to show $\mathcal{H} \subset \cup_{\overline{z} \in Z^n} B_{\overline{z}}$.

For all $f \in \mathcal{H}$, we know that for all $1 \leq i \leq n$, $f(x_i)$ belongs to one of $B(z_i, \epsilon)$ where $z_i \in Z$, due to the fact that $\cup_{i=1}^n \mathcal{H}(x_i) \subset \cup_{z \in Z} B(z, \epsilon)$. When $x \in O_{x_i}$, then
\begin{equation}
|f(x) - z_i| \leq |f(x) - f(x_i)| + |f(x_i) - z_i| < \epsilon + \epsilon = 2\epsilon.
\end{equation}
Therefore,
\begin{equation}
\sup_{1 \leq i \leq n} \sup_{x \in O_{x_i}}|f(x) - z_i| < 2\epsilon.
\end{equation}
\end{proof}

\begin{lemma}\label{lemma:dominated}
(Lebesgue's dominated convergence theorem) Let $(f_n)_{n \in \mathbb{N}}$ be a sequence of measurable functions in the space $\overline{\Omega}$ with measure $\mu$. Suppose that the sequence converges point-wise to a function $f$ and is dominated by some integrable function $g$ in the sense that $|f_n(x)| \leq g(x)$, for all numbers $n$ in the index set of the sequence and all points $x \in \overline{\Omega}$. Then $f$ is integrable (in the Lebesgue sense) and
\begin{equation}
\lim_{n \rightarrow \infty} \int_{\overline{\Omega}} |f_n(x)-f(x)|d\mu(x) = 0.
\end{equation}
\end{lemma}

\begin{definition}
(Sobolev space) The Sobolev space $W^{k,p}(\Omega)$ contains all locally summable functions $u: \Omega \rightarrow \mathbb{R}$ such that for each multi-index $\alpha$ with $|\alpha| \leq k$, $D^\alpha u$ exists in the weak sense and belongs to $L^p(\Omega)$. Furthermore, if $u \in W^{k,p}(\Omega)$, we define its norm to be
\begin{equation}
\begin{aligned}
\Vert u \Vert_{W^{k,p}(\Omega)} &:= \left(\sum_{|\alpha| \leq k}\int_{\Omega} |D^\alpha u|^p\right)^{1/p}. \qquad p \in [1, \infty).\\
\Vert u \Vert_{W^{k,p}(\Omega)} &:= \sum_{|\alpha| \leq k}\text{ess sup}_{\Omega} |D^\alpha u|. \qquad p = \infty.
\end{aligned}
\end{equation}
\end{definition}

\begin{theorem}\label{thm:sob_hil}
When $p=2$, the Sobolev space $H^p(\Omega) = W^{k,2}(\Omega)$ is a Hilbert space.
\end{theorem}

\begin{theorem}\label{thm:trace}
(Trace Theorem) Assume $\Omega$ is bounded and $\Omega$ is $C^1$. Then there exists a bounded linear operator:
\begin{equation}
T: W^{1,p}(\Omega) \rightarrow L^p(\partial\Omega)
\end{equation}
such that
\begin{itemize}
\item $Tu =u|_{\partial \Omega}$ if $u \in W^{1,p}(\Omega) \cap C(\overline{\Omega})$.
\item For each $u \in W^{1,p}(\Omega)$ and a constant $C$ depending only on $p$ and $\Omega$, 
\begin{equation}
\Vert Tu \Vert_{L^p(\partial\Omega)} \leq C \Vert u \Vert_{W^{1,p}(\Omega)}.
\end{equation}
\end{itemize}
\end{theorem}

\section{Proofs of the Barron Space}
\subsection{Proof of Theorem \ref{thm:barron_property}}

\begin{proof}
(Proof of Theorem \ref{thm:barron_property})
Because $X$ embeds continuously into $C^{2,1}(\Omega)$, there exists constants $C_1,C_2>0$ such that
\begin{equation}
\Vert D^\alpha g \Vert_{C^0(\Omega)} \leq C_1 \Vert g \Vert_{X}, \ [D^\alpha g]_{C^{0,1}(\Omega)} \leq C_2 \Vert g \Vert_{X}, \ \forall g \in X.
\end{equation}

\textbf{Banach Space}. By construction, $\mathcal{B}_{X,\Omega}$ is isometric to the quotient space $\mathcal{M}(B^X)/ N_K$ where
\begin{equation}
    N_K = \left\{\mu \in \mathcal{M}(B^X) | \int_{B^X} \rho(g) d\mu(g)) = 0, \forall \boldsymbol{x} \in \Omega \right\}.
\end{equation}
In particular, $\mathcal{B}_{X,\Omega}$ is a normed vector space with the norm $\Vert \cdot \Vert_{X,\Omega}$. Consider the mapping
\begin{equation}
\mathcal{M}(B^X) \rightarrow C^2(\overline{\Omega}), \ \mu \rightarrow f_\mu = \int_{B^X} \rho(g) d(\mu(g)).
\end{equation}
We prove that $f_\mu \in C^2(\overline{\Omega})$. For the continuity, we have
\begin{equation}
\begin{aligned}
|f_\mu(x) - f_\mu(y)| &= \int_{B^X} |\rho(g(x)) - \rho(g(y))| d(\mu(g)) \\
&\leq \int_{B^X} |g(x) - g(y)| d(\mu(g)).
\end{aligned}
\end{equation}
For the continuity at point $x$, consider arbitrary sequence $(x_n)_{n \in \mathbb{N}} \rightarrow x$. Consider the sequence of function $h_n: B^X \rightarrow \mathbb{R}$, $f \mapsto f(x_n)$ and $h: B^X \rightarrow \mathbb{R}$, $f \mapsto f(x)$. We have that $|h_n(f)| = |f(x_n)|\leq \Vert f \Vert_{C^0(\overline{\Omega})}$, which is integrable.
By Lemma \ref{lemma:dominated},
\begin{equation}
\begin{aligned}
\lim_{n \rightarrow \infty} |f_\mu(x) - f_\mu(x_n)|
&\leq  \lim_{n \rightarrow \infty}\int_{B^X} |g(x) - g(x_n)| d(\mu(g))\\
&=  \lim_{n \rightarrow \infty}\int_{B^X} |h(g) - h(g_n)| d(\mu(g))\\
&= \int_{B^X} \lim_{n \rightarrow \infty} |h(g) - h(g_n)| d(\mu(g))\\
&= \int_{B^X} \lim_{n \rightarrow \infty} |g(x) - g(x_n)| d(\mu(g))\\
&= 0,
\end{aligned}
\end{equation}
due to the continuity of $g$. Since the sequence $(x_n)_{n \in \mathbb{N}}$ is arbitrary, we know that $f_\mu$ is continuous at all points $x \in \overline{\Omega}$.

$f_\mu$ is first order differentiable due to the following:
\begin{equation}
\begin{aligned}
\frac{f_\mu(x) - f_\mu(x_n)}{x - x_n} -  \int_{B^X} \frac{\partial \rho(g(x))}{\partial x_i} d(\mu(g))&= \int_{B^X} \left[\frac{\rho(g(x)) - \rho(g(x_n))}{x - x_n} - \frac{\partial \rho(g(x))}{\partial x_i}\right] d(\mu(g)),
\end{aligned}
\end{equation}
where $x_i$ is the $i$-th coordinate of $x$, $(x_n)_i \rightarrow x_i$ as $n \rightarrow \infty$, and $(x_n)_{-i} = x_{-i}$, i.e., other coordinates are the same.
Because $X$ embeds continuously into $C^{2,1}(\Omega)$, we have
\begin{equation}
\begin{aligned}
&\left|\frac{\rho(g(x)) - \rho(g(x_n))}{x - x_n}\right|\\
&\leq \left|\frac{g(x) - g(x_n)}{x - x_n}\right|\\
&\leq C_2.
\end{aligned}
\end{equation}
Since the activation function and its derivatives are bounded, we also have
\begin{equation}
\begin{aligned}
\left|\frac{\partial \rho(g(x))}{\partial x_i}\right| &= \left|\rho^\prime(g(x))\frac{\partial \rho(g(x))}{\partial x_i}\right|\\
&\leq \left|\frac{\partial \rho(g(x))}{\partial x_i}\right|\\
&\leq C_1.
\end{aligned}
\end{equation}
By Lemma \ref{lemma:dominated},
\begin{equation}
\begin{aligned}
&\lim_{n \rightarrow \infty} \left|\frac{f_\mu(x) - f_\mu(x_n)}{x - x_n} -  \int_{B^X} \frac{\partial \rho(g(x))}{\partial x_i} d(\mu(g))\right|\\
&\leq  \int_{B^X} \lim_{n \rightarrow \infty}\left|\frac{\rho(g(x)) - \rho(g(x_n))}{x - x_n} - \frac{\partial \rho(g(x))}{\partial x_i}\right| d(\mu(g))\\
&= 0.
\end{aligned}
\end{equation}
Thus
\begin{equation}
\frac{\partial f_\mu(x)}{\partial x_i} = \int_{B^X} \frac{\partial \rho(g(x))}{\partial x_i}d(\mu(g)).
\end{equation}
Then, by induction we can know that $f_\mu$ is $C^k$ continuous and 
\begin{equation}
D^\alpha f_\mu = \int_{B^X} D^\alpha g d(\mu(g)).
\end{equation}
Back to the mapping
\begin{equation}
\mathcal{M}(B^X) \rightarrow C^2(\overline{\Omega}), \ \mu \rightarrow f_\mu = \int_{B^X} \rho(g) d(\mu(g)).
\end{equation}
It is continuous because
\begin{equation}
\begin{aligned}
&\qquad \left\| \int_{B^X} \rho(g) d(\mu(g)) \right\|_{C^2(\Omega)} \\
&\leq \int_{B^X} \Vert \rho(g) \Vert_{C^2(\Omega)} d(|\mu|(g))\\
&= \int_{B^X}\left\{ \Vert \rho(g) \Vert_{C^0(\Omega)} + \sum_{i=1}^d\Vert \rho'(g) \partial_i g \Vert_{C^0(\Omega)}\right\} d(|\mu|(g))\\
&\quad + \int_{B^X}\left\{ \sum_{i,j=1}^d\Vert \rho''(g) \partial_i g \partial_j g \Vert_{C^0(\Omega)} + \sum_{i,j=1}^d\Vert \rho'(g) \partial_{ij} g \Vert_{C^0(\Omega)}\right\} d(|\mu|(g))\\
&\leq \int_{B^X}\left\{ 1 + \sum_{i=1}^d\Vert \partial_i g \Vert_{C^0(\Omega)} + \sum_{i,j=1}^d\Vert \partial_i g \partial_j g \Vert_{C^0(\Omega)} + \sum_{i,j=1}^d\Vert  \partial_{ij} g \Vert_{C^0(\Omega)}\right\} d(|\mu|(g))\\
&\leq \left[1+(d+d^2)C_1 + d^2C_1^2\right] \Vert \mu \Vert_{\mathcal{M}(B^X)},
\end{aligned}
\end{equation}
by the definition of Brochner integrals and note that $x \in \mathbb{R}^d$, i.e., $d$ is the input dimension, and that $\partial_i$ denotes the first order derivative with respect to the $i$-th coordinate, and $\partial^2_{ij}$ denotes the second order derivative with respect to the $i$-th and the $j$-th coordinates. Due to the continuity, $N_K$ is the kernel of a continuous linear map. Therefore, $N_K$ is a closed subspace of $\mathcal{M}_{B^X}$. By the theorem in functional analysis, we conclude that $\mathcal{B}_{X,\Omega}$ is a Banach space. 

\textbf{$\mathcal{B}_{X,\Omega}$ embeds continuously into $C^{2,1}(\Omega)$.} In the proof of statement (1), we already have $\Vert f_\mu \Vert_{C^2(\Omega)}\leq 2C_1 \Vert \mu \Vert_{\mathcal{M}(B^X)}$. By taking infimum over $\mu$, we have $\Vert f \Vert_{C^2(\Omega)}\leq 2C_1 \Vert f \Vert_{\mathcal{B}_{X,\Omega}}$. Furthermore, for any $x \neq y \in \Omega$, we have
\begin{equation}
\begin{aligned}
| f_\mu(x) - f_\mu(y)|
&\leq \int_{B^X} | \rho(g(x)) - \rho(g(y)) | d |\mu|(g)\\
&\leq \int_{B^X} | g(x) - g(y) | d |\mu|(g)\\
&\leq \int_{B^X} [g]_{C^{0,1}(\overline{\Omega})} |x - y| d|\mu|(g)\\
&\leq C_2 \Vert \mu \Vert_{\mathcal{M}(B^X)} |x-y|.
\end{aligned}
\end{equation}
For the derivative, we have the similar conclusion:
\begin{equation}
\begin{aligned}
|\partial_i f_\mu(x) - \partial_i f_\mu(y)|
&\leq \int_{B^X} |\partial_i[ \rho(g(x)) ]- \partial_i [\rho(g(y))] | d |\mu|(g)\\
&\leq \int_{B^X} | \rho^\prime(g(x))\partial_i g(x) - \rho^\prime(g(y))\partial_i g(y) | d |\mu|(g)\\
&\leq \int_{B^X}\left\{ | \rho^\prime(g(x))\partial_i g(x) - \rho^\prime(g(x))\partial_i g(y) | + | \rho^\prime(g(x))\partial_i g(y) - \rho^\prime(g(y))\partial_i g(y) | \right\} d |\mu|(g) \\
&\leq \int_{B^X} \left\{ [g]_{C^{1,1}(\overline{\Omega})} |x - y| + [g]_{C^{0,1}(\overline{\Omega})}\Vert  \partial_i g\Vert_{C^0(\overline{\Omega})}|x - y| \right\} d|\mu|(g)\\
&\leq (C_2 + C_1C_2) \Vert \mu \Vert_{\mathcal{M}(B^X)} |x-y|,
\end{aligned}
\end{equation}
where $\partial_i$ denotes the first order derivative with respect to the $i$-th coordinate. Also, for the second order derivatives, we have
\begin{equation}
\begin{aligned}
|\partial_{ij} f_\mu(x) - \partial_{ij} f_\mu(y)|
&\leq \int_{B^X} |\partial_{ij}[ \rho(g(x)) ]- \partial_{ij} [\rho(g(y))] | d |\mu|(g)\\
&\leq \int_{B^X} (I_1 + I_2) d |\mu|(g).
\end{aligned}
\end{equation}
where $\partial^2_{ij}$ denotes the second order derivative with respect to the $i$-th and the $j$-th coordinates, and we have
\begin{equation}
\begin{aligned}
I_1 &= |\rho''(g(x)) \partial_i g(x) \partial_j g(x) - \rho''(g(y)) \partial_i g(y) \partial_j g(y)|\\
&\leq |\rho''(g(x)) \partial_i g(x) \partial_j g(x) - \rho''(g(x)) \partial_i g(y) \partial_j g(y)| + |\rho''(g(x)) \partial_i g(y) \partial_j g(y) - \rho''(g(y)) \partial_i g(y) \partial_j g(y)|\\
&\leq |\partial_i g(x) \partial_j g(x) - \partial_i g(y) \partial_j g(y)| +C_1^2 |\rho''(g(x))  - \rho''(g(y)) |\\
&\leq |\partial_i g(x) \partial_j g(x) - \partial_i g(x) \partial_j g(y)|+|\partial_i g(x) \partial_j g(y) - \partial_i g(y) \partial_j g(y)| +C_1^2 |g(x)  - g(y) |\\
&\leq 2C_1C_2|x-y| +C_1^2C_2 |x-y |\\
&\leq (2C_1C_2+C_1^2C_2)|x-y|.
\end{aligned}
\end{equation}
and
\begin{equation}
\begin{aligned}
I_2 &= |\rho'(g(x)) \partial_i \partial_j g(x) - \rho'(g(y)) \partial_i \partial_j g(y)|\\
&\leq |\rho'(g(x)) \partial_i \partial_j g(x) - \rho'(g(x)) \partial_i \partial_j g(y)| +|\rho'(g(x)) \partial_i \partial_j g(y) - \rho'(g(y)) \partial_i \partial_j g(y)| \\
&\leq 2C_2|x-y|.
\end{aligned}
\end{equation}
In sum, we have
\begin{equation}
\begin{aligned}
|\partial_{ij} f_\mu(x) - \partial_{ij} f_\mu(y)|
&\leq(2C_1C_2 + C_1^2C_2+2C_2)|x-y|.
\end{aligned}
\end{equation}

After taking infimum over $\mu$, we come to the conclusion.

\textbf{The closed unit ball of $\mathcal{B}_{X,\Omega}$ is a closed subset of $C^2(\Omega)$.} We assume that $(f_n)_{n \in \mathbb{N}}$ is a sequence such that $\Vert f_n \Vert_{X,K} \leq 1$. Choose a sequence of measures $(\mu_n)_{n \in \mathbb{N}}$ such that $f_n = f_{\mu_n}$ and $\Vert \mu_n \Vert \leq 1 + \frac{1}{n}$. These meansures exist because $f_n$ are from the unit ball of $\mathcal{B}(X,\Omega)$. By the compactness theorem of Radon meansures, there exists a subsequence $\mu_{n_k}$ and a Radon measure $\mu$, such that $\mu_{n_k}$ weak converge to $\mu$ with $\Vert \mu \Vert \leq 1$. Since all functions $\rho(g)$ are $C^2$ continuous and bounded, we have all $D^{\alpha}f_{n_k}$ converge in the product topology (pointwise convergence) for all $|\alpha| \leq 2$, since $f_{n} \in C^2(\overline{\Omega})$. In particular, if $f_{\mu_n} \rightarrow \hat{f}$ uniformly in the norm of $C^2$, then $\hat{f} = f_\mu \in B^{\mathcal{B}_{X,\Omega}}$, i.e. the unit ball of $\mathcal{B}_{X,\Omega}$ is closed in the $C^2(\Omega)$ topology.
\end{proof}

\subsection{Proof of Theorem \ref{thm:embedding}}
\begin{proof}
(Proof of Theorem \ref{thm:embedding}). This is immediate by the definition of generalized Barron space and that of neural networks in Definition \ref{def:DNN}.
\end{proof}

\subsection{Proof of Theorem \ref{thm:approximation}}
To prove the approximation property, we first prove the following useful lemma.
\begin{lemma}
Let $\mathcal{G}$ be a set in a Hilbert space $H$ such that $\Vert g \Vert_H \leq R$ for all $g \in \mathcal{G}$. If $f$ is in the closed convex hull of $G$, then for every $m \in \mathbb{N}$ and $\epsilon > 0$, there exist $m$ elements $g_1, \cdots , g_m \in G$ such that
\begin{equation}
\Vert f - \frac{1}{m} \sum_{i=1}^m g_i \Vert_H \leq \frac{R+\epsilon}{\sqrt{m}}.
\end{equation}
\end{lemma}
\begin{proof}
This lemma is proved by using the law of large numbers. See \cite{256500} for details.
\end{proof}
\begin{proof}
(Proof of Theorem \ref{thm:approximation})
Due to the choice of $\mathcal{W}^0$ as linear functions, the constants of continuous embedding are $C_1,C_2=1$. By the fact that $C^{2,1}(\overline{\Omega})$ embeds continuously into $H^2(\Omega)$, we have $\Vert f \Vert_{H^2} \leq 2 \Vert f \Vert_{\mathcal{W}^L}$.

Recall that the unit ball of $\mathcal{W}^L$ is the closed convex hull of the class $\mathcal{H} = \left\{ \sigma(g) : \Vert g \Vert_{\mathcal{W}^L} = 1 \right\}$. Thus by Lemma, there exists $g_1,\cdots,g_m\in\mathcal{W}^{L-1}$ and $\epsilon_1,\cdots,\epsilon_m\in\left\{-1,1\right\}$ such that
\begin{equation}
\Vert f - \frac{1}{m}\sum_{i=1}^m \epsilon_i \sigma(g_i(\boldsymbol{x})) \Vert_{H^2} \leq \frac{3 \Vert f \Vert_{\mathcal{W}^L}}{\sqrt{m}}.
\end{equation}
If $L=1$, $g_i$ are linear functions and $u_m(\boldsymbol{x}) = \sum_{i=1}^m \frac{\epsilon_i}{m} \sigma(g_i(\boldsymbol{x}))$ is a two-layer neural netowrk. Thus the case $L=1$ in the theorem is proved.

We prove the remain by induction. Assume that the theorem has been proved for the case of $L - 1$. Then we note that $\Vert g_i \Vert_{\mathcal{W}^L} = 1$, so for $1 \leq i \leq m$ we can find a finite $L - 1$-layer network $\hat{g}_i$ such that
\begin{equation}
\begin{aligned}
\Vert f - \frac{1}{m}\epsilon_i \sigma(\hat{g}_i(\boldsymbol{x})) \Vert_{H^2} &\leq \Vert f - \frac{1}{m}\epsilon_i \sigma(g_i(\boldsymbol{x})) \Vert_{H^2} + \frac{1}{m} \sum_{i=1}^m \Vert g_i- \hat{g}_i \Vert_{H^2}\\
&\leq \frac{3}{\sqrt{m}} + \frac{m}{m} \frac{3(L-1)}{\sqrt{m}}\\
&= \frac{3L}{\sqrt{m}}.
\end{aligned}
\end{equation}
We merge the $m$ trees associated with $\hat{g}_i$ into a single tree, increasing the width of each layer by a factor of $m$, and add an outer layer of width $m$ with coefficients $W$.

For the remaining part of the theorem, we only need to apply Theorem \ref{thm:trace}.
\end{proof}

\section{Proofs of Rademacher Complexity}\label{section:appendix_b}
In this section, we provide proofs for Rademacher complexity. We first provide full details for the second order derivatives of neural networks.
\begin{equation}
\begin{aligned}
    &\ \frac{\partial^2 u_{\boldsymbol{\theta}}(\boldsymbol{x})}{\partial \boldsymbol{x}^2}\\
    &\overset{(1)}{=} \frac{\partial \text{vec}(\boldsymbol{W}^L \cdot \boldsymbol{\Phi}^{L-1} \boldsymbol{W}^{L-1} \cdot \dots \cdot \boldsymbol{\Phi}^1 \boldsymbol{W}^1)}{\partial \boldsymbol{x}}\\
    &\overset{(2)}{=}\sum_{l=1}^{L-1}\frac{\partial \text{vec} (\boldsymbol{W}^L \cdot \boldsymbol{\Phi}^{L-1} \boldsymbol{W}^{L-1} \cdot \dots \cdot \boldsymbol{\Phi}^1 \boldsymbol{W}^1)}{\partial \text{vec}(\boldsymbol{\Phi}^l)}\frac{\partial \text{vec}(\boldsymbol{\Phi}^l)}{\partial \boldsymbol{x}},\\
    &\overset{(3)}{=} \sum_{l=1}^{L-1} (\boldsymbol{W}^l\cdots \boldsymbol{\Phi}^1\boldsymbol{W}^1)^\mathrm{T} \otimes (\boldsymbol{W}^L\boldsymbol{\Phi}^{L-1}\cdots\boldsymbol{W}^{l+1})\frac{\partial \text{vec}(\boldsymbol{\Phi}^l)}{\partial \boldsymbol{x}}\\
    &\overset{(4)}{=} \left\{\sum_{l=1}^{L-1} (\boldsymbol{W}^l\cdots \boldsymbol{\Phi}^1\boldsymbol{W}^1)^\mathrm{T} \otimes (\boldsymbol{W}^L\boldsymbol{\Phi}^{L-1}\cdots\boldsymbol{W}^{l+1})\frac{\partial \text{vec}(\boldsymbol{\Phi}^l)}{\partial \boldsymbol{x}_j}\right\}_{1 \leq j \leq d}\\
    &= \{\sum_{l=1}^{L-1} (\boldsymbol{W}^l\cdots \boldsymbol{\Phi}^1\boldsymbol{W}^1)^\mathrm{T} \otimes (\boldsymbol{W}^L\boldsymbol{\Phi}^{L-1}\cdots\boldsymbol{W}^{l+1})\\
    &\qquad \frac{\partial \text{vec}(\text{diag}[\sigma'(\boldsymbol{W}^l\sigma(\cdots\boldsymbol{W}^1\boldsymbol{x}))])}{\partial \boldsymbol{x}_j}\}_{1 \leq j \leq d}\\
    &= \{\sum_{l=1}^{L-1} (\boldsymbol{W}^l\cdots \boldsymbol{\Phi}^1\boldsymbol{W}^1)^\mathrm{T} \otimes (\boldsymbol{W}^L\boldsymbol{\Phi}^{L-1}\cdots\boldsymbol{W}^{l+1})\\ &\qquad \text{vec}(\text{diag}[\frac{\partial\sigma'(\boldsymbol{W}^l\sigma(\cdots\boldsymbol{W}^1\boldsymbol{x}))}{\partial \boldsymbol{x}_j}])\}_{1 \leq j \leq d}\\
    &=\{\sum_{l=1}^{L-1} (\boldsymbol{W}^l\cdots \boldsymbol{\Phi}^1\boldsymbol{W}^1)^\mathrm{T} \otimes (\boldsymbol{W}^L\boldsymbol{\Phi}^{L-1}\cdots\boldsymbol{W}^{l+1})\\ &\qquad \text{vec}(\text{diag}[\boldsymbol{\Psi}^l\boldsymbol{W}^{l}\cdots\boldsymbol{\Psi}^1\boldsymbol{W}^1_{:,j})])\}_{1 \leq j \leq d}\\
    &\overset{(5)}{=}\left\{\sum_{l=1}^{L-1} 
    (\boldsymbol{W}^L\boldsymbol{\Phi}^{L-1}\cdots\boldsymbol{W}^{l+1})
    \text{diag}(\boldsymbol{\Psi}^l\boldsymbol{W}^l\cdots\boldsymbol{\Psi}^1\boldsymbol{W}^1_{:,j})
    (\boldsymbol{W}^l\cdots \boldsymbol{\Phi}^1\boldsymbol{W}^1)\right\}_{1 \leq j \leq d}.
\end{aligned}
\end{equation}
In (1), note that $\text{vec} (\boldsymbol{W}^L \cdot \boldsymbol{\Phi}^{L-1} \boldsymbol{W}^{L-1} \cdot \dots \cdot \boldsymbol{\Phi}^1 \boldsymbol{W}^1)\in\mathbb{R}^d,\boldsymbol{x}\in\mathbb{R}^d$, thus the result of (1) is in $\mathbb{R}^{d\times d}$. In (2), we apply the chain rule, with the first term $\in\mathbb{R}^{d\times m_l^2}$ and the second term $\in\mathbb{R}^{m_l^2\times d}$. In (3) we use the formula $\frac{\partial \text{vec}(AXB)}{\partial \text{vec}(X)}=B^\mathrm{T}\otimes A$, and the first term $\in\mathbb{R}^{d\times m_l}$, second $\in\mathbb{R}^{1\times m_l}$, third $\in\mathbb{R}^{m_l^2\times d}$. In (4), we decompose the calculation into dimensional-wise with $\frac{\partial\text{vec}(\boldsymbol{\Phi}^l)}{{\partial \boldsymbol{x}_j}}\in\mathbb{R}^{m_l^2}$. In (5), we use the fact that $B^\mathrm{T}\otimes A \text{vec}(X) = \text{vec}(AXB)$.

\subsection{Spectral Norm for Complexity}
In this subsection, we shall use a covering number approach to Rademacher complexity. The following lemma is the key to connect them.
\begin{lemma} \cite{bartlett2017spectrally}
Let $\mathcal{F}$ be a real-valued function class taking values in $[0,1]$, and assume that $\mathbf{0} \in \mathcal{F}$. Then
\begin{equation}
\text{Rad}\left(\mathcal{F};{ S}\right) \leq \inf _{\alpha>0}\left(\frac{4 \alpha}{\sqrt{n}}+\frac{12}{n} \int_{\alpha}^{\sqrt{n}} \sqrt{\log \mathcal{N}\left(\mathcal{F}_{ S}, \varepsilon,\|\cdot\|_{2,2}\right)} d \varepsilon\right),
\end{equation}
where $S$ is the dataset, and $\mathcal{F}_S$ is the set containing the image of the dataset $S$ under all mappings in $\mathcal{F}$.
\end{lemma}
We note that this lemma requires that the hypothesis is in the interval $[0,1]$. In practice, we can consider the class of truncated neural networks. More specifically, if the class of neural network is denoted $\mathcal{F}$, then we consider the following class:
\begin{equation}
\widehat{\mathcal{F}} = \mathcal{F}_+ + \mathcal{F}_{-},
\end{equation}
where $\mathcal{F}_+ = \left\{f: f \cap [0,1], f \in \mathcal{F}\right\}$ and $\mathcal{F}_- = \left\{f: f \cap [-1,0], f \in \mathcal{F}\right\}$. Then, the function class $\widehat{F}$ is bounded in the interval $[-1,1]$, which is suitable for prediction of the target function $u^*({\bx})$ which is bounded by 1. In addition, their Rademacher complexity have the relationship: $\text{Rad}(\widehat{\mathcal{F}}) \leq \text{Rad}(\mathcal{F}_+) + \text{Rad}(\mathcal{F}_-)$. Throughout this paper, we will adopt the truncated neural network function class unless specified.

\begin{definition}
(Matrix Covering) We use $\mathcal{N}(U, \epsilon, \Vert \cdot \Vert)$ to denote the least cardinality of any subset $V \subset U$ that covers $U$ at scale $\epsilon$ with norm $\Vert \cdot \Vert$, i.e.,
\begin{equation}
\sup_{A \in U}\min_{B \in V} \Vert A - B \Vert \leq \epsilon.
\end{equation}
\end{definition}

\begin{lemma}
Consider the Hilbert space $\mathcal{H}$ with the norm $\Vert \cdot \Vert_{\mathcal{H}}$. Let $U \in \mathcal{H}$ has the reresentation $U = \sum_{i=1}^d \alpha_i V_i$, where all $V_i \in \mathcal{H}$, and $\alpha_i > 0$. Then for any positive integer $k$, there exists a choice of non-negative integers $(k_1,k_2,\cdots,k_d)$, such that $\sum_{i=1}^d k_i = k$, and
\begin{equation}
\Vert U - \frac{\Vert \alpha \Vert_1}{k} \sum_{i=1}^d k_i V_i \Vert^2_{\mathcal{H}} \leq \frac{\Vert \alpha \Vert_1}{k}\sum_{i=1}^d \alpha_i \Vert V_i \Vert^2_{\mathcal{H}} \leq \frac{\Vert \alpha \Vert_1^2}{k} \max_{1 \leq i \leq d}\Vert V_i \Vert^2_{\mathcal{H}}.
\end{equation}
\end{lemma}
\begin{proof}
See \cite{pisier1981remarques,bartlett2017spectrally} for details.
\end{proof}

\begin{lemma}\label{lemma:1}
\cite{bartlett2017spectrally} Let conjugate exponents $(p, q)$ and $(r, s)$ be given with $p \leq 2$, as well as positive reals $(a, b, \epsilon)$, and positive integer $m$. Let matrix $X \in \mathbb{R}^{n \times d}$ be given with $\Vert X \Vert_{p,p} \leq b$, where $n$ is the number of training data, and $d$ is the input dimension. Then,
\begin{equation}
\log \mathcal{N}\Big( \left\{ XA: A \in \mathbb{R}^{d \times m}, \Vert A^\mathrm{T} \Vert_{q,s} \leq a \right\}, \epsilon, \Vert \cdot \Vert_{2,2} \Big) \leq \lceil \frac{a^2 b^2 m^{2/r}}{\epsilon^2} \rceil \log(2dm),
\end{equation}
where $m$ can be interpreted as the hidden dimension of the model.
\end{lemma}
\begin{proof}
Fix the dataset matrix $X$, and construct $Y \in \mathbb{R}^{n \times d}$ by $Y_{:, j} := X_{:, j} / \Vert X_{:, j} \Vert_p$. We set the following quantites:
\begin{equation}
\begin{aligned}
N &:= 2dm \in \mathbb{N}^*,\\
k &:= \lceil \frac{a^2 b^2 m^{2/r}}{\epsilon^2} \rceil \in \mathbb{N}^*,\\
\overline{a} &:= a m^{1/r} \Vert X \Vert_p.
\end{aligned}
\end{equation}
Then, we define the following matrix sets:
\begin{equation}
\begin{aligned}
\{ V_1, V_2, \cdots, V_N\} &:= \left\{ g Y \boldsymbol{e}_i \boldsymbol{e}_j^{\mathrm{T}: g \in \{-1, 1\}}, i \in \{1,2,\cdots,d\}, j \in \{1,2,\cdots,m\}  \right\},\\
\mathcal{C} &:= \left\{ \frac{\overline{c}}{k} \sum_{i=1}^N k_i V_i : k_i \geq 0, \sum_{i=1}^N k_i = k \right\} = \left\{ \frac{\overline{c}}{k} \sum_{i=1}^N k_i V_i : k_i \geq 0, \sum_{i=1}^N k_i = k \right\}.
\end{aligned}
\end{equation}
We first note that since $p \leq 2$,
\begin{equation}
\max_i \Vert V_i \Vert_2 \leq \max_i \Vert Y \boldsymbol{e}_i \Vert_2 = \max_i \frac{\Vert X \boldsymbol{e}_i \Vert_2}{\Vert Y \boldsymbol{e}_i \Vert_p} \leq 1,
\end{equation}
where we use the fact that the matrix $\boldsymbol{e}_i \boldsymbol{e}_j^{\mathrm{T}}$ has only one non-zero entry at $(i, j)$.

We then show that $\mathcal{C}$ is the desired cover. Due to the construction of $\mathcal{C}$, its cardinal number $|\mathcal{C}| \leq N^k$, as we can interpret its construction as choose one from $V_i$ at each time and there are $k$ such choices to form an element in $\mathcal{C}$.

We consider a matrix $\Vert A \Vert_{q,s} \leq a$, and construct a covering element in $\mathcal{C}$ as follows. Let the matrix $\alpha \in \mathbb{R}^{d \times m}$, whose elements in the $j$-th row are all equal to $\Vert X_{:,j} \Vert_p$, then $XA = Y(\alpha \odot A)$.

\begin{equation}
\begin{aligned}
\|\alpha\|_{p, r} &=\left\|\left(\left\|\alpha_{:, 1}\right\|_{p}, \ldots,\left\|\alpha_{:, m}\right\|_{p}\right)\right\|_{r} \\
&=\left\|\left(\left\|\left(\left\|X_{:, 1}\right\|_{p}, \ldots,\left\|X_{:, d}\right\|_{p}\right)\right\|_{p}, \ldots,\left\|\left(\left\|X_{:, 1}\right\|_{p}, \ldots,\left\|X_{:, d}\right\|_{p}\right)\right\|_{p}\right)\right\|_{r} \\
&=m^{1 / r}\left\|\left(\left\|X_{:, 1}\right\|_{p}, \ldots,\left\|X_{:, d}\right\|_{p}\right)\right\|_{p}=m^{1 / r}\left(\sum_{j=1}^{d}\left\|X_{:, j}\right\|_{p}^{p}\right)^{1 / p} \\
&=m^{1 / r}\left(\sum_{j=1}^{d} \sum_{i=1}^{n} X_{i, j}^{p}\right)^{1 / p}=m^{1 / r}\|X\|_{p}.
\end{aligned}
\end{equation}
Define $B:=\alpha \odot A$, whereby using conjugacy of $\|\cdot\|_{p, r}$ and $\|\cdot\|_{q, s}$ gives
\begin{equation}
\|B\|_{1} \leq\langle\alpha,|A|\rangle \leq\|\alpha\|_{p, r}\|A\|_{q, s} \leq m^{1 / r}\|X\|_{p} a=\bar{a}.
\end{equation}
Consequently, $X A$ is equal to
$$
Y B=Y \sum_{i=1}^{d} \sum_{j=1}^{m} B_{i j} \mathbf{e}_{i} \mathbf{e}_{j}^{\top}=\|B\|_{1} \sum_{i=1}^{d} \sum_{j=1}^{m} \frac{B_{i j}}{\|B\|_{1}}\left(Y \mathbf{e}_{i} \mathbf{e}_{j}^{\top}\right) \in \bar{a} \cdot \operatorname{conv}\left(\left\{V_{1}, \ldots, V_{N}\right\}\right)
$$
where $\operatorname{conv}\left(\left\{V_{1}, \ldots, V_{N}\right\}\right)$ is the convex hull of $\left\{V_{1}, \ldots, V_{N}\right\}$.

Combining the preceding constructions with Lemma A.6 there exist nonnegative integers $\left(k_{1}, \ldots, k_{N}\right)$ with $\sum_{i} k_{i}=k$ with
$$
\left\|X A-\frac{\bar{a}}{k} \sum_{i=1}^{N} k_{i} V_{i}\right\|_{2}^{2}=\left\|Y B-\frac{\bar{a}}{k} \sum_{i=1}^{N} k_{i} V_{i}\right\|_{2}^{2} \leq \frac{\bar{a}^{2}}{k} \max _{i}\left\|V_{i}\right\|_{2}^2 \leq \frac{a^{2} m^{2 / r}\|X\|_{p}^{2}}{k} \leq \epsilon^{2}
$$
The desired cover element is thus $\frac{a}{k} \sum_{i} k_{i} V_{i} \in \mathcal{C}$.
\end{proof} 

We first revisite the definition of neural networks. Given weight matrices $\mathcal{W}=\left(W_{1}, \ldots, W_{L}\right)$, where $L$ is the network depth, we define the mapping $F_{\mathcal{W}}$, and more generally for $i \leq L$ define $\mathcal{W}_{1}^{i}:=\left(W_{1}, \ldots, W_{i}\right)$ and
\begin{equation}
F_{\mathcal{W}}(Z) = W_L\sigma\left(W_{L-1}  \cdots \sigma\left(W_{1} Z\right) \cdots\right),
\end{equation}
\begin{equation}
F_{\mathcal{W}_{1}^{i}}(Z):=\sigma\left(W_{i} \sigma\left(W_{i-1} \cdots \sigma\left(W_{1} Z\right) \cdots\right)\right),
\end{equation}
with the convention $F_{\emptyset}(Z)=Z$, where $Z$ is the input.

Define two sequences of matrix spaces $\mathcal{V}_{1}, \ldots, \mathcal{V}_{L}$ and $\mathcal{W}_{2}, \ldots, \mathcal{W}_{L+1}$, where $\mathcal{V}_{i}$ has a norm $|\cdot|_{i}$ and $\mathcal{W}_{i}$ has norm $||| \cdot |||_{i}$. Specifically, we choose all vector spaces as Euclidean spaces, and choose $|\cdot|_{i} = \Vert \cdot \Vert_{2,2}$ and $|||\cdot|||_{i} = \Vert \cdot \Vert_{2, 2}$.
The inputs $Z \in \mathcal{V}_{1} = \mathbb{R}^{d \times n}$ satisfy a norm constraint $\Vert Z \Vert_{\infty.\infty} \leq 1$, which means the absolute values of all entries in all data is not larger than $B$. Specifically, we are using $Z=X^{\mathrm{T}}$.

The linear operators $A_{i}: \mathcal{V}_{i} \rightarrow \mathcal{W}_{i+1}$ are associated with some operator norm $|A_{i}|_{i \rightarrow i+1} \leq c_{i}$, where we use $|A_{i}|_{i \rightarrow i+1} = \Vert A_i \Vert_{2}$, which is the spectral norm and satisfies:
\begin{equation}
\Vert A_i \Vert_{2} = \sup_{\Vert Z \Vert_{2,2} \leq 1} \Vert A_i Z \Vert_{2,2} := c_i.
\end{equation}
For the activation function, we consider the sine activation function, which is an 1-Lipschitz mappings $\sigma: \mathcal{W}_{i+1} \rightarrow \mathcal{V}_{i+1}$, having the Lipschitz constant $\rho_{i}=1$, measured with respect to norms $|\cdot|_{i} = \Vert \cdot \Vert_{\infty}$ and $||| \cdot |||_{i}= \Vert \cdot \Vert_{\infty}$. In other words, for any $z, z^{\prime} \in \mathcal{W}_{i+1}$, we have
\begin{equation}
|\sigma(z)-\sigma(z^{\prime})|_{i+1} = \Vert \sigma(z)-\sigma(z^{\prime}) \Vert_{2,2} \leq \Vert z-z^{\prime} \Vert_{2,2} = |||z-z^{\prime} |||_{i+1}.
\end{equation}
We will prove the following lemma on covering number.
\begin{lemma}\label{Lemma:Covering}
Let the resolutions for covering $(\epsilon_{1}, \ldots, \epsilon_{L})$ be given, and given the operator norm bounds $(c_{1}, \ldots, c_{L})$. Suppose the matrices $\mathcal{W}=\left(W_{1}, \ldots, W_{L}\right)$ lie within the set $\mathcal{B}_{1} \times \cdots \times \mathcal{B}_{L}$, where $\mathcal{B}_{i}$ are arbitrary classes with the property that each $W_{i} \in \mathcal{B}_{i}$ has $|W_{i}|_{i \rightarrow i+1} = \Vert W_{i} \Vert_{2} \leq c_{i}$. Lastly, let the dataset $Z = X^{\mathrm{T}}$ be given with $|Z|_{1} = \Vert Z \Vert_{\infty,\infty} \leq 1$. Then, letting
\begin{equation}
\tau:=\sum_{j= 1}^L \epsilon_{j} \rho_{j} \prod_{l=j+1}^{L} \rho_{l} c_{l} = \sum_{j= 1}^L \epsilon_{j} \prod_{l=j+1}^{L} c_{l},
\end{equation}
then the neural net images $\mathcal{H}_{Z}:=\left\{F_{\mathcal{W}}(Z): \mathcal{W} \in \mathcal{B}_{1} \times \cdots \times \mathcal{B}_{L}\right\}$ have covering number bound:
\begin{equation}
\begin{aligned}
&\quad \mathcal{N}(\mathcal{H}_{Z}, \tau,|\cdot|_{L+1})\\
&= \mathcal{N}(\mathcal{H}_{Z}, \tau,\Vert \cdot \Vert_{2,2})\\
&\leq \prod_{i=1}^{L} \sup_{W_{j} \in \mathcal{B}_{j}, \forall j < i} \mathcal{N}\left(\left\{W_{i} F_{\mathcal{W}_1^{i-1}}(Z): W_{i} \in \mathcal{B}_{i}\right\}, \epsilon_{i},||| \cdot |||_{i+1}\right)\\
&= \prod_{i=1}^{L} \sup_{W_{j} \in \mathcal{B}_{j}, \forall j < i} \mathcal{N}\left(\left\{W_{i} F_{\mathcal{W}_1^{i-1}}(Z): W_{i} \in \mathcal{B}_{i}\right\}, \epsilon_{i},\Vert \cdot \Vert_{2,2}\right).
\end{aligned}
\end{equation}
\end{lemma}

\begin{proof}
In this paper, we consider the case when all $c_1 \geq 1$.
We first denote $\mathcal{F}_i$ as the covering set of the image set of all $i$-layer neural networks, which are inductively constructed as follows.
Choose an $\epsilon_{1}$-cover $\mathcal{F}_{1}$ of the set of one-layer neural networks $\left\{W_{1} Z: W_{1} \in \mathcal{B}_{1}\right\}$, then its cardinality satisfies
\begin{equation}
\left|\mathcal{F}_{1}\right| \leq \mathcal{N}\left(\left\{W_{1} Z: W_{1} \in \mathcal{B}_{1}\right\}, \epsilon_{1},||| \cdot |||_{2}\right)=: N_{1},
\end{equation}
by the definition of covering number. For every element $F \in \mathcal{F}_{i}$, which is a covering vector, we construct an $\epsilon_{i+1}$-cover $\mathcal{G}_{i+1}(F)$ of the set
\begin{equation}
\left\{W_{i+1} \sigma_{i}(F): W_{i+1} \in \mathcal{B}_{i+1}\right\},
\end{equation}
where $F$ is chosen and fixed. Since the covers are proper, i.e.,
\begin{equation}
F \in \mathcal{F}_i \subset \left\{W_{i} F_{\mathcal{W}_1^{i-1}}(Z): W_{j} \in \mathcal{B}_{j}, \forall j \leq i\right\},
\end{equation}
meaning that $F=W_{i} F_{W_1^{i-1}}(Z)$ for some matrices $\left(W_{1}, \ldots, W_{i}\right) \in$ $\mathcal{B}_{1} \times \cdots \times \mathcal{B}_{i}$. Then, we obtain
\begin{equation}
\left|\mathcal{G}_{i+1}(F)\right| \leq \sup _{\forall j \leq i . W_{j} \in \mathcal{B}_{j}} \mathcal{N}\left(\left\{W_{i+1} F_{W_{1}, \ldots, W_{i}}(Z): W_{i+1} \in \mathcal{B}_{i+1}\right\}, \epsilon_{i+1}, |||\cdot|||_{i+2}\right)=: N_{i+1}.
\end{equation}
Lastly we construct the cover
\begin{equation}
\mathcal{F}_{i+1}:=\bigcup_{F \in \mathcal{F}_{i}} \mathcal{G}_{i+1}(F),
\end{equation}
whose cardinality satisfies
\begin{equation}
\left|\mathcal{F}_{i+1}\right| \leq\left|\mathcal{F}_{i}\right| \cdot N_{i+1} \leq \prod_{l=1}^{i+1} N_{l}.
\end{equation}
Define $\mathcal{F}:=\left\{\sigma(F): F \in \mathcal{F}_{L}\right\}$. By construction, $\mathcal{F}$ satisfies the desired cardinality constraint. To show that it is indeed a cover, fix any $\left(W_{1}, \ldots, W_{L}\right)$ satisfying the above constraints, and for convenience define recursively the mapped elements
\begin{equation}
F_{1}=W_{1} X \in \mathcal{W}_{2}, \quad G_{i}=\sigma\left(F_{i}\right) \in \mathcal{V}_{i+1} \quad F_{i+1}=W_{i+1} G_{i} \in \mathcal{W}_{i+2}.
\end{equation}
The goal is to show the existence of $\hat{G}_{L} \in \mathcal{F}$ satisfying:
\begin{equation}
|G_{L}-\hat{G}_{L}|_{L+1} = \Vert G_{L}-\hat{G}_{L}\Vert_{2,2} \leq \tau
\end{equation}
To this end, inductively construct approximating elements $\left(\hat{F}_{i}, \hat{G}_{i}\right)$ as follows. The Base case: set $\hat{G}_{0}=X$ since $G_0 = X$. For other cases, choose $\hat{F}_{i} \in \mathcal{F}_{i}$ with $||| W_{i} \hat{G}_{i-1}-\hat{F}_{i}|||_{i+1} \leq \epsilon_{i}$, since $\mathcal{F}_i$ is an $\epsilon_i$ cover of the following set:
\begin{equation}
\left\{W_{i} F_{\mathcal{W}_1^{i-1}}(Z): W_{j} \in \mathcal{B}_{j}, \forall j \leq i\right\}.
\end{equation}
And set $\hat{G}_{i}:=\sigma(\hat{F}_{i})$. To complete the proof, it will be shown inductively that
\begin{equation}
\left|G_{i}-\widehat{G}_{i}\right|_{i+1} \leq \sum_{1 \leq j \leq i} \epsilon_{j} \rho_{j} \prod_{l=j+1}^{i} \rho_{l} c_{l}
\end{equation}
For the base case,
\begin{equation}
\left|G_{0}-\widehat{G}_{0}\right|_{1}=0
\end{equation}
For the inductive step, we obtain
\begin{equation}
\begin{aligned}
\left|G_{i+1}-\widehat{G}_{i+1}\right|_{i+2} & \leq \rho_{i+1}|||F_{i+1}-\widehat{F}_{i+1}|||_{i+2} \\
& \leq \rho_{i+1}|||F_{i+1}-W_{i+1} \widehat{G}_{i}|||_{i+2}+\rho_{i+1}||| W_{i+1} \widehat{G}_{i}-\widehat{F}_{i+1}|||_{i+2} \\
& \leq \rho_{i+1}\left|W_{i+1}\right|_{i+1 \rightarrow i+2}\left|G_{i}-\widehat{G}_{i}\right|_{i+1}+\rho_{i+1} \epsilon_{i+1} \\
& \leq \rho_{i+1} c_{i+1}\left(\sum_{j \leq i} \epsilon_{j} \rho_{j} \prod_{l=j+1}^{i} \rho_{l} c_{l}\right)+\rho_{i+1} \epsilon_{i+1} \\
&=\sum_{j \leq i+1} \epsilon_{j} \rho_{j} \prod_{l=j+1}^{i+1} \rho_{l} c_{l},
\end{aligned}
\end{equation}
where we note that $\prod_{l=i+1}^{i} \rho_{l} c_{l} = 1$.
\end{proof}
The whole-network covering bound in terms of spectral and $(2,1)$ norms now follows by the general norm covering number and the matrix covering lemma.
\begin{theorem}
Let spectral norm bounds $\left(s_{1}, \ldots, s_{L}\right)$, and matrix $(2,1)$ norm bounds $\left(b_{1}, \ldots, b_{L}\right)$ be given. Let data matrix $X \in \mathbb{R}^{n \times d}$ be given, where the $n$ rows correspond to data points. Let $\mathcal{H}_{X}$ denote the family of matrices obtained by evaluating $X$ with all choices of network $F_{\mathcal{W}}$, i.e., 
\begin{equation}
\mathcal{H}_{X}:=\left\{F_{\mathcal{W}}\left(X^{\mathrm{T}}\right): \mathcal{W}=\left(W_{1}, \ldots, W_{L}\right),\left\|W_{i}\right\|_{2} \leq s_{i},\left\|W_{i}\right\|_{2,1} \leq b_{i}\right\}.
\end{equation}
Then for any $\epsilon > 0$, we have the covering number bound:
\begin{equation}
\log \mathcal{N}\left(\mathcal{H}_{X}, \epsilon,\|\cdot\|_{2,2}\right) \leq \frac{nd  \log \left(2 h^{2}\right)}{\epsilon^{2}}\left(\prod_{j=1}^{L} s_{j}^{2}\right)\left(\sum_{i=1}^{L}\left(\frac{b_{i}}{s_{i}}\right)^{2 / 3}\right)^{3},
\end{equation}
where the network width is denoted $h$. 
\end{theorem}
\begin{proof}
We set the matrix constraint sets as $\mathcal{B}_{i}=\left\{W_{i}:\left\|W_{i}\right\|_{2} \leq s_{i},\left\|W_{i}\right\|_{2,1} \leq b_{i}\right\}$, and lastly the per-layer cover resolutions $\left(\epsilon_{1}, \ldots, \epsilon_{L}\right)$ set according to
\begin{equation}
\epsilon_{i}:=\frac{\alpha_{i} \epsilon}{\rho_{i} \prod_{j>i} \rho_{j} s_{j}} = \frac{\alpha_{i} \epsilon}{ \prod_{j=i+1}^L s_{j}}, \quad \text {where} \quad \alpha_{i}:=\frac{1}{\bar{\alpha}}\left(\frac{b_{i}}{s_{i}}\right)^{2 / 3} \quad, \quad \bar{\alpha}:=\sum_{j=1}^{L}\left(\frac{b_{j}}{s_{j}}\right)^{2 / 3}.
\end{equation}
By this choice, it follows that the final cover resolution $\tau$ provided by Lemma satisfies
\begin{equation}
\tau \leq \sum_{j=1}^L \epsilon_{j} \prod_{l=j+1}^{L} s_{l}= \sum_{j=1}^L \alpha_{j} \epsilon=\epsilon
\end{equation}
To start, the covering number estimate from Lemma \ref{Lemma:Covering} can be combined with Lemma \ref{lemma:1} with $p=2, s=1$ to get
\begin{equation}
\begin{aligned}
&\quad \log \mathcal{N}\left(\mathcal{H}_{X}, \epsilon,\|\cdot\|_{2,2}\right) \\
&= \sum_{i=1}^{L} \sup_{W_{j} \in \mathcal{B}_{j}, \forall j < i} \log \left(\mathcal{N}\left(\left\{W_{i} F_{\mathcal{W}_1^{i-1}}(Z): W_{i} \in \mathcal{B}_{i}\right\}, \epsilon_{i},\Vert \cdot \Vert_{2,2}\right) \right)\\
&= \sum_{i=1}^{L} \sup_{W_{j} \in \mathcal{B}_{j}, \forall j < i} \log \left(\mathcal{N}\left(\left\{ F_{\mathcal{W}_1^{i-1}}(Z)^\mathrm{T}W_{i}^{\mathrm{T}}: W_{i} \in \mathcal{B}_{i}\right\}, \epsilon_{i},\Vert \cdot \Vert_{2,2}\right) \right)\\
&\leq \sum_{i=1}^{L} \sup_{W_{j} \in \mathcal{B}_{j}, \forall j < i} \frac{b_i^2\Vert F_{\mathcal{W}_1^{i-1}}(Z)^\mathrm{T}\Vert_{2,2}^2\log(2h^2)}{\epsilon_i^2}
\end{aligned}
\end{equation}
We can bound the intermediate outputs of the neural network as follows:
\begin{equation}
\begin{aligned}
\left\|F_{\left(W_{1}, \ldots, W_{i-1}\right)}\left(Z\right)^{\top}\right\|_{2,2} 
&=\left\|F_{\left(W_{1}, \ldots, W_{i-1}\right)}\left(Z\right)\right\|_{2,2} \\
&=\| \sigma\left(W_{i-1} F_{\left(W_{1}, \ldots, W_{i-2}\right)}\left(Z\right)\|_{2,2}\right.\\
& \leq \left\|W_{i-1} F_{\left(W_{1}, \ldots, W_{i-2}\right)}\left(Z\right)\right\|_{2,2} \\
& \leq \left\|W_{i-1}\right\|_{2}\left\|F_{\left(W_{1}, \ldots, W_{i-2}\right)}\left(Z\right)\right\|_{2,2},
\end{aligned}
\end{equation}
which by induction gives
\begin{equation}
\max _{j}\left\|F_{\left(W_{1}, \ldots, W_{i-1}\right)}\left(Z\right)^{\mathrm{T}}\right\|_{2,2} \leq\|Z\|_{2,2} \prod_{j=1}^{i-1} \left\|W_{j}\right\|_{2}.
\end{equation}
Combining all equations, and then expanding the choice of $\epsilon_{i}$ and collecting terms, we attain
\begin{equation}
\begin{aligned}
\log \mathcal{N}\left(\mathcal{H}_{X}, \epsilon,\|\cdot\|_{2}\right) & \leq \sum_{i=1}^{L} \sup _{\left(W_{1}, \ldots, W_{i-1}\right),j<i, W_{j} \in \mathcal{B}_{j}} \frac{b_{i}^{2}\|X\|_{2,2}^{2} \prod_{j<i} \left\|W_{j}\right\|_{2}^{2}}{\epsilon_{i}^{2}} \log \left(2 h^{2}\right) \\
& \leq \|X\|_{2,2}^{2}\sum_{i=1}^{L} \frac{b_{i}^{2}  \prod_{j<i} s_{j}^{2}}{\epsilon_{i}^{2}} \log \left(2 h^{2}\right) \\
&=\|X\|_{2,2}^{2}\frac{\log \left(2 h^{2}\right) \prod_{j=1}^{L} s_{j}^{2}}{\epsilon^{2}} \sum_{i=1}^{L} \frac{b_{i}^{2}}{\alpha_{i}^{2} s_{i}^{2}}\\
&=\|X\|_{2,2}^{2}\frac{\log \left(2 h^{2}\right) \prod_{j=1}^{L} s_{j}^{2}}{\epsilon^{2}} \sum_{i=1}^{L} \frac{\bar{\alpha}^2b_{i}^{2}}{ s_{i}^{2}} \cdot  \left(\frac{s_{i}^{2}}{b_{i}^{2}}\right)^{2/3} \\
&=\|X\|_{2,2}^{2}\frac{\log \left(2 h^{2}\right) \prod_{j=1}^{L} s_{j}^{2}}{\epsilon^{2}}\bar{\alpha}^2 \sum_{i=1}^{L} \left(\frac{b_{i}^{2}}{ s_{i}^{2}}\right)^{1/3} \\
&=\|X\|_{2,2}^{2}\frac{\log \left(2 h^{2}\right) \prod_{j=1}^{L}s_{j}^{2}}{\epsilon^{2}}\left(\bar{\alpha}^{3}\right)\\
&\leq nd\|X\|_{\infty,\infty}^{2}\frac{\log \left(2 h^{2}\right) \prod_{j=1}^{L}s_{j}^{2}}{\epsilon^{2}}\left(\bar{\alpha}^{3}\right)\\
\end{aligned}
\end{equation}
where we have used:
\begin{equation}
\epsilon_{i} = \frac{\alpha_{i} \epsilon}{ \prod_{j=i+1}^L s_{j}},  \quad \alpha_{i}=\frac{1}{\bar{\alpha}}\left(\frac{b_{i}}{s_{i}}\right)^{2 / 3}, \quad \bar{\alpha}:=\sum_{j=1}^{L}\left(\frac{b_{j}}{s_{j}}\right)^{2 / 3}.
\end{equation}
\end{proof}

\begin{lemma} 
For every $L$, and every set of $n$ points $S \subset \overline{\Omega}$, the hypothesis class $\mathcal{NN}^L_{M,N}$ given by the neural networks
\begin{equation}
\mathcal{NN}^L_{M,N} := \left\{\boldsymbol{x} \mapsto W_L \sigma (W_{L-1} \sigma(\cdots \sigma(W_1\boldsymbol{x})) )\ |\ \Vert W_l \Vert_{2} \leq M(l), \frac{\Vert W_l \Vert_{2,1}}{\Vert W_l \Vert_{2}} \leq N(l), \forall l \right\},
\end{equation}
satisfies the Rademacher complexity bound
\begin{equation}
\text{Rad}(\mathcal{NN}^L_{M,N};S) \leq   \frac{4}{n\sqrt{n}} + \frac{18\sqrt{d\log(2h^2)}\log n}{\sqrt{n}} \prod_{l=1}^L M(l)\Big(\sum_{l=1}^L N(l)^{2/3}\Big)^{3/2},
\end{equation} 
where $h$ is the maximal width of the neural network, i.e.,
\begin{equation}
h = \max(m_L, \cdots, m_0).
\end{equation}
\end{lemma}
\begin{proof}
Consider the covering number bound:
\begin{equation}
\log \mathcal{N}\left((\mathcal{NN}_{M,N}^L)_S, \epsilon,\|\cdot\|_{2}\right) \leq \frac{nd\log \left(2 h^{2}\right)}{ \epsilon^{2}}\left(\prod_{j=1}^{L} s_{j}^{2} \right)\left(\sum_{i=1}^{L}\left(\frac{b_{i}}{s_{i}}\right)^{2 / 3}\right)^{3}=: \frac{R}{\epsilon^{2}}.
\end{equation}
What remains is to relate covering numbers and Rademacher complexity via a Dudley entropy integral:
\begin{equation}
\text{Rad}(\mathcal{NN}^L_{M,N};S) \leq \inf _{\alpha>0}\left(\frac{4 \alpha}{\sqrt{n}}+\frac{12}{n} \int_{\alpha}^{\sqrt{n}} \sqrt{\frac{R}{\epsilon^{2}}} \mathrm{~d} \epsilon\right)=\inf _{\alpha>0}\left(\frac{4 \alpha}{\sqrt{n}}+\log (\sqrt{n} / \alpha) \frac{12 \sqrt{R}}{n}\right).
\end{equation}
The inf is uniquely minimized at $\alpha:=3 \sqrt{R / n}$, but the desired bound may be obtained by the simple choice $\alpha:=1 / n$, and plugging the resulting Rademacher complexity estimate:
\begin{equation}
\begin{aligned}
\text{Rad}(\mathcal{NN}^L_{M,N};S) &\leq  \frac{4}{n\sqrt{n}} + \log(n^{3/2})\frac{12\sqrt{R}}{n}\\
&= \frac{4}{n\sqrt{n}} +\frac{18\log n }{n}\sqrt{R}\\
&\leq \frac{4}{n\sqrt{n}} + \frac{18\sqrt{d\log(2h^2)}\log n}{\sqrt{n}} \prod_{l=1}^L M(l)\Big(\sum_{l=1}^L N(l)^{2/3}\Big)^{3/2}.
\end{aligned}
\end{equation}
\end{proof}

Recall the expressions of the differentiated PINNs:
\begin{equation}
\frac{\partial u_{\boldsymbol{\theta}}(\boldsymbol{x})}{\partial \boldsymbol{x}} = W_L\Phi_{L-1} W_{L-1} \dots \Phi_1 W_1 \in \mathbb{R}^{d},
\end{equation}
\begin{equation}
\frac{\partial^2 u_{\boldsymbol{\theta}}(\boldsymbol{x})}{\partial \boldsymbol{x}^2} =
\left\{\sum_{l=1}^{L-1} 
({W}_L{\Phi}_{L-1}\cdots{W}_{l+1})
\text{diag}({\Psi}_l{W}_l\cdots{\Psi}_1({W}_1)_{:,j})
({W}_l\cdots {\Phi}_1{W}_1)\right\}_{1 \leq j \leq d}.
\end{equation}
where
\begin{equation}
\Phi_l = \text{diag}[\sigma'(W_l\sigma(W^{l-1}\sigma(\cdots\sigma(W^1\boldsymbol{x})\cdots)] \in \mathbb{R}^{m_l\times m_l},
\end{equation}
\begin{equation}
\Psi_l = \text{diag}[\sigma''(W^l\sigma(W^{l-1}\sigma(\cdots\sigma(W^1\boldsymbol{x})\cdots)] \in \mathbb{R}^{m_l\times m_l}.
\end{equation}
The forwards passes for one input $\bx_i\in \mathbb{R}^d$ of the PINN model are:
\begin{equation}
W_L\Phi_{L-1} W_{L-1} \dots \Phi_1 W_1 b(\bx_i) \in \mathbb{R},
\end{equation}
\begin{equation}
\sum_{j=1}^d \sum_{l=1}^{L-1} 
{W}_L{\Phi}_{L-1}\cdots{W}_{l+1}
\text{diag}({\Psi}_l{W}_l\cdots{\Psi}_1({W}_1)_{:,j})
{W}_l\cdots {\Phi}_1{W}_1A_{:,j} \in \mathbb{R}.
\end{equation}
Now we begin to prove their Rademacher complexities. 
\begin{lemma}
The hypothesis class $\mathcal{F}_2$ given by
\begin{equation}
\mathcal{F}_2 = \left\{\bx \mapsto W_L\Phi_{L-1} W_{L-1} \dots \Phi_1 W_1 b(\bx)\ |\ \Vert W_l \Vert_{2} \leq M(l), \frac{\Vert W_l \Vert_{2,1}}{\Vert W_l \Vert_{2}} \leq N(l), \forall l \right\},
\end{equation}
satisfies the Rademacher complexity bound
\begin{equation}
\text{Rad}(\mathcal{F}_2;S) \leq \frac{4}{n\sqrt{n}} + \frac{18L\sqrt{2d\log(2h^2)}\log n}{\sqrt{n}}\left(\prod_{l=1}^L M(l)\right)^2 \left(\sum_{l=1}^L N(l)^{2/3}\right)^{3/2}.
\end{equation} 
\end{lemma}
\begin{proof}
We only need to consider how to cover the set of PINNs. We first consider the first-order derivatives, i.e., how to cover:
\begin{equation}
\bx \in \mathbb{R}^d \mapsto W_L\Phi_{L-1} W_{L-1} \dots \Phi_1 W_1 b(\bx) \in \mathbb{R}.
\end{equation}
To do so, we consider the covering in Lemma \ref{Lemma:Covering}. Specifically, denote $\mathcal{F}_i$ as the $\epsilon_i$ covering of the following set:
\begin{equation}
\left\{W_{i} F_{\mathcal{W}_1^{i-1}}(Z): W_{j} \in \mathcal{B}_{j}, \forall j \leq i\right\},
\end{equation}
which is the set of all $i$-layer neural networks, with weight matrices in the set $\mathcal{B}_{j}$. In Lemma \ref{Lemma:Covering}, we have inductively constructed covers $\mathcal{F}_{1}, \ldots, \mathcal{F}_{L}$. Specifically, $\mathcal{F}_{1}$ is an $\epsilon_{1}$-cover of the set of one-layer neural networks $\left\{W_{1} Z: W_{1} \in \mathcal{B}_{1}\right\}$, and its cardinality satisfies
\begin{equation}
\left|\mathcal{F}_{1}\right| \leq \mathcal{N}\left(\left\{W_{1} Z: W_{1} \in \mathcal{B}_{1}\right\}, \epsilon_{1},\Vert \cdot \Vert_{2,2}\right).
\end{equation} 
For every element $F \in \mathcal{F}_{i}$, which is a covering vector of the set of image of all $i$-layer neural networks, we construct an $\epsilon_{i+1}$-cover $\mathcal{G}_{i+1}(F)$ of the set
\begin{equation}
\left\{W_{i+1} \sigma_{i}(F): W_{i+1} \in \mathcal{B}_{i+1}\right\},
\end{equation}
where $F$ is chosen and fixed. Then, we obtain
\begin{equation}
\left|\mathcal{G}_{i+1}(F)\right| \leq \sup _{\forall j \leq i . W_{j} \in \mathcal{B}_{j}} \mathcal{N}\left(\left\{W_{i+1} F_{W_{1}, \ldots, W_{i}}(Z): W_{i+1} \in \mathcal{B}_{i+1}\right\}, \epsilon_{i+1}, \Vert \cdot \Vert_{2,2}\right),
\end{equation}
where we note that the matrices $W_{1}^i$ are fixed. Lastly we construct the cover
\begin{equation}
\mathcal{F}_{i+1}:=\bigcup_{F \in \mathcal{F}_{i}} \mathcal{G}_{i+1}(F),
\end{equation}
whose cardinality satisfies
\begin{equation}
\left|\mathcal{F}_{i+1}\right| \leq\left|\mathcal{F}_{i}\right| \cdot N_{i+1} \leq \prod_{l=1}^{i+1} N_{l}.
\end{equation}
We consider how to cover the term $\Phi_i$ within the PINN. We note that each $\Phi_i = \Phi_i(Z_j)$ for the $j$-th data is a diagonal matrix:
\begin{equation}
\text{diag}(\Phi_i(Z_j)) = \sigma'(W_{i}\cdots\sigma(W_1Z_j)),
\end{equation}
where the operator diag means to take the diagonal of the matrix to form a vector, which is a $i$-layer neural network in particular. Therefore, $\mathcal{F}_{i}$ can cover $W_{i}\cdots\sigma(W_1Z)$ with error not larger than $\sum_{j \leq i} \epsilon_{j} \prod_{l=j+1}^{i}c_{l}$, and in particular the function class $\sigma'(\mathcal{F}_i)$ can cover $\sigma'(W_{i}\cdots\sigma(W_1Z))$ with error not larger than $\sum_{j \leq i} \epsilon_{j} \prod_{l=j+1}^{i}c_{l}$, where we use the fact that $\sigma'$ is 1-Lipschitz. 

Then, we consider the $\left(\prod_{k=1}^i c_k \right)\epsilon_i$-cover of the following set of intermediate output of a PINN:
\begin{equation}
\left\{W_{i} G_{\mathcal{W}_1^{i-1}}(Z): W_{j} \in \mathcal{B}_{j}, \forall j \leq i\right\},
\end{equation}
where
\begin{equation}
G_{\mathcal{W}_1^{i-1}}(Z) = \left[\Phi_{i-1}(z_j) W_{i-1} \dots \Phi_1(z_j) W_1 b(z_j)\right]_{j=1}^n \in \mathbb{R}^{m_{i-1} \times n}.
\end{equation}
We stress that the image of PINN should be computed data-wise, as each $\Phi_i(Z_j)$ is a matrix. More concretely, we shall inductively construct covers $\mathcal{G}_{1}, \ldots, \mathcal{G}_{L}$. Specifically, $\mathcal{G}_{1}$ is an $c_1 \epsilon_{1}$-cover of the set of one-layer neural networks $\left\{W_{1} b(Z) \in \mathbb{R}^{m_1 \times n}: W_{1} \in \mathcal{B}_{1}\right\}$, and its cardinality satisfies
\begin{equation}
\left|\mathcal{G}_{1}\right| \leq \mathcal{N}\left(\left\{W_{1} b(Z): W_{1} \in \mathcal{B}_{1}\right\}, c_1\epsilon_{1},\Vert \cdot \Vert_{2,2}\right).
\end{equation} 
For every element $G \in \mathcal{G}_{i}$, and $F \in \mathcal{F}_i$, we construct an $\prod_{k=1}^{i+1}c_k\epsilon_{i+1}$-cover of the set
\begin{equation}
\mathcal{H}(F, G) = \left\{\left[W_{i+1}\text{diag}(\sigma'\left(F(Z_j)\right))G\right]_{j=1}^n \in \mathbb{R}^{m_{i+1} \times n}: W_{i+1} \in \mathcal{B}_{i+1}\right\},
\end{equation}
where $G$ and $F$ are chosen and fixed. We denote
\begin{equation}
\mathcal{G}_{i+1} := \bigcup_{F \in \mathcal{F}_{i}, G \in \mathcal{G}_{i}}\mathcal{H}\left(F,G\right).
\end{equation}
whose cardinality satisfies
\begin{equation}
|\mathcal{H}(F,G)| \leq \sup _{\forall j \leq i, W_{j} \in \mathcal{B}_{j}} \mathcal{N}\left(\left\{W_{i+1} G_{W_{1}^{i}}(Z)\right\}, \prod_{k=1}^{i+1}c_k\epsilon_{i+1}, \Vert\cdot\Vert_{2,2}\right):=M_{i+1},
\end{equation}
where $\forall j \leq i, W_{j} \in \mathcal{B}_{j}$ are fixed. Consequently, its cardinality can be bounded as follows
\begin{equation}
|\mathcal{G}_{i+1}| \leq \prod_{l=1}^{i+1} M_{l} N_{l}.
\end{equation}
After the construction, we shall show that how the function class can cover the PINN before the $l$-th layer with an error no larger than $l\left(\prod_{k=1}^l c_k\right) \sum_{j=1}^l \epsilon_{j}  \prod_{k=j+1}^{l} c_{k}$, which will be shown inductively. For the first layer, we have
\begin{equation}
\Vert G_1 - {W}_1 b(Z) \Vert_{2,2} \leq c_1\epsilon_1,
\end{equation}
for some $G_1 \in \mathcal{G}_1$, due to its definition. For general cases, consider the following inequalities:
\begin{equation}
\begin{aligned}
&\quad \Vert\hat{\Phi}_{l} \hat{G}_{l} - {\Phi}_{l} \cdots {\Phi}_1 {W}_1 b(Z) \Vert_{2,2}\\
&=\Vert\left[ \hat{\phi}_{l}(Z_j) \odot \hat{G}_{l}(Z_j) - {\phi}_{l}(Z_j) \odot W_{l}\cdots {\Phi}_1(Z_j) {W}_1 b(Z_j) \right]_{j=1}^n\Vert_{2,2}\\ 
&\leq \Vert \left[ (\hat{\phi}_{l}(Z_j) -  {\phi}_{l}(Z_j)) \odot \hat{G}_{l}(Z_j)\right]_{j=1}^n \Vert_{2,2} + \Vert \left[ {\phi}_{l}(Z_j)\odot(\hat{G}_l - {W}_{l}\cdots {\Phi}_1 {W}_1 b(Z_j)) \right]_{j=1}^n\Vert_{2,2}\\
&\leq \Vert (\hat{\phi}_{l}(Z) - \phi_{l}(Z) ) \Vert_{2,2}\Vert \hat{G}_{l}(Z) \Vert_{\infty,\infty}+\Vert  {\phi}_{l}(Z)\Vert_{\infty,\infty}\Vert\hat{G}_{l} - {W}_{l}\cdots {\Phi}_1 {W}_1 b(Z)) \Vert_{2,2}\\
&\leq \left(\prod_{k=1}^l c_k\right) \sum_{j=1}^l \epsilon_{j}  \prod_{k=j+1}^{l} c_{k} + 
l\left(\prod_{k=1}^l c_k\right) \sum_{j=1}^l \epsilon_{j}  \prod_{k=j+1}^{l} c_{k}\\
&= (l+1)\left(\prod_{k=1}^l c_k\right) \sum_{j=1}^l \epsilon_{j}  \prod_{k=j+1}^{l} c_{k},
\end{aligned}
\end{equation}
where $\hat{\Phi}_l \in \mathcal{F}_l$. Furthermore, consider the new $W_{l+1}$, we attain
\begin{equation}
\begin{aligned}
&\quad \Vert\hat{G}_{l+1} - W_{l+1}{\Phi}_{l} \cdots {\Phi}_1 {W}_1 b(Z) \Vert_{2,2}\\
&\leq \Vert\hat{G}_{l+1} - W_{l+1}\hat{\Phi}_{l} \hat{G}_{l-1} \Vert_{2,2} + \Vert W_{l+1}\hat{\Phi}_{l} \hat{G}_{l-1} - W_{l+1}{\Phi}_{l} \cdots {\Phi}_1 {W}_1 b(Z)\Vert_{2,2}\\
&\leq \left(\prod_{k=1}^{l+1} c_k\right) \epsilon_{l+1} + (l+1)\left(\prod_{k=1}^{l+1} c_k\right) \sum_{j=1}^l \epsilon_{j}  \prod_{k=j+1}^{l} c_{k}\\
&\leq (l+1)\left(\prod_{k=1}^{l+1} c_k\right) \sum_{j=1}^{l+1} \epsilon_{j}  \prod_{k=j+1}^{l+1} c_{k}.
\end{aligned}
\end{equation}
By the same logic, we have
\begin{equation}
\begin{aligned}
&\quad \log \mathcal{N}\left(\mathcal{H}_{X}, \tau,\|\cdot\|_{2,2}\right) \\
&= \sum_{i=1}^{L} \sup_{W_{j} \in \mathcal{B}_{j}, \forall j < i} \log \left(\mathcal{N}\left(\left\{W_{i} F_{\mathcal{W}_1^{i-1}}(Z)\right\},  \epsilon_{i},\Vert \cdot \Vert_{2,2}\right) \right) + \sup_{W_{j} \in \mathcal{B}_{j}, \forall j < i} \log \left(\mathcal{N}\left(\left\{W_{i} G_{\mathcal{W}_1^{i-1}}(Z)\right\}, \prod_{k=1}^i c_i \epsilon_{i},\Vert \cdot \Vert_{2,2}\right) \right)\\
&\leq \sum_{i=1}^{L} \sup_{W_{j} \in \mathcal{B}_{j}, \forall j < i} \frac{ b_i^2\Vert F_{\mathcal{W}_1^{i-1}}(Z)^\mathrm{T}\Vert_{2,2}^2\log(2h^2)}{\epsilon_i^2} + \sup_{W_{j} \in \mathcal{B}_{j}, \forall j < i} \frac{ b_i^2\Vert G_{\mathcal{W}_1^{i-1}}(Z)^\mathrm{T}\Vert_{2,2}^2\log(2h^2)}{\left(\prod_{k=1}^{i} c_k\right)^2\epsilon_i^2}.
\end{aligned}
\end{equation}
Set $\mathcal{B}_{i}=\left\{W_{i}:\left\|W_{i}\right\|_{2} \leq s_{i},\left\|W_{i}\right\|_{2,1} \leq b_{i}\right\}$, and set the per-layer cover resolutions $\left(\epsilon_{1}, \ldots, \epsilon_{L}\right)$ as
\begin{equation}
\epsilon_{i}:=\frac{\alpha_{i} \epsilon}{L \left(\prod_{k=1}^L s_k\right)\prod_{j>i}  s_{j}} \quad \text { where } \quad \alpha_{i}:=\frac{1}{\bar{\alpha}}\left(\frac{b_{i}}{s_{i}}\right)^{2 / 3} \quad, \quad \bar{\alpha}:=\sum_{j=1}^{L}\left(\frac{b_{j}}{s_{j}}\right)^{2 / 3}.
\end{equation}
By this choice, it follows that the final cover resolution $\tau$ provided by Lemma satisfies
\begin{equation}
\begin{aligned}
\tau 
&\leq L\left(\prod_{k=1}^L s_k\right) \sum_{j=1}^L \epsilon_{j}  \prod_{k=j+1}^{L} s_{k}\\
&\leq  \sum_{i=1}^L \frac{\alpha_{i} \epsilon}{\prod_{j>i}  s_{j}} \prod_{l=j+1}^{L} s_{l}\\
&\leq \sum_{j \leq L} \alpha_j \epsilon\\
&= \epsilon.
\end{aligned}
\end{equation}
Therefore, the covering number bound can be written as:
\begin{equation}
\begin{aligned}
&\quad \log \mathcal{N}\left(\mathcal{H}_{X}, \epsilon,\|\cdot\|_{2}\right) \\
& \leq  \sum_{i=1}^{L} \sup_{A_{j} \in \mathcal{B}_{j}, \forall j < i} \frac{ b_i^2\Vert F_{\mathcal{A}_1^{i-1}}(Z)^\mathrm{T}\Vert_{2,2}^2\log(2h^2)}{\epsilon_i^2} + \sup_{A_{j} \in \mathcal{B}_{j}, \forall j < i} \frac{ b_i^2\Vert G_{\mathcal{A}_1^{i-1}}(Z)^\mathrm{T}\Vert_{2,2}^2\log(2h^2)}{\left(\prod_{k=1}^{i} c_k\right)^2\epsilon_i^2} \\
&\leq \sum_{i=1}^{L} \frac{ b_i^2nd\log(2h^2)}{\epsilon_i^2} + \frac{ b_i^2\left(\prod_{k=1}^{i} s_k\right)^2\Vert Z \Vert_{2,2}^2\log(2h^2)}{\left(\prod_{k=1}^{i} s_k\right)^2\epsilon_i^2} \\
&\leq 2nd\log \left(2 h^{2}\right) \sum_{i=1}^{L} \frac{ b_i^2\log(2h^2)}{\epsilon_i^2}\\
&\leq 2nd\log \left(2 h^{2}\right)L^2\left(\prod_{k=1}^L s_k\right)^2  \sum_{i=1}^{L} \frac{ b_i^2\left(\prod_{j>i}  s_{j}\right)^2}{\epsilon^2}\\
&\leq2nd\log \left(2 h^{2}\right)L^2\left(\prod_{k=1}^L s_k\right)^4  \left(\bar{\alpha}^{3}\right)/\epsilon^2.
\end{aligned}
\end{equation}
Consequently, the Rademacher complexity of the function class constructed by all first-order derivatives is bounded by
\begin{equation}
\frac{4}{n\sqrt{n}} + \frac{18L\sqrt{2d\log(2h^2)}\log n}{\sqrt{n}}\left(\prod_{l=1}^L M(l)\right)^2 \left(\sum_{l=1}^L N(l)^{2/3}\right)^{3/2}.
\end{equation} 
\end{proof}

\begin{lemma}
The hypothesis class $\mathcal{F}_1$ given by
\begin{equation}
\begin{aligned}
\mathcal{F}_1 &= \left\{\bx \mapsto \sum_{j=1}^d \sum_{l=1}^{L-1} 
{W}_L{\Phi}_{L-1}\cdots{W}_{l+1}
\text{diag}({\Psi}_l{W}_l\cdots{\Psi}_1({W}_1)_{:,j})
{W}_l\cdots {\Phi}_1{W}_1A_{:,j}(\bx)  \right\},
\end{aligned}
\end{equation}
where the weight matrices satisfy
\begin{equation}
\Vert W_l \Vert_{2} \leq M(l), \frac{\Vert W_l \Vert_{2,1}}{\Vert W_l \Vert_{2}} \leq N(l), \forall l,
\end{equation}
satisfies the Rademacher complexity bound
\begin{equation}
\text{Rad}(\mathcal{F}_1;S) \leq \frac{4}{n\sqrt{n}} + \frac{18(L+1)\sqrt{2d\log(2h^2)}\log n}{\sqrt{n}}\left(\prod_{l=1}^L M(l)\right)^3 \left(\sum_{l=1}^L N(l)^{2/3}\right)^{3/2}.
\end{equation} 
\end{lemma}
\begin{proof}
We consider the $d(L-1)$ terms one-by-one, and focus on the following term in particular:
\begin{equation}
{W}_L{\Phi}_{L-1}\cdots{W}_{i+1}
\text{diag}({\Psi}_i{W}_i\cdots{\Psi}_1({W}_1)_{:,j})
{W}_i\cdots {\Phi}_1{W}_1A_{:,j}(Z),
\end{equation}
where we recall that $A_{:j}$ is the $j$-th row of the fixed coefficient function and $A_{:,j}(Z) \in \mathbb{R}^{d \times n}$. We consider the covering in Lemma \ref{Lemma:Covering}. Specifically, denote $\mathcal{F}_i$ as the $\epsilon_i$ covering of the following set:
\begin{equation}
\left\{W_{i} F_{\mathcal{W}_1^{i-1}}(Z): W_{j} \in \mathcal{B}_{j}, \forall j \leq i\right\},
\end{equation}
which is the set of all $i$-layer neural networks, with weight matrices in the set $\mathcal{B}_{j}$. In Lemma \ref{Lemma:Covering}, we have inductively constructed covers $\mathcal{F}_{1}, \ldots, \mathcal{F}_{L}$. Specifically, $\mathcal{F}_{1}$ is an $\epsilon_{1}$-cover of the set of one-layer neural networks $\left\{W_{1} Z: W_{1} \in \mathcal{B}_{1}\right\}$, and its cardinality satisfies
\begin{equation}
\left|\mathcal{F}_{1}\right| \leq \mathcal{N}\left(\left\{W_{1} Z: W_{1} \in \mathcal{B}_{1}\right\}, \epsilon_{1},\Vert \cdot \Vert_{2,2}\right).
\end{equation} 
For every element $F \in \mathcal{F}_{i}$, which is a covering vector of the set of image of all $i$-layer neural networks, we construct an $\epsilon_{i+1}$-cover $\mathcal{G}_{i+1}(F)$ of the set
\begin{equation}
\left\{W_{i+1} \sigma(F): W_{i+1} \in \mathcal{B}_{i+1}\right\},
\end{equation}
where $F \in \mathcal{F}_i$ is chosen and fixed. Then, we obtain
\begin{equation}
\left|\mathcal{G}_{i+1}(F)\right| \leq \sup _{\forall j \leq i . W_{j} \in \mathcal{B}_{j}} \mathcal{N}\left(\left\{W_{i+1} F_{W_{1}, \ldots, W_{i}}(Z): W_{i+1} \in \mathcal{B}_{i+1}\right\}, \epsilon_{i+1}, \Vert \cdot \Vert_{2,2}\right).
\end{equation}
Lastly we construct the cover
\begin{equation}
\mathcal{F}_{i+1}:=\bigcup_{F \in \mathcal{F}_{i}} \mathcal{G}_{i+1}(F),
\end{equation}
whose cardinality satisfies
\begin{equation}
\left|\mathcal{F}_{i+1}\right| \leq\left|\mathcal{F}_{i}\right| \cdot N_{i+1} \leq \prod_{l=1}^{i+1} N_{l}.
\end{equation}
We consider how to cover the term $\Phi_i$ within the PINN. We note that $\Phi_i = \Phi_i(Z_j) \in \mathbb{R}_{m_i \times m_i}$ is a diagonal matrix:
\begin{equation}
\text{diag}(\Phi_i(Z_j)) = \sigma'(W_{i}\cdots\sigma(W_1Z_j)),
\end{equation}
where the which is a $i$-layer neural network in particular. Therefore, $\sigma'(\mathcal{F}_{i})$ can cover $\sigma'(W_{i}\cdots\sigma(W_1Z))$ with error not larger than $\sum_{j \leq i} \epsilon_{j} \prod_{l=j+1}^{i}c_{l}$. Similarly, since $\Psi_i$ is also a $i$-layer neural net, we can do the same reasoning to it. In these covering, we use the fact that $\sigma'$ and $\sigma''$ are 1-Lipschitz functions.

Then, we consider the $\left(\prod_{k=1}^i c_k \right)\epsilon_i$-cover of the following set of intermediate output of a PINN:
\begin{equation}
\left\{W_{i} G_{\mathcal{W}_1^{i-1}}(Z): W_{j} \in \mathcal{B}_{j}, \forall j \leq i\right\},
\end{equation}
where
\begin{equation}
G_{\mathcal{W}_1^{i-1}}(Z) = \left[\Phi_{i-1}(z_j) W_{i-1} \dots \Phi_1(z_j) W_1 A_{:,}(z_j)\right]_{j=1}^n \in \mathbb{R}^{m_{i-1} \times n}.
\end{equation}
More concretely, we shall inductively construct covers $\mathcal{G}_{1}, \ldots, \mathcal{G}_{L}$. Specifically, $\mathcal{G}_{1}$ is an $c_1 \epsilon_{1}$-cover of the set of one-layer neural networks $\left\{W_{1} A_{:,}(Z) \in \mathbb{R}^{m_1 \times n}: W_{1} \in \mathcal{B}_{1}\right\}$, and its cardinality satisfies
\begin{equation}
\left|\mathcal{G}_{1}\right| \leq \mathcal{N}\left(\left\{W_{1} A_{:,}(Z): W_{1} \in \mathcal{B}_{1}\right\}, c_1\epsilon_{1},\Vert \cdot \Vert_{2,2}\right).
\end{equation} 
For every element $G \in \mathcal{G}_{i}$, and $F \in \mathcal{F}_i$, we construct an $\left(\prod_{k=1}^{i+1}c_k\right)\epsilon_{i+1}$-cover of the set
\begin{equation}
\mathcal{H}(F, G) = \left\{\left[W_{i+1}\text{diag}(F(Z_j))G\right]_{j=1}^n \in \mathbb{R}^{m_{i+1} \times n}: W_{i+1} \in \mathcal{B}_{i+1}\right\},
\end{equation}
where $G$ and $F$ are chosen and fixed. We stree that $F(Z_j) \in \mathbb{R}^{m_i}$, and thus $\text{diag}(F(Z_j)) \in \mathbb{R}^{m_i \times m_i}$. Since for each input data, the matrix $F(Z_j)$ depends on itself, we cannot write the expression of the functions in $\mathcal{H}(F,G)$ as a direct matrix multiplication. We denote
\begin{equation}
\mathcal{G}_{i+1} := \bigcup_{F \in \mathcal{F}_{i}, G \in \mathcal{G}_{i}}\mathcal{H}\left(F,G\right).
\end{equation}
whose cardinality satisfies
\begin{equation}
|\mathcal{H}(F,G)| \leq \sup _{\forall j \leq i, W_{j} \in \mathcal{B}_{j}} \mathcal{N}\left(\left\{W_{i+1} G_{W_{1}^{i}}(Z)\right\}, \left(\prod_{k=1}^{i+1}c_k\right)\epsilon_{i+1}, \Vert\cdot\Vert_{2,2}\right)=: M_{i+1}.
\end{equation}
Consequently,
\begin{equation}
|\mathcal{G}_{i+1}| \leq \prod_{l=1}^{i+1} M_{l} N_{l}.
\end{equation}
For the first layer, we have
\begin{equation}
\Vert G_1 - {W}_1 A_{:,}(Z) \Vert_{2,2} \leq c_1\epsilon_1,
\end{equation}
for some $G_1 \in \mathcal{G}_1$, due to its definition. For general cases, consider the following inequalities:
\begin{equation}
\begin{aligned}
&\quad \Vert\hat{\Phi}_{l} \hat{G}_{l} - {\Phi}_{l} \cdots {\Phi}_1 {W}_1 A_{:,}(Z) \Vert_{2,2}\\
&=\Vert\hat{\phi}_{l} \odot \hat{G}_{l} - {\phi}_{l} \odot W_{l}\cdots {\Phi}_1 {W}_1 A_{:,}(Z) \Vert_{2,2}\\ 
&\leq \Vert (\hat{\phi}_{l} -  {\phi}_{l}) \odot \hat{G}_{l}\Vert_{2,2} + \Vert  {\phi}_{l}\odot(\hat{G}_l - {W}_{l}\cdots {\Phi}_1 {W}_1 A_{:,}(Z)) \Vert_{2,2}\\
&\leq \Vert (\hat{\phi}_{l} - \phi_{l} ) \Vert_{2,2}\Vert \hat{G}_{l} \Vert_{\infty,\infty}+\Vert  {\phi}_{l}\Vert_{\infty,\infty}\Vert\hat{G}_{l} - {W}_{l}\cdots {\Phi}_1 {W}_1 A_{:,}(Z)) \Vert_{2,2}\\
&\leq \left(\prod_{k=1}^l c_k\right) \sum_{j=1}^l \epsilon_{j}  \prod_{k=j+1}^{l} c_{k} + 
l\left(\prod_{k=1}^l c_k\right) \sum_{j=1}^l \epsilon_{j}  \prod_{k=j+1}^{l} c_{k}\\
&\leq (l+1)\left(\prod_{k=1}^l c_k\right) \sum_{j=1}^l \epsilon_{j}  \prod_{k=j+1}^{l} c_{k},
\end{aligned}
\end{equation}
where $\hat{\Phi}_l \in \mathcal{F}_i$ and $\phi = \text{diag}(\Phi)$. More concretely, we only each data point one-by-one in the above reasoning, i.e.,
\begin{equation}
\begin{aligned}
&\quad \Vert\hat{\Phi}_{l} \hat{G}_{l} - {\Phi}_{l} \cdots {\Phi}_1 {W}_1 A_{:,}(Z) \Vert_{2,2}\\
&=\Vert\left[ \hat{\phi}_{l}(Z_j) \odot \hat{G}_{l}(Z_j) - {\phi}_{l}(Z_j) \odot W_{l}\cdots {\Phi}_1(Z_j) {W}_1 A_{:,}(Z_j) \right]_{j=1}^n\Vert_{2,2}\\ 
&\leq \Vert \left[ (\hat{\phi}_{l}(Z_j) -  {\phi}_{l}(Z_j)) \odot \hat{G}_{l}(Z_j)\right]_{j=1}^n \Vert_{2,2} + \Vert \left[ {\phi}_{l}(Z_j)\odot(\hat{G}_l - {W}_{l}\cdots {\Phi}_1 {W}_1 A_{:,}(Z_j)) \right]_{j=1}^n\Vert_{2,2}\\
&\leq \Vert (\hat{\phi}_{l}(Z) - \phi_{l}(Z) ) \Vert_{2,2}\Vert \hat{G}_{l}(Z) \Vert_{\infty,\infty}+\Vert  {\phi}_{l}(Z)\Vert_{\infty,\infty}\Vert\hat{G}_{l} - {W}_{l}\cdots {\Phi}_1 {W}_1 A_{:,}(Z)) \Vert_{2,2}\\
&\leq \left(\prod_{k=1}^l c_k\right) \sum_{j=1}^l \epsilon_{j}  \prod_{k=j+1}^{l} c_{k} + 
l\left(\prod_{k=1}^l c_k\right) \sum_{j=1}^l \epsilon_{j}  \prod_{k=j+1}^{l} c_{k}\\
&\leq (l+1)\left(\prod_{k=1}^l c_k\right) \sum_{j=1}^l \epsilon_{j}  \prod_{k=j+1}^{l} c_{k},
\end{aligned}
\end{equation}
Furthermore, consider the new $W_l$, we attain
\begin{equation}
\begin{aligned}
&\quad \Vert\hat{G}_{l+1} - W_{l+1}{\Phi}_{l} \cdots {\Phi}_1 {W}_1 A_{:,}(Z) \Vert_{2,2}\\
&\leq \Vert\hat{G}_{l+1} - W_{l+1}\hat{\Phi}_{l} \hat{G}_{l-1} \Vert_{2,2} + \Vert W_{l+1}\hat{\Phi}_{l} \hat{G}_{l-1} - W_{l+1}{\Phi}_{l} \cdots {\Phi}_1 {W}_1 A_{:,}(Z)\Vert_{2,2}\\
&\leq \left(\prod_{k=1}^{l+1} c_k\right) \epsilon_{l+1} + (l+1)\left(\prod_{k=1}^{l+1} c_k\right) \sum_{j=1}^l \epsilon_{j}  \prod_{k=j+1}^{l} c_{k}\\
&\leq (l+1)\left(\prod_{k=1}^{l+1} c_k\right) \sum_{j=1}^{l+1} \epsilon_{j}  \prod_{k=j+1}^{l+1} c_{k}.
\end{aligned}
\end{equation}
Consider the case when $l=i$, and in particular the term $\Omega_i = \text{diag}({\Psi}_i{W}_i\cdots{\Psi}_1({W}_1)_{:,j})$, then the PINN can be rewritten as
\begin{equation}
{W}_L{\Phi}_{L-1}\cdots{W}_{i+1}\Omega_i{W}_i\cdots {\Phi}_1{W}_1A_{:,j}(Z),
\end{equation}
where we note that the only difference between the first-order derivatives and the second-order ones is in the term $\Omega_i$.
\begin{equation}
\begin{aligned}
&\quad \Vert\hat{\Omega}_{i} \hat{G}_{i} - {\Omega}_{i} \cdots {\Phi}_1 {W}_1 A_{:,}(Z) \Vert_{2,2}\\
&=\Vert\left[\hat{\omega}_{i} \odot \hat{G}_{i} - {\omega}_{i} \odot W_{i}\cdots {\Phi}_1 {W}_1 A_{:,}(Z_j)\right]_{j=1}^n \Vert_{2,2}\\ 
&\leq \Vert \left[ (\hat{\omega}_{i}(Z_j) -  {\omega}_{i}(Z_j)) \odot \hat{G}_{i}(Z_j)\right]_{j=1}^n\Vert_{2,2} + \Vert \left[ {\omega}_{i}(Z_j)\odot(\hat{G}_i - {W}_{i}\cdots {\Phi}_1 {W}_1 A_{:,}(Z_j))\right] \Vert_{2,2}\\
&\leq \Vert (\hat{\omega}_{i} - \omega_{i} ) \Vert_{2,2}\Vert \hat{G}_{i} \Vert_{\infty,\infty}+\Vert  {\omega}_{i}\Vert_{\infty,\infty}\Vert\hat{G}_{i} - {W}_{i}\cdots {\Phi}_1 {W}_1 A_{:,}(Z)) \Vert_{2,2}\\
&\leq \left(\prod_{k=1}^i c_k\right) \text{Approximate Error of Omega} + 
i\left(\prod_{k=1}^i c_k\right)^2 \sum_{j=1}^i \epsilon_{j} \prod_{k=j+1}^{i} c_{k}.
\end{aligned}
\end{equation}
Based on our discussions in the previous lemma on first-order derivatives, we know that the constructed function class can cover the PINN before the $i$-th layer with an error no larger than $i\left(\prod_{k=1}^i c_k\right) \sum_{j=1}^i \epsilon_{j}  \prod_{k=j+1}^{i} c_{k}$, i.e., the approximation error of $\Omega_l$ should be that quantity. Thus, we proceed with our bound and attain:
\begin{equation}
\begin{aligned}
&\quad \Vert\hat{\Omega}_{l} \hat{G}_{l} - {\Omega}_{l} \cdots {\Phi}_1 {W}_1 A_{:,}(Z) \Vert_{2,2}\\
&\leq \left(\prod_{k=1}^i c_k\right) \text{Approximate Error of Omega} + 
i\left(\prod_{k=1}^i c_k\right)^2 \sum_{j=1}^i \epsilon_{j} \prod_{k=j+1}^{i} c_{k}\\
&\leq i\left(\prod_{k=1}^i c_k\right)^2 \sum_{j=1}^i \epsilon_{j}  \prod_{k=j+1}^{i} c_{k} + 
i\left(\prod_{k=1}^i c_k\right)^2 \sum_{j=1}^i \epsilon_{j} \prod_{k=j+1}^{i} c_{k}\\
&\leq 2i\left(\prod_{k=1}^i c_k\right)^2 \sum_{j=1}^i \epsilon_{j} \prod_{k=j+1}^{i} c_{k}.
\end{aligned}
\end{equation}
Then, we guess the approximation error after the $l$-th layer is:
\begin{equation}
2l\left(\prod_{k=1}^l c_k\right)^2 \sum_{j=1}^l \epsilon_{j} \prod_{k=j+1}^{i} c_{k}.
\end{equation}
For $l \leq i$, the above bound has already been shown. For the case when $l = i + 1$, consider the new $W_{i+1}$, we attain
\begin{equation}
\begin{aligned}
&\quad \Vert\hat{G}_{i+1} - W_{i+1}{\Omega}_{i} \cdots {\Phi}_1 {W}_1 A_{:,}(Z) \Vert_{2,2}\\
&\leq \Vert\hat{G}_{i+1} - W_{i+1}\hat{\Omega}_{i} \hat{G}_{i} \Vert_{2,2} + \Vert W_{i+1}\hat{\Omega}_{i} \hat{G}_{i-1} - W_{i+1}{\Omega}_{i} \cdots {\Phi}_1 {W}_1 A_{:,}(Z)\Vert_{2,2}\\
&\leq \left(\prod_{k=1}^{i+1} c_k\right) \epsilon_{i+1} +2ic_{i+1}\left(\prod_{k=1}^i c_k\right)^2 \sum_{j=1}^i \epsilon_{j} \prod_{k=j+1}^{i} c_{k}\\
&\leq 2i\left(\prod_{k=1}^{i+1} c_k\right)^2 \sum_{j=1}^{i+1} \epsilon_{j}  \prod_{k=j+1}^{i+1} c_{k}.
\end{aligned}
\end{equation}
For the $l$-th layer after the $i$-th layer, there is no $\Omega$ terms any more, and we only need to focus on $\Phi$:
\begin{equation}
\begin{aligned}
&\quad \Vert\hat{\Phi}_{l} \hat{G}_{l} - {\Phi}_{l} \cdots {\Phi}_1 {W}_1 A_{:,}(Z) \Vert_{2,2}\\
&=\Vert\hat{\phi}_{l} \odot \hat{G}_{l} - {\phi}_{l} \odot W_{l}\cdots {\Phi}_1 {W}_1 A_{:,}(Z) \Vert_{2,2}\\ 
&\leq \Vert (\hat{\phi}_{l} -  {\phi}_{l}) \odot \hat{G}_{l}\Vert_{2,2} + \Vert  {\phi}_{l}\odot(\hat{G}_l - {W}_{l}\cdots {\Phi}_1 {W}_1 A_{:,}(Z)) \Vert_{2,2}\\
&\leq \Vert (\hat{\phi}_{l} - \phi_{l} ) \Vert_{2,2}\Vert \hat{G}_{l} \Vert_{\infty,\infty}+\Vert  {\phi}_{l}\Vert_{\infty,\infty}\Vert\hat{G}_{l} - {W}_{l}\cdots {\Phi}_1 {W}_1 A_{:,}(Z)) \Vert_{2,2}\\
&\leq \left(\prod_{k=1}^l c_k\right)^2 \sum_{j=1}^l \epsilon_{j}  \prod_{k=j+1}^{l} c_{k} + 
2l\left(\prod_{k=1}^l c_k\right)^2 \sum_{j=1}^l \epsilon_{j}  \prod_{k=j+1}^{l} c_{k}\\
&\leq2 (l+1)\left(\prod_{k=1}^l c_k\right)^2 \sum_{j=1}^l \epsilon_{j}  \prod_{k=j+1}^{l} c_{k}.
\end{aligned}
\end{equation}
And we consider the multiplication of $W_{l+1}$:
\begin{equation}
\begin{aligned}
&\quad \Vert\hat{G}_{l+1} - W_{l+1}{\Phi}_{l} \cdots {\Phi}_1 {W}_1 A_{:,}(Z) \Vert_{2,2}\\
&\leq \Vert\hat{G}_{l+1} - W_{l+1}\hat{\Phi}_{l} \hat{G}_{l-1} \Vert_{2,2} + \Vert W_{l+1}\hat{\Phi}_{l} \hat{G}_{l-1} - W_{l+1}{\Phi}_{l} \cdots {\Phi}_1 {W}_1 A_{:,}(Z)\Vert_{2,2}\\
&\leq \left(\prod_{k=1}^{l+1} c_k\right) \epsilon_{l+1} + 2 (l+1)\left(\prod_{k=1}^{l+1} c_k\right)^2 \sum_{j=1}^l \epsilon_{j}  \prod_{k=j+1}^{l} c_{k}\\
&\leq 2(l+1)\left(\prod_{k=1}^{l+1} c_k\right)^2 \sum_{j=1}^{l+1} \epsilon_{j}  \prod_{k=j+1}^{l+1} c_{k}.
\end{aligned}
\end{equation}
After covering the class of second-order derivatives in the PINN model, by the same logic, we have the Rademacher complexity of the function class constructed by all second-order derivatives is bounded by
\begin{equation}
\frac{4}{n\sqrt{n}} + \frac{18(L+1)\sqrt{2d\log(2h^2)}\log n}{\sqrt{n}}\left(\prod_{l=1}^L M(l)\right)^3 \left(\sum_{l=1}^L N(l)^{2/3}\right)^{3/2}.
\end{equation} 
\end{proof}

\subsection{Tree-Like Function Space}
Barron space is an important pat in this paper. The spectral norm based bound can be connected with the Barron space, and thus provide Rademacher complexity bound for tree-like functions in the Barron space.

To illustrate the idea, we use the original network, \begin{equation}
\mathcal{NN}^L_{M} := \left\{\boldsymbol{x} \mapsto W_L \sigma (W_{L-1} \sigma(\cdots \sigma(W_1\boldsymbol{x})) )\ |\ \Vert W_l \Vert_{2} \leq M(l), \frac{\Vert W_l \Vert_{2,1}}{\Vert W_l \Vert_{2}} \leq N(l), \forall l \right\},
\end{equation}
which satisfies the Rademacher complexity bound
\begin{equation}
\text{Rad}(\mathcal{NN}^L_{M};S) \leq \prod_{l=1}^L M(l)\Big(\sum_{l=1}^L N(l)^{2/3}\Big)^{3/2}\frac{\log(2h^2)\log n}{{n}^{1/2}},
\end{equation} 
where $h$ is the maximal width of the neural network, i.e.,
\begin{equation}
h = \max(m_L, \cdots, m_0).
\end{equation}
In other words, the Rademacher complexity is related to the following quantity:
\begin{equation}
\prod_{l=1}^L \Vert W_l \Vert_2\Big(\sum_{l=1}^L \left(\frac{\Vert W_l \Vert_{2,1}}{\Vert W_l \Vert_{2}}\right)^{2/3}\Big)^{3/2}\frac{\log(2h^2)\log n}{{n}^{1/2}},
\end{equation} 
which can be upper bounded by the path norm or $(1,\infty)$ norm related to the Barron space. Concretely, all matrix norms are equivalent, which means there exists a constant $C(h)$ that depends on the maximal width $h$, such that
\begin{equation}
\prod_{l=1}^L \Vert W_l \Vert_2\Big(\sum_{l=1}^L \left(\frac{\Vert W_l \Vert_{2,1}}{\Vert W_l \Vert_{2}}\right)^{2/3}\Big)^{3/2} \leq C(h)\prod_{l=1}^L\ \Vert W_l \Vert_{1,\infty}.
\end{equation} 
These intuitions are summarized in the following theorem.
\begin{lemma}
(Rademacher Complexity of Tree-Like Functions). For every $L$, and every set of $n$ points $S \subset \overline{\Omega}$, the hypothesis class $\mathcal{NN}^L_M$ given by the neural networks
\begin{equation}
\mathcal{NN}^L_{M} := \left\{\boldsymbol{x} \mapsto W_L \sigma (W_{L-1} \sigma(\cdots \sigma(W_1\boldsymbol{x})) )\ |\ \Vert W_l \Vert_{1,\infty} \leq M(l),\forall l \right\},
\end{equation}
satisfies the Rademacher complexity bound
\begin{equation}
\text{Rad}(\mathcal{NN}^L_{M};S) \leq \left(\prod_{l=1}^L \Vert W_l \Vert_{1,\infty}\right)\frac{C(h)\log n}{{n}^{1/2}},
\end{equation} 
where $h$ is the maximal width of the neural network, i.e.,
\begin{equation}
h = \max(m_L, \cdots, m_0),
\end{equation}
and $C(h)$ is a universal constant depending only on $h$.
\end{lemma}
\begin{lemma}
(Rademacher Complexity of Tree-Like Functions). For every $L$, and every set of $n$ points $S \subset \overline{\Omega}$, the hypothesis class $\mathcal{PINN}^L_M$ given by the neural networks
\begin{equation}
\mathcal{PINN}^L_{M} := \left\{\boldsymbol{x} \mapsto \mathcal{L}u_{\bt}(\bx)\ |\ u_{\bt} \in \mathcal{NN}_M^L \right\},
\end{equation}
satisfies the Rademacher complexity bound
\begin{equation}
\text{Rad}(\mathcal{NN}^L_{M};S) \leq \left(\prod_{l=1}^L \Vert W_l \Vert_{1,\infty}\right)^3\frac{C(h,K)\log n}{{n}^{1/2}},
\end{equation} 
where $h$ is the maximal width of the neural network, i.e.,
\begin{equation}
h = \max(m_L, \cdots, m_0),
\end{equation}
and $C(h, K3)$ is a universal constant depending only on $h, K$.
\end{lemma}

\section{Proofs of Main Results}
\subsection{Proof of Theorem \ref{thm:generalization}}
\begin{proof}
(Proof of Theorem \ref{thm:generalization})
Let $u_{\hat{\bt}}$ parameterized by $\hat{\bt}$ satisfy the conditions in Theorem \ref{thm:approximation}, i.e., 
\begin{equation}
\begin{aligned}
\Vert u_{\bt} - f\Vert_{H^2(\mathbb{P})} &\leq \frac{3L\Vert f \Vert_{\mathcal{W}^L(\Omega)}}{\sqrt{m}},\\
\Vert u_{\bt} - f\Vert_{L^2(\mathbb{Q})} &\leq \frac{3C_{\Omega}L\Vert f \Vert_{\mathcal{W}^L(\Omega)}}{\sqrt{m}},\\
\Vert \hat{W}_L \Vert_{1,\infty} &\leq  \Vert u^* \Vert_{\mathcal{W}^L(\Omega)}.
\end{aligned}
\end{equation}
where $\Vert \hat{W}_L \Vert_{1,\infty}$ is the $L$-th layer weight parameter matrix of $\hat{\bt}$, and take the probability measures $\mathbb{P}, \mathbb{Q}$ as
\begin{equation}
\begin{aligned}
\mathbb{P} = \frac{1}{n_r}\sum_{\boldsymbol{x} \in S \cap \Omega} \delta_{\boldsymbol{x}}, \qquad
\mathbb{Q} = \frac{1}{n_b}\sum_{\boldsymbol{x} \in S \cap \partial \Omega} \delta_{\boldsymbol{x}},
\end{aligned}
\end{equation}
which means $\mathbb{P}$ contains the empirical distribution of residual points and $\mathbb{Q}$ contains the empirical distribution of boundary points. Thus
\begin{equation}
\begin{aligned}
R_S(\theta) &= \frac{1}{n_b}\sum_{i=1}^{n_b} {\left(u_{\hat{\bt}}(\boldsymbol{x}_{b,i})-g(\boldsymbol{x}_{b,i})\right)}^2+\frac{1}{n_r}\sum_{i=1}^{n_r} {\left(\mathcal{L}u_{\hat{\bt}}(\boldsymbol{x}_{r,i})-f(\boldsymbol{x}_{r,i})\right)}^2\\
&\leq \Vert u_{\bt} - u^*\Vert_{H^2(\mathbb{P})} + 2K\Vert u_{\bt} - u^*\Vert_{L^2(\mathbb{Q})}\\
&\leq \frac{3(2KC_{\Omega}+1)L\Vert f \Vert_{\mathcal{W}^L(\Omega)}}{\sqrt{m}}
\end{aligned}
\end{equation}
Then, by the fact that $\boldsymbol{\theta}^* = \arg \min_{\bt}  R_S(\boldsymbol{\theta}) + \lambda \Vert W_L \Vert^2_{1,\infty}$ and $\lambda = {3(2KC_{\Omega}+1)L^2}/{m}$, we have
\begin{equation}
\begin{aligned}
R_S(\bt^*) + \lambda \Vert W_L^* \Vert_{1, \infty}^2 &\leq R_S(\hat{\bt}) + \lambda \Vert \hat{W}_L \Vert_{1,\infty}^2\\
&\leq \frac{3 (2KC_{\Omega}+1) L^2\Vert u^* \Vert^2_{\mathcal{W}^L(\Omega)}}{m} + \lambda \Vert u^* \Vert^2_{\mathcal{W}^L(\Omega)}\\
&= 2 \lambda \Vert u^* \Vert_{\mathcal{W}^L(\Omega)}.
\end{aligned}
\end{equation}
In particular,
\begin{equation}
\Vert u_{\bt^*} \Vert_{\mathcal{W}^L} \leq \Vert W_L^* \Vert_{1, \infty} \leq \frac{2 \lambda \Vert u^* \Vert_{\mathcal{W}^L(\Omega)}}{\lambda} = 2 \Vert u^* \Vert_{\mathcal{W}^L(\Omega)}.
\end{equation}
The Rademacher complexity of the neural network model used for the boundary points prediction is upper bounded by
\begin{equation}
\Vert W_L \Vert_{1,\infty}\frac{C(h)\log n_b}{{\sqrt{n_b}}}\leq 2\Vert u^* \Vert_{\mathcal{W}^L(\Omega)}\frac{C(h)\log n_b}{{\sqrt{n_b}}}.
\end{equation} 
Hence, consider the function class of the composition of the MSE loss function and the neural network model. Since we have truncated the neural network function to $[-1,1]$, i.e.,
\begin{equation}
l(\bx, \bx') = \frac{1}{2}\Vert \bx - \bx' \Vert_2^2 \leq 2,
\end{equation}
we know that the loss function is $\overline{c}$-Lipschitz. Therefore, we can attain the following generalization bound for boundary points prediction:
\begin{equation}
R_{D \cap \partial \Omega}(\bt^*) \leq R_{S \cap \partial \Omega}(\bt^*) + 16\Vert u^* \Vert_{\mathcal{W}^L(\Omega)}\frac{C(h)\log n_b}{{\sqrt{n_b}}} + 2\sqrt{\frac{2\log(2/\delta)}{n_b}}.
\end{equation}
Similarly, Rademacher complexity of the PINN used for residual points, i.e., the differentiated networks are upper bounded by
\begin{equation}
\left(\Vert W_L \Vert_{1,\infty}\right)^3\frac{C(h,K)\log n_r}{{\sqrt{n_r}}},
\end{equation} 
This is due to the fact that for the neural networks in the tree-like function space section, we force $\Vert W_l \Vert_{1,\infty} \leq 1$, for all $1 \leq l \leq L-1$. Thus, only $\Vert W_L\Vert_{1,\infty}$ matters for its Rademacher complexity.
Therefore, the generalization bound for the residual loss is:
\begin{equation}
R_{D \cap \Omega}(\bt^*) \leq R_{S \cap \Omega}(\bt^*) + 16\left(\Vert u^* \Vert_{\mathcal{W}^L(\Omega)}\right)^3\frac{C(h,K)\log n_r}{{\sqrt{n_r}}} + 2\sqrt{\frac{2\log(2/\delta)}{n_r}}.
\end{equation}
In sum, we obtain the two generalization bounds on the boundary and in the residual, respectively.
\end{proof}

\subsection{Proof of Theorem \ref{thm:post_generalization}}
\begin{proof}
(Proof of Theorem \ref{thm:post_generalization}). Consider the function class
\begin{equation}
\mathcal{H}_{M,N}^L = \left\{\boldsymbol{x} \mapsto l(u(\boldsymbol{x}), u_{\bt}(\boldsymbol{x})) \ |\ \Vert \boldsymbol{W}^l \Vert_{2} \leq M(l), \frac{\Vert \boldsymbol{W}^l \Vert_{2,1}}{\Vert \boldsymbol{W}^l \Vert_{2}} \leq N(l), l=1,...,L \right\},
\end{equation}
where $M(1),...,M(L)$ and $N(1),...,N(L)$ are positive integers, $M$ and $N$ are the collection of all $M(1),...,M(L)$ and $N(1),...,N(L)$, respectively. And $l(\cdot,\cdot)$ is the mean square error (MSE) loss function, $u(\bx)$ is the PDE solution, and $\boldsymbol{W}^l$ is the $l$-th layer weight matrix of neural network $u_{\bt}(\bx)$. Then the class of composition of all $L$ layers neural networks and the loss function is 
\begin{equation}
\mathcal{H}^L = \cup_{M(1)=1}^\infty \cdots \cup_{M(L)=1}^\infty \cup_{N(1)=1}^\infty \cdots \cup_{N(L)=1}^\infty \mathcal{H}_{M,N}^L,
\end{equation}
where $M = (M(1), \cdots, M(L))$ and $N = (N(1), \cdots, N(L))$. Therefore, we subdivide $\delta > 0$ into
\begin{equation}
\delta(M,N) = \frac{\delta}{\left[\prod_{l=1}^L M(l)(M(l)+1)\right]\left[\prod_{l=1}^L N(l)(N(l)+1)\right]},
\end{equation}
such that
\begin{equation}
\sum_{M(1)=1}^\infty \cdots \sum_{M(L)=1}^\infty \sum_{N(1)=1}^\infty \cdots \sum_{N(L)=1}^\infty \delta(M,N) = \delta.
\end{equation}
By the result of Rademacher complexity of neural networks in Lemmas \ref{lemma:rad_nn} and \ref{lemma:rad_pinn}, for any given $\delta$ and any positive integers $M(1),...,M(L)$ and $N(1),...,N(L)$ with probability at least $1-\delta(M,N)$ over $S$, we have
\begin{equation}
\begin{aligned}
R_{D\cap\partial\Omega}&(\boldsymbol{\theta})-R_{S\cap\partial\Omega}(\boldsymbol{\theta}) \leq 8\mathbb{E}_S\text{Rad}(\mathcal{NN}_{M,N}^L;S) + 2\sqrt{\frac{\log(2/\delta(M,N))}{2n_b}}\\
&\leq \frac{32}{n_b\sqrt{n_b}} + \frac{144\sqrt{d\log(2h^2)}\log n_b}{\sqrt{n_b}} \prod_{l=1}^L M(l)\Big(\sum_{l=1}^L N(l)^{2/3}\Big)^{3/2}+ 2 \sqrt{\frac{\log(2/\delta(M,N))}{2n_b}}.
\end{aligned}
\end{equation}
For any parameter $\boldsymbol{\theta}$ minimizes the empirical loss, choose the integers $M(1),...,M(L)$ and $N(1),...,N(L)$ such that
\begin{equation}
\begin{aligned}
M(l) - 1 &< \Vert \boldsymbol{W}^l \Vert_{2} \leq M(l),\\
N(l) - 1 &<  \frac{\Vert\boldsymbol{W}^l \Vert_{2,1}}{\Vert\boldsymbol{W}^l \Vert_{2}} \leq N(l),
\end{aligned}
\end{equation}
and the integers are the smallest integers satisfying the above equations.
Then we have
\begin{equation}
R_{D\cap\partial\Omega}(\boldsymbol{\theta})\leq R_{S\cap\partial\Omega}(\boldsymbol{\theta})+ \frac{32}{n_b\sqrt{n_b}} + \frac{144\sqrt{d\log(2h^2)}\log n_b}{\sqrt{n_b}} \prod_{l=1}^L M(l)\Big(\sum_{l=1}^L N(l)^{2/3}\Big)^{3/2}+ 2 \sqrt{\frac{\log(2/\delta(M,N))}{2n_b}} ,
\end{equation}
where we note that $M(l) = \lceil \Vert \boldsymbol{W}^l \Vert_{2} \rceil$, and $N(l) = \lceil \Vert \boldsymbol{W}^l \Vert_{2,1} / \Vert \boldsymbol{W}^l \Vert_{2} \rceil$, in which $\lceil a \rceil$ of $a \in \mathbb{R}$ is the smallest integer that is greater than or equal to $a$.

The above bound just holds with probability $1-\delta(M,N)$ for any pair $(\boldsymbol{\theta},M,N)$ as long as $\bt$ satisfies $M(l) = \lceil \Vert \boldsymbol{W}^l \Vert_{2} \rceil$, and $N(l) = \lceil \Vert \boldsymbol{W}^l \Vert_{2,1} / \Vert \boldsymbol{W}^l \Vert_{2} \rceil$. Since $\sum_{M,N} \delta(M, N) = \delta$, the bound holds with probability $1-\delta$.

We have already proved the generalization bound of the boundary loss in PINN. That for residual loss is similar. Specifically, let 
\begin{equation}
\mathcal{G}_{M,N}^L = \left\{\bx \mapsto l(f(\boldsymbol{x}), \mathcal{L}u_{\bt}(\bx))\ |\ \Vert \boldsymbol{W}^l \Vert_{2} \leq M(l),\frac{\Vert \boldsymbol{W}^l \Vert_{2,1}}{\Vert \boldsymbol{W}^l \Vert_{2}} \leq N(l), l=1,...,L \right\}.
\end{equation}
Then the class of composition of the loss function and all $L$ layers differentiated neural networks becomes
\begin{equation}
\mathcal{G}^L = \cup_{M(1)=1}^\infty \cdots \cup_{M(L)=1}^\infty\cup_{N(1)=1}^\infty \cdots \cup_{N(L)=1}^\infty \mathcal{G}_{M,N}^L.
\end{equation}
Similarly, using our assumption of truncated neural network, we obtain:
\begin{equation}
\text{Rad}(\mathcal{G}_M^L;S) \leq 4 \text{Rad}(\mathcal{PINN}_{M,N}^L).
\end{equation}
By the result of Rademacher complexity of neural networks in Lemma \ref{lemma:rad_pinn}, for any given $\delta$ and any positive integers $M(1),...,M(L)$ and $N(1),...,N(L)$ with probability at least $1-\delta(M,N)$ over the training dataset $S$, we have
\begin{equation}
\begin{aligned}
&\quad R_{D\cap\Omega}(\boldsymbol{\theta})-R_{S\cap\Omega}(\boldsymbol{\theta})\\ 
&\leq 8\mathbb{E}_S\text{Rad}(\mathcal{PINN}_{M,N}^L;S) + 2 \sqrt{\frac{\log(2/\delta(M,N))}{2n_r}}\\
&\leq \frac{64K + 32d(L-1)K}{n_r\sqrt{n_r}} + 2 \sqrt{\frac{\log(2/\delta(M,N))}{2n_r}} +  \frac{144K\sqrt{d\log(2h^2)}\log n_r}{\sqrt{n_r}}\\
&\quad \prod_{l=1}^L M(l) \left(\sum_{l=1}^L N(l)^{2/3}\right)^{3/2} \left[ 1 + \sqrt{2} L\prod_{l=1}^L M(l)+ \sqrt{2} d(L^2-1)\left(\prod_{l=1}^L M(l)\right)^2 \right].
\end{aligned}
\end{equation}
For any parameter $\boldsymbol{\theta}$ minimizes the empirical loss, choose the integers $M(1),...,M(L)$ and $N(1),...,N(L)$ such that
\begin{equation}
\begin{aligned}
M(l) - 1 &< \Vert \boldsymbol{W}^l \Vert_{2} \leq M(l),\\
N(l) - 1 &<  \frac{\Vert\boldsymbol{W}^l \Vert_{2,1}}{\Vert\boldsymbol{W}^l \Vert_{2}} \leq N(l),
\end{aligned}
\end{equation}
and the integers are the smallest integers satisfying the above equations.
The above bound just holds with probability $1-\delta(M,N)$ for any pair $(\boldsymbol{\theta},M,N)$ as long as $\bt$ satisfies $M(l) = \lceil \Vert \boldsymbol{W}^l \Vert_{2} \rceil$, and $N(l) = \lceil \Vert \boldsymbol{W}^l \Vert_{2,1} / \Vert \boldsymbol{W}^l \Vert_{2} \rceil$. Since $\sum_{M,N} \delta(M, N) = \delta$, the bound holds with probability $1-\delta$.
\end{proof}

\subsection{Proof of Theorem \ref{thm:L2_generalization}}
\begin{proof}
We shall use Assumption \ref{assumption:L2} and the fact that
\begin{equation}
a + b \leq 2^{\frac{p-1}{p}} \left(a^p + b^p\right)^{\frac{1}{p}}.
\end{equation}
Specifically, specify $p=2$, we have
\begin{equation}
\begin{aligned}
\Vert u_{\theta} - u \Vert_{L_2(\Omega)} &\leq C_1^{-1} \left(\Vert \mathcal{L}u_{\theta} - \mathcal{L}u  \Vert_{L_2(\Omega)} + \Vert u_{\theta} - u  \Vert_{L_2(\partial\Omega)}\right)\\
&\leq \sqrt{2}C_1^{-1}\left(\Vert \mathcal{L}u_{\theta} - \mathcal{L}u  \Vert_{L_2(\Omega)}^2 + \Vert u_{\theta} - u  \Vert_{L_2(\partial\Omega)}^2\right)^{1/2}\\
&\leq \sqrt{2}C_1^{-1} \left(R_{D\cap\Omega}(\theta) + R_{D\cap\partial\Omega}(\theta)\right)^{1/2}.
\end{aligned}
\end{equation}
\end{proof}

\section{Related Work}
In this section, we summarize related works. We focus on related works on PINNs, Rademacher complexity of neural networks, and the theory of PINNs.

\subsection{Physics-Informed Neural Networks}
We first introduce some background on physics-informed neural networks (PINNs), which are the models we have considered throughout this paper. 

Due to its success in approximating high-dimensional functions while generalizing well, deep learning has been used to solve partial differential equations (PDEs). Among them, PINNs \cite{raissi2019physics} approximate the solutions of PDEs by neural networks, and then optimize them by stochastic gradient descent for expectation minimization to let them satisfy the physical rule described by the PDE. Later, the extended PINNs (XPINNs) \cite{jagtap2020extended} which adopt domain decomposition methods show faster convergence and better generalization performances than vanilla PINNs, but the underlying reason for this remains unknown. Prior to XPINN, CPINN \cite{jagtap2020conservative} is also a domain decomposition-based PDE solver. However, CPINN is only applicable to conservation laws and does not allow the general spatio-temporal domain decomposition. 

To the best of our knowledge, the present work provides the first proof on generalization of PINNs and XPINNs, and the first analysis on when and how  XPINNs perform better than PINNs. 

\subsection{Rademacher Complexity of Neural Networks}
In this subsection, we review the Rademacher complexity of neural networks, which plays a key role in our generalization theory on PINNs and XPINNs.

In statistical learning theory, the Rademacher complexity measures the richness of a class of functions on which the generalization error bound is based. In the literature, there have been various controls and estimations on the Rademacher complexity of the class of neural network functions. 

There are various ways to bound the Rademacher complexity of the class of neural networks, namely the norm-based control (adopted in this study), and sharpness. 
For norm-based capacity control, \cite{neyshabur2015norm} bounds Rademacher complexity by product of Frobenius norms of parameter matrices. However, their bounds grow exponentially as the depth increases, which contradicts the fact that deeper networks generalize better. To eliminate the exponential dependency on network depth, \cite{bartlett2017spectrally} uses a covering number approach to show a bound scaling as $O({\prod_{l=1}^L \Vert \boldsymbol{W}^l \Vert_2 (\sum_{l=1}^L (\frac{\Vert\boldsymbol{W}^{l}\Vert_{2,1}}{\Vert \boldsymbol{W}^l \Vert_2}^{\frac{2}{3}})^{\frac{3}{2}})}/{\sqrt{m}})$. Although the explicit dependency on network depth $L$ disappears, the bound still has polynomial dependency ($L^3$) on the depth due to the fact that $\Vert\boldsymbol{W}^{l}\Vert_{2,1} \geq \Vert \boldsymbol{W}^l \Vert$. To derive size-independent sample complexity for neural networks, \cite{Golowich2018SizeIndependentSC} further proves several useful results. Firstly, \cite{Golowich2018SizeIndependentSC} improves the dependency on depth from $L^3$ in \cite{bartlett2017spectrally} to $\sqrt{L}$. Secondly, \cite{Golowich2018SizeIndependentSC} uses Shatten $p$-norms of matrices to derive bounds which totally remove any dependency on the depth. \cite{neyshabur2017exploring} empirically validates the effectiveness of these norm-based capacity controls to explain the generalization mystery of deep learning. Another line of work focuses on sharpness, which adopts robustness of the training error to the perturbations in the parameters as a complexity measure for neural networks. \cite{neyshabur2017pac} combines sharpness measure with PAC-Bayesian approach, providing a generalization bound scaling at $O(\frac{\prod_{l=1}^L \Vert \boldsymbol{W}^l \Vert (\sum_{l=1}^L (\frac{\Vert\boldsymbol{W}^{l}\Vert_{F}}{\Vert \boldsymbol{W}^l \Vert}}{\sqrt{m}})$, which is shown to be similar to the bound in \cite{bartlett2017spectrally} when weights are sparse, and tighter than \cite{bartlett2017spectrally} when the weights
are fairly dense and are of uniform magnitude.

In this paper, we mainly consider \cite{bartlett2017spectrally} to control the Rademacher complexity of PINNs.

\subsection{Theory on PINNs}
Due to the success of PINNs in approximating high-dimensional complicated functions such as solutions of PDEs, theoretical evidence accounting for the outstanding empirical performance has increasingly attracted considerable attention.

The most related work is \cite{Luo2020TwoLayerNN}, where the authors consider Barron space for two-layer networks for prior and posterior generalization bounds. \cite{Luo2020TwoLayerNN} also leverages neural tangent kernel to show global convergence of PINNs. \cite{mishra2020estimates} introduces an abstract formalism and the stability properties of the underlying PDE are leveraged to derive an estimate for the generalization
error in terms of the training error and number of training samples. By adapting the Schauder approach and the maximum principle, \cite{shin2020convergence} shows that as number of training samples go to infinity, the minimizer converges to the solution in $C^0$ and $H^1$. \cite{lu2021priori} uses the Barron space for two-layer neural networks to provide a prior analysis on PINN with softplus activation, via adopting the similarity between softplus and ReLU.

Our work extends existing results to multi-layer networks, which is more general and realistic, and considers various kinds of capacity controls for PINNs, namely the Barron norm and the spectral norm. Extensive experiments and analytical examples further validate the effectiveness of our theory. Our work is also the first to analyze when and how XPINN is better than PINN.

\section{Why Barron Space?}
This subsection is devoted to clarify why we choose Barron space theory for developing our prior bound. Overall, it has the following two advantages.


Firstly, we should choose a theory that can measure complexity of both networks and target functions, which plays a key role in the prior generalization bound in Theorem \ref{thm:generalization}. In the Barron space, we are able to measure the complexity of target functions easily via Barron norm, and we can further show that complexities of trained neural networks are controlled by that of the target functions. Since the success of deep learning owns to its data-dependent training, i.e. although the class of networks has huge complexity, gradient descent does find out a simple network, which is reflected by the Barron space theory.

Secondly, the Barron space in high dimension neural networks resembles Sobolev and Besov space which are indispensable building blocks for low dimension classical theory. A proper function space is essential in analyzing PDEs. The class of network functions define a natural function space, i.e. the Barron space. By studying the target function of the PDE problem in the Barron space by its norm, the generalization error of the trained network can be obtained in terms of that norm. This reasoning resembles prior error analysis in classical finite element method where the error is controlled by the Sobolev norm of the target. Therefore, the Barron space adopted is appropriate for PDE analysis. 

\section{Additional Comparison}
\subsection{Comparison of Boundary Loss via Theorem \ref{thm:generalization}}
The comparison will be done via computing their respective theoretical bounds.
In particular, the generalization performance of PINN depends on the upper bound in Theorem \ref{thm:generalization}, which is 
\begin{equation}
R_{S \cap \partial \Omega}(\bt^*) + 8\Vert u^* \Vert_{\mathcal{W}^L(\Omega)}\frac{C(h)\log n_b}{{\sqrt{n_b}}} + 2\sqrt{\frac{\log(2/\delta)}{n_b}}.
\end{equation}
where $n_b$ is the number of boundary training points.

For XPINN's generalization, we can apply Theorem \ref{thm:generalization} to each of the subdomains in XPINN. Specifically, for the $i$-th sub-net in the $i$-th subdomain, i.e. the $\Omega_i, i\in\left\{1,2,...,N_D\right\}$, its generalization performance is upper bounded by
\begin{equation}
R_{S \cap \partial \Omega_i}(\bt^*) + 8\Vert u^* \Vert_{\mathcal{W}^L(\Omega_i)}\frac{C(h)}{{\sqrt{n_{b,i}}}} + 2\sqrt{\frac{\log(2/\delta)}{n_{b,i}}}.
\end{equation}
where $n_{b,i}$ is the number of training boundary points in the $i$-th subdomain. 

Hence, since the $i$-th subdomain has $n_{b,i}$ training boundary points and is in charge of the prediction of $\frac{n_{b,i}}{n_b}$ proportion of testing data, we weighted average their generalization errors to get the generalization error of XPINN
\begin{equation}
\begin{aligned}
\sum_{i=1}^{N_D} \frac{n_{b,i}}{n_b}\left(R_{S \cap \partial \Omega_i}(\bt^*) + 8\Vert u^* \Vert_{\mathcal{W}^L(\Omega_i)}\frac{C(h)}{{\sqrt{n_{b,i}}}} + 2\sqrt{\frac{\log(2/\delta)}{n_{b,i}}}\right),
\end{aligned}
\end{equation}
If we omit the last term and assume the empirical losses of PINN and XPINN are similar, i.e.
\begin{equation}
\begin{aligned}
R_{S \cap \partial \Omega} &\approx \sum_{i=1}^{N_D}\frac{n_{b,i}}{n_b}R_{S \cap \partial \Omega_i}, \sqrt{\frac{2\log(2/\delta)}{n_{b,i}}} &\ll \Vert u^* \Vert_{\mathcal{W}^L(\Omega)}, \Vert u^* \Vert_{\mathcal{W}^L(\Omega_{i})},
\end{aligned}
\end{equation}
then comparing the generalization ability of PINN and XPINN reduces to the following
comparison:\begin{equation}
    \underbrace{\Vert u^* \Vert_{\mathcal{W}^L(\Omega)}}_{\text{PINN}} \ \qquad \text{versus} \qquad \underbrace{\sum_{i=1}^{N_D}\sqrt{\frac{n_{b,i}}{n_b}}\Vert u^* \Vert_{\mathcal{W}^L(\Omega_i)}}_{\text{XPINN}},
\end{equation}
where model having smaller corresponding quantity is more generalizable.

\subsection{Comparison of Boundary Loss via Theorem \ref{thm:post_generalization}}
In this subsection, we compare PINN with XPINN by Theorem \ref{thm:post_generalization}, where we focus on the boundary losses of PINN and XPINN. we denote the upper bound of PINN testing loss as $B_{\text{PINN}}$ and those of the sub-net $i$ in XPINN as $B_{i,\text{XPINN}}$, $i \in \left\{1,2,...,N_D\right\}$ which are provided by the right sides of Theorem \ref{thm:post_generalization}, i.e. the bounds are
\begin{equation}
B_{\text{PINN}} =  R_{S\cap\partial\Omega}(\boldsymbol{\theta})+ \frac{32}{n_b\sqrt{n_b}} + \frac{144\sqrt{d\log(2h^2)}\log n_b}{\sqrt{n_b}} \prod_{l=1}^L M(l)\Big(\sum_{l=1}^L N(l)^{2/3}\Big)^{3/2}+ 2 \sqrt{\frac{\log(2/\delta(M,N))}{2n_b}}.
\end{equation}
\begin{equation}
B_{i,\text{XPINN}} =  R_{S\cap\partial\Omega_i}(\boldsymbol{\theta})+ \frac{32}{n_{b,i}\sqrt{n_{b,i}}} + \frac{144\sqrt{d\log(2h^2)}\log n_{b,i}}{\sqrt{n_{b,i}}} \prod_{l=1}^L M_i(l)\Big(\sum_{l=1}^L N_i(l)^{2/3}\Big)^{3/2}+ 2 \sqrt{\frac{\log(2/\delta(M_i,N_i))}{2n_{b,i}}}.
\end{equation}
Specifically, we assume that all sub-PINNs as well as the PINN model use neural networks with depth $L$ and width $h$. In the bound of PINN, $n_r$ is the total number of residual training samples. $M(l) = \lceil \Vert \boldsymbol{W}^l \Vert_{2} \rceil$, and $N(l) = \lceil \Vert \boldsymbol{W}^l \Vert_{2,1} / \Vert \boldsymbol{W}^l \Vert_{2} \rceil$, where $\boldsymbol{W}^l$ is the $l$-th layer parameter matrix in the PINN model. Moreover, in the bound of XPINN, $n_{r,i}$ is the number of residual training samples in subdomain $i$. $M_i(l) = \lceil \Vert \boldsymbol{W}^l_i \Vert_{2} \rceil$, and $N_i(l) = \lceil \Vert \boldsymbol{W}^l_i \Vert_{2,1} / \Vert \boldsymbol{W}^l_i \Vert_{2} \rceil$, where $\boldsymbol{W}^l_i$ is the $l$-th layer parameter matrix of the $i$-th subnet in the XPINN model.
Because the $i$-th sub-net in XPINN is in charge of the prediction of $\frac{n_{b,i}}{n_b}$ proportion of testing data, we weight-averaged their bounds to get that of XPINN, i.e., $B_{\text{XPINN}} = \sum_{i=1}^{N_D} ({n_{b,i}}/{n_b})B_{i,\text{XPINN}}$ where $B_{\text{XPINN}}$ is the bound for XPINN. Thus, we only need to compare $B_{\text{PINN}}$ with $B_{\text{XPINN}}$, where the model having smaller corresponding quantity is more generalizable.

\newpage

\bibliographystyle{unsrtnat}
\bibliography{references}

\end{document}